\documentclass[lettersize,journal]{IEEEtran}

\usepackage{amsmath,amsfonts,bm}

\def\eqref#1{equation~\ref{#1}}

\def\1{\mathbbm{1}}

\def\rvx{{\mathbf{x}}}
\def\rvy{{\mathbf{y}}}

\DeclareMathAlphabet{\mathsfit}{\encodingdefault}{\sfdefault}{m}{sl}
\SetMathAlphabet{\mathsfit}{bold}{\encodingdefault}{\sfdefault}{bx}{n}

\DeclareMathOperator*{\argmin}{arg\,min}

\usepackage{amsmath,amsfonts}
\usepackage[ruled,linesnumbered]{algorithm2e}
\usepackage{array}
\usepackage[caption=false,font=normalsize,labelfont=sf,textfont=sf]{subfig}
\usepackage{textcomp}
\usepackage{color}
\usepackage{bm}
\usepackage{stfloats}
\usepackage{url}
\usepackage{verbatim}
\usepackage{graphicx}
\usepackage{amsmath}
\usepackage{amssymb}
\usepackage{amsthm}
\usepackage{cite}
\usepackage[colorlinks,linkcolor=blue,citecolor=blue]{hyperref} 
\usepackage[capitalize]{cleveref}
\usepackage{thmtools, thm-restate}
\usepackage{enumerate}
\usepackage{booktabs}
\usepackage{caption}
\usepackage{float}
\usepackage{multicol}
\usepackage{xcolor}
\usepackage[normalem]{ulem}
\newtheorem{assumption}{Assumption}

\newtheorem{definition}{Definition}
\newtheorem{lemma}{Lemma}
\newtheorem{theorem}{Theorem}
\newtheorem{corollary}{Corollary}

\newcommand{\name}{\textsc{AugFL}}
\newcommand{\etal}{\textit{et al}.}
\newcommand{\ie}{\textit{i}.\textit{e}.}

\allowdisplaybreaks[4]

\begin{document}

\title{\name{}: Augmenting Federated Learning with Pretrained Models}

\author{
Sheng Yue,~\IEEEmembership{Member,~IEEE},
Zerui Qin,~\IEEEmembership{Student Member,~IEEE},
Yongheng Deng,~\IEEEmembership{Member,~IEEE},
Ju Ren,~\IEEEmembership{Senior Member,~IEEE},
Yaoxue Zhang,~\IEEEmembership{Senior Member,~IEEE},
and Junshan Zhang,~\IEEEmembership{Fellow,~IEEE}
\thanks{This work was accepted in part by the 22nd International Symposium on Theory, Algorithmic Foundations, and Protocol Design for Mobile Networks and Mobile Computing (MobiHoc)~\cite{yue2021inexact}.}
\thanks{Sheng Yue is with the School of Cyber Science and Technology, Sun Yat-sen University. Email: yuesh5@mail.sysu.edu.cn.}
\thanks{Zerui Qin, Yongheng Deng, Ju Ren (corresponding author), and Yaoxue Zhang are with the Department of Computer Science and Technology, Tsinghua University. Ju Ren and Yaoxue Zhang are also with Zhongguancun Laboratory. Emails: qzr24@mails.tsinghua.edu.cn, \{dyh2024,renju, zhangyx\}@mail.tsinghua.edu.cn.}
\thanks{Junshan Zhang is with the Department of Electrical and Computer Engineering, University of California, Davis. Email: jazh@ucdavis.edu.}
\thanks{Sheng Yue and Zerui Qin contributed equally to this work.}
}




\maketitle

\begin{abstract}
Federated Learning (FL) has garnered widespread interest in recent years. However, owing to strict privacy policies or limited storage capacities of training participants such as IoT devices, its effective deployment is often impeded by the scarcity of training data in practical decentralized learning environments. 
In this paper, we study enhancing FL with the aid of (large) pre-trained models (PMs), that encapsulate wealthy general/domain-agnostic knowledge, to alleviate the data requirement in conducting FL from scratch. Specifically, we consider a networked FL system formed by a central server and distributed clients. First, we formulate the PM-aided personalized FL as a regularization-based federated meta-learning problem, where clients join forces to learn a meta-model with knowledge transferred from a private PM stored at the server. Then, we develop an inexact-ADMM-based algorithm, \name{}, to optimize the problem with no need to expose the PM or incur additional computational costs to local clients. Further, we establish theoretical guarantees for \name{} in terms of communication complexity, adaptation performance, and the benefit of knowledge transfer in general non-convex cases. Extensive experiments corroborate the efficacy and superiority of \name{} over existing baselines.
\end{abstract}


\section{Introduction}
\label{sec:introduction}

\IEEEPARstart{F}{ederated} Learning (FL)~\cite{mcmahan2017communication,li2020federated,kairouz2021advances} has gained prominence as a distributed learning paradigm allowing a large number of decentralized users to collaboratively train models without sharing their local data, which has garnered significant attention from both academia and industry~\cite{bonawitz2019towards,dinh2020federated,nguyen2021federated,xue2021asynchronous,li2021privacy,gecer2024federated,deng2024communication}. Despite its rapid advancements, the effective deployment of FL has been hampered by a significant hurdle: in practice, the participants, such as IoT devices~\cite{yang2019federated}, can provide only scarce training data due to limited storage capacities or strict privacy policies~\cite{zhang2022federated,wang2023attrleaks}. Therefore, it is often unsatisfactory to carry out FL from scratch in many modern data-hungry applications like natural language processing~\cite{brown2020language} and robotic control~\cite{collins2019quantifying}.

More recently, the field of artificial intelligence has undergone a significant transformation with the advent of large \textit{pretrained models} (PMs)~\cite{bommasani2021opportunities}. Advanced large models, such as GPT-4~\cite{achiam2023gpt}, PaLM~\cite{chowdhery2023palm} and LLaMA~\cite{touvron2023llama} boasting billions of parameters, have drawn considerable attention due to their remarkable performance in various AI tasks ranging from natural language processing, content generation to more complex tasks such as planning and reasoning~\cite{wei2022emergent}. Given the wealthy general/domain-agnostic knowledge of PMs, and considering the current limitation of FL as noted, a pertinent question is: \textit{can we augment FL with the power of PM to reduce the data requirement in conducting FL from scratch? }


A straightforward solution is to utilize the PM as a warm start for FL and finetune it federatively with clients' domain-specific data~\cite{shi2022just,hou2023privately,yu2023selective}. Yet, owing to the potentially colossal parameter sizes, directly deploying PMs on resource-limited clients may outweigh their computational and storage capabilities. More recently, many efforts have been made to compress the PM and exploit the smaller, compressed model for warm-up or knowledge transfer~\cite{yu2021LLMfinetune,mireshghallah2022differentially,li2023backdoor,wang2023can,yuan2023rethinking}. However, compression for large models is still in the early stage and remains highly nontrivial due to the complex configuration and black-box nature thereof~\cite{deng2020model}. More importantly, sharing either the original or compressed PM, or even its representations, would inevitably threaten the ownership of PMs~\cite{kang2023grounding, roy2019mitigating, jeong2022privacy}. With the growing trend toward proprietary PMs such as GPT-4~\cite{achiam2023gpt}, Claude-3~\cite{claude2024}, GLM-4~\cite{glm2024chatglm}, and Kimi~\cite{qin2024mooncake}, which developed at substantial expense of time and cost, the ownership of PMs is increasingly recognized as a critical issue in the community~\cite{li2022fedipr, fan2023fate, xiao2023offsite}.

To overcome these limitations, this paper introduces a new PM-aided personalized FL framework, which can enhance FL by PM without exposing the information of PM or introducing additional computational burden to clients. First, we formulate the PM-aided personalized FL as a regularized federated meta-learning problem to collaboratively learn a \textit{meta-model} with the PM knowledge transfer characterized by a regularization. Then, we exploit \textit{Alternative Direction Method of Multipliers} (ADMM) to decompose the original problem into a set of subproblems that can be solved in parallel across clients. Notably, ADMM can decouple the regularizer from the computation at local clients, and the knowledge transfer is thus achieved solely on the server side, \ie, `transparent' to all clients. Observe that conventional ADMM techniques are computationally intensive owing to the need for exact 
solutions to the (potentially non-convex) subproblems in each round. To tackle this, we develop an inexact variant of ADMM, \name, where we exploit linear approximation as well as Hessian estimation to transform each subproblem into a quadratic form that can be solved with a closed-form solution, achieving a computational complexity of $\mathcal{O}(n)$ per round (with $n$ the model dimension) and maintaining the lowest among existing federated meta-learning approaches. 

Theoretically, we analyze the convergence and performance of the proposed method. We note that the error induced by linear approximation and Hessian estimation complicates the proof of the convergence of \name{}. The existing results~\cite{hong2016convergence,wang2019global,barber2024convergence} cannot be applied directly because the sufficient descent condition of the Lagrangian function is violated. We develop a new technical path to resolve this issue and establish the convergence guarantee for general non-convex cases. Further, we characterize the adaptation performance of the learned meta-model and quantify the benefit of the PM knowledge transfer. In particular, unlike previous approaches~\cite{lin2020collaborative,fallah2020personalized}, \name{} can converge under mild conditions -- not requiring the regular \textit{similarity} assumptions on training clients. Therefore, it can be applied to unbalanced local datasets, unleashing the potential to cope with the statistical heterogeneity of FL. 

In a nutshell, the main contributions of this paper are three-fold:
\begin{itemize}
\item We cast the PM-aided personalized FL as a regularizaton-based federated meta-learning problem, where clients join forces to learn a meta-model with the knowledge transfer from a private PM at the server. 
\item We devise an inexact-ADMM-based algorithm, \name{}, capable of decomposing the original problem into a set of subproblems that can be solved in parallel across clients while enabling the knowledge transfer computed only on the server side. \name{} solves the sub-problems via linear approximation and Hessian estimation, reducing the computational cost per round to $\mathcal{O}(n)$, maintaining the lowest among existing federated meta-learning methods.
\item We carry out a comprehensive analysis of the proposed method for general nonconvex cases, where we provide the convergence guarantees and communication complexity. Besides, we characterize the adaptation performance and quantify the benefit of the PM knowledge transfer. 
\item We evaluate the performance of \name{} across different benchmark datasets. The extensive experimental results showcase that \name{} outperforms existing baselines, in terms of both performance and convergence speed.
\end{itemize}

In what follows, we review related literature in \cref{sec:related_work}, including federated transfer learning, personalized federated learning, and ADMM. In \cref{sec:method}, we introduce the investigated framework and methodology. In \cref{sec:analysis}, we provide rigorous performance guarantees for our method. In \cref{sec:instantiation}, we provide practical instantiation for the component of knowledge transfer. Finally, we present the experimental results in \cref{sec:experiment} and conclude our work in \cref{sec:conclustion}.

\section{Related Work}
\label{sec:related_work}


Federated learning, introduced in~\cite{mcmahan2017communication,koneny2016fl}, enables decentralized clients to collaboratively learn a shared model while keeping their local data private. In recent years, it has attracted significant attention and has been applied in many practical domains and products (see \cite{kairouz2021advances} for a comprehensive review).



\textbf{Pretrained-model-aided federated learning.} Recent advances in (large) PMs have prompted various attempts to integrate them into FL. A straightforward solution is to exploit the PM to initialize the FL process, providing a ``warm start''. Nguyen \etal~\cite{nguyen2022begin} examine the impact of pretraining and initialization in FL, showcasing that the performance of FL can be significantly enhanced by leveraging the PM as an initial model. Tan \etal~\cite{tan2022federated} propose an approach where a frozen PM acts as an encoder deployed on each client, followed by prototype-based knowledge distillation to train a personalized projector on top of the PM. Yet, the colossal parameters of PM may overweigh the computational and storage capabilities of resource-constrained clients.


Rather than deploying the complete PM on clients, \textit{Federated Transfer Learning} (FTL)~\cite{kang2023grounding}, a marriage of FL and transfer learning techniques, has been deemed as a promising solution to adapt large PMs to domain-specific client models. Many FTL approaches have been proposed to compress the PM and exploit the smaller and compressed model for knowledge transfer. For example, Mireshghallah \etal~\cite{mireshghallah2022differentially} study differentially private model compression based on knowledge distillation and pruning. Wang \etal~\cite{wang2023can} develop a distribution matching algorithm to select public data, mirroring the distribution of the client private data, to fine-tune compressed models effectively. Unfortunately, the compression for large PMs is highly challenging due to their complex configuration and black-box nature~\cite{deng2020model}. To deal with this problem, several knowledge-distillation-based methods are proposed to distill prototypes or representations of the PM to clients~\cite{zhang2024upload, he2020group}. However, sharing either parameters or representations of the PM inevitably threatens the ownership of the PM~\cite{kang2023grounding, li2021privacy, roy2019mitigating}, which has been deemed as a critical issue due to the increasing number of proprietary PMs developed at substantial time and cost~\cite{xiao2023offsite}. Furthermore, downloading additional PM parameters or representations can also impose computational and storage burdens on resource-limited clients.

\textbf{Personalized federated learning.} Personalization is crucial for FL, especially in the decentralized learning environments with statistical heterogeneity~\cite{tan2022pfl}. Jeong \etal~\cite{jeong2018communication} utilize data augmentation to alleviate the local data discrepancy. The method first maintains a generative adversarial network (GAN) at the server, which is then distributed to clients to generate i.i.d. data. Wang \etal~\cite{wang2020optimizing} propose a client selection strategy based on $Q$-learning to balance the bias introduced by non-i.i.d. data. Arivazhagan \etal~\cite{arivazhagan2019federated} develop a structure-based method dividing the model into base and personalized layers, facilitating domain adaptation. Smith \etal~\cite{smith2017federated} implement a multi-task framework based on a primal-dual formulation, which considers the data on different clients as distinct tasks, allowing the learning of a personalized model for each client. 

Federated meta-learning is a prominent research direction in personalized FL, where the integration of \textit{Model-Agnostic Meta-Learning} (MAML)~\cite{finn2017model} and FL has attracted increasing attention. Chen \etal~\cite{chen2018federated} introduce a federated meta-learning framework (called Fed-Meta) based on FedAvg~\cite{mcmahan2017communication} and the MAML-type algorithms, indicating the performance improvement and convergence acceleration over vanilla FedAvg. Jiang \etal~\cite{jiang2019improving} comprehensively analyze the connections between FedAvg and MAML. Lin \etal~\cite{lin2020collaborative} analyze the convergence properties
and computational complexity of the combination of MAML and FedAvg in a strongly convex setting. Further, Fallah \etal~\cite{fallah2020personalized} establish the convergence guarantee for the non-convex setting. However, these studies focus on performing FL from scratch without exploring the knowledge transfer from PMs.

\textbf{ADMM.} A number of existing works~\cite{hong2016convergence,magnusson2015convergence,wang2014convergence,wang2019global} analyze the convergence of ADMM for the case where the solution to each subproblem is computed exactly. Wang \etal~\cite{wang2018convergence} extend the ADMM method from two-block to multi-block form. In addition, there are also a few works~\cite{barber2024convergence,jiang2019structured,lanza2017nonconvex,mukkamala2020convex} studying the performance of ADMM in an inexact setting. It is worth noting that these existing results cannot be directly applied in analyzing our algorithm, because in \name{}, the sufficient descent condition of the Lagrangian is violated due to the introduced linearization and Hessian estimation.



\section{Pretrained-Model-Aided Personalized Federated Learning}
\label{sec:method}

In this section, we first introduce the setting and formulation of the investigated problem and then present our solution.

\subsection{Problem Formulation}
\label{sec:formulation}

As depicted in \cref{fig:system_model}, we consider a PM-aided FL framework, where a set of distributed clients, denoted as $\mathcal{I}$, are connected to a central server which possesses a private pretrained model, parameterized by $\theta_p\in\mathbb{R}^m$.\footnote{In this paper, the PM follows its general definition: a model that is trained on a large dataset prior to being fine-tuned or used for specific tasks~\cite{han2021pre, marcelino2018transfer, zhuang2020comprehensive}.} Each client $i\in\mathcal{I}$ has a labeled dataset, $\mathcal{D}_i=\{(\mathbf{x}^j_i,\mathbf{y}^j_i)\}^{D_i}_{j=1}$, with total $D_i$ samples, where  $(\mathbf{x}^j_i,\mathbf{y}^j_i)\in\mathcal{X}_i\times\mathcal{Y}_i$ denotes a data sample with input $\mathbf{x}^j_i$ and label $\mathbf{y}^j_i$, following an unknown distribution $P_i$. Of note, we focus on heterogeneous cases where the distributions may not be identical across the clients. In this setting, the goal of clients is to effectively utilize the knowledge of the PM and collaboratively learn personalized models without getting access to the PM or exchanging their local data with other clients or the central server. 


\begin{figure}[t]
    \centering
    \includegraphics[width=0.85\linewidth]{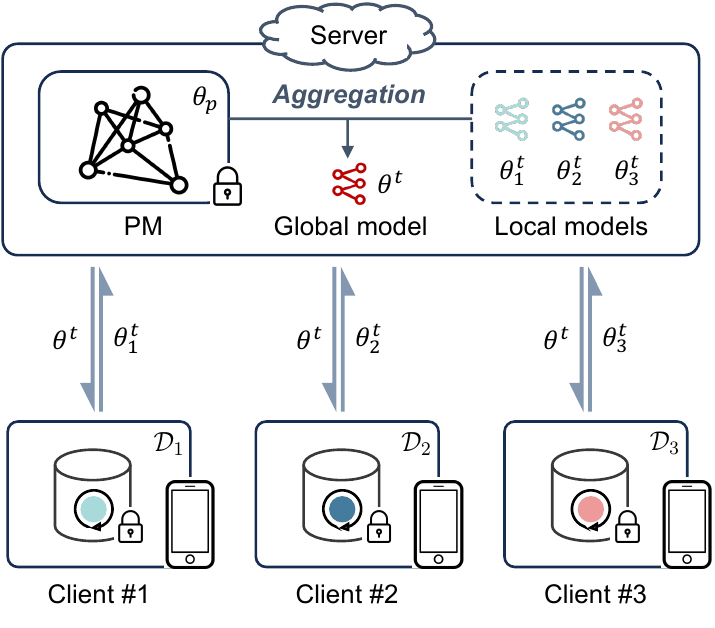}
    \caption{An illustration of the pretrained-model-aided federated meta-learning framework, where the server and clients form a networked system for federated learning. It consists of two main phases: 1) the distributed learning of local models based on the updated global model and 2) the knowledge transfer from the pretrained model to the current global model.}
    \label{fig:system_model}
\end{figure}

More formally, given a model parameter $\phi\in\mathbb{R}^n$, we define the loss function on $\mathcal{D}_i$ as 
\begin{align}
    L_i(\phi,\mathcal{D}_i)\doteq\frac{1}{D_i}\sum^{D_i}_{j=1}l_i(\phi,(\rvx^j_i,\rvy^j_i)),
\end{align}
where $l_i(\phi,(\rvx,\rvy))$ measures the error of $\phi$ in predicting $\rvy$ given input $\rvx$, which is assumed to be differentiable w.r.t. $\phi$. Following the same line as in \textit{Model-Agnostic Meta-Learning} (MAML)~\cite{finn2017model}, we divide dataset $\mathcal{D}_i$ of client $i$ into two disjoint sets, \ie, the \textit{support set} $\mathcal{D}^{s}_i$ and \textit{query set} $\mathcal{D}^{q}_i$, 
based on which, we formulate the PM-aided personalized FL as the following regularized optimization problem:
\begin{align}
    \label{prim_prob}
    \min_\theta&~\sum_{i\in\mathcal{I}}w_i L_i(\phi_i(\theta),\mathcal{D}^{q}_i)+\lambda R_h(\theta,\theta_p)\\
    \label{eq:phi}
    \mathrm{s.t.}&~~\phi_i(\theta)=\theta-\alpha\nabla L_i(\theta,\mathcal{D}^{s}_i),~i\in\mathcal{I}.
\end{align}
Here, $\theta\in\mathbb{R}^n$ denotes the model parameter; $w_i$ is the weight parameter (often set as $D_i/\sum_{i\in\mathcal{I}}D_i$) that characterizes client $i$'s contribution; and $\alpha$ is the learning rate. We denote $R_h$ as a general loss function for knowledge transfer~\cite{park2019relational,ji2021show,zhao2022decoupled}, with $\lambda\ge0$ a balancing parameter that trades off the losses of collaborative learning and knowledge transfer.\footnote{We remark the primary objective of this work is to establish a principle framework for pretrained-model-aided FL, which is independent of the specific form of $R_h$. In particular, we provide an instantiation in \cref{sec:instantiation}.} Given this formulation, we aim to find an initial model (namely meta-model) that after slight updating, \ie, one-step gradient descent, leads to a good personalized (small) model for each individual local data while equipped with the general knowledge transferred from the PM. It therefore achieves the adaptation to heterogeneous data. To illustrate this formulation, consider training a natural language processing model over a set of clients. In this problem, $P_i$ and $\mathcal{D}_i$ represent the distribution and the collected data of client $i$'s expressions, respectively. Thus, the  $\theta_p$ can be a pretrained \textit{Large Language Model} (LLM) like GPT-3.5. Upon solving Problem (\ref{prim_prob})-(\ref{eq:phi}), each client $i$ could take the meta-model and update it by going over their own data $\mathcal{D}_i$ and performing just one or few steps of gradient descent to obtain a personalized model that not only fits its own expressions well but is also equipped with the ability of general-purpose language generation.

In a nutshell, the strengths of this formulation are threefold: 1) It gives personalized solutions that can capture any heterogeneity among clients; 2) it can transfer valuable pretrained information to clients' specific domains, enhancing the robustness and generalization of clients' local models; 3) except for the training clients, the learned meta-model can quickly adapt to \textit{new} clients via slightly updating it using limited domain-specific data, \ie, the capability of few-slot learning. However, directly addressing Problem (\ref{prim_prob})-(\ref{eq:phi}) by existing federated meta-learning methods, such as Per-FedAvg~\cite{fallah2020personalized}, requires clients to download the PM and handle the knowledge transfer on the client side, which incurs significant computational and storage burdens for resource-limited clients. We resolve this problem in the following section.

\subsection{\name: An Inexact-ADMM-Based Algorithm for Regularized Federated Meta-Learning}
\label{sec:methodology}

As alluded to earlier, existing approaches cannot handle the regularized optimization well. In this section, we introduce a new algorithm based on ADMM to solve Problem (\ref{prim_prob})-(\ref{eq:phi}). 

Observe that the original problem (\ref{prim_prob})-(\ref{eq:phi}) is equivalent to the following constrained optimization problem:
\begin{align}
    \label{equi_prob}
    \begin{split}
        \min_{\theta_{1:|\mathcal{I}|},\theta}&~\sum_{i\in\mathcal{I}}w_i L_i(\phi_i(\theta_i),\mathcal{D}^{q}_i)+\lambda R_h(\theta,\theta_p)\\
        \mathrm{s.t.}~&~~\phi_i(\theta)=\theta-\alpha\nabla L_i(\theta,\mathcal{D}^{s}_i),\\
        &~~\theta_i - \theta = 0,~i\in\mathcal{I}.
    \end{split}
\end{align}
We form the corresponding \textit{augmented Lagrangian function} of Problem~(\ref{equi_prob}) as follows:
\begin{align}
    \label{eq:lagrangian}
    \nonumber
    \mathcal{L}(\{\theta_i,y_i\},\theta)&\doteq\sum_{i\in\mathcal{I}}(w_i L_i(\phi_i(\theta_i),\mathcal{D}^{q}_i)+\langle y _i,\theta_i-\theta\rangle\\
    &+\frac{\rho_i}{2}\Vert \theta_i-\theta\Vert^2)+\lambda R_h(\theta,\theta_p),
\end{align}
where $y_i\in\mathbb{R}^n$ is the dual variable and $\rho_i>0$ is the penalty parameter for each $i\in\mathcal{I}$. 

When classical ADMM~\cite{boyd2011distributed} is applied, the variables $\theta_i$, $\theta$ and $y_i$ are updated alternatively in Problem~(\ref{equi_prob}), that is, 
\begin{align}
    \label{eq:ori_ADMM}
    \begin{cases}
    \theta^{t+1}=\argmin_{\theta}\mathcal{L}(\{\theta^t_i,y^t_i\},\theta),\\
    \theta^{t+1}_i=\argmin_{\theta_i}\mathcal{L}_i(\theta_i,y^t_i,\theta^{t+1}),\\
    y^{t+1}_i=y^t_i+\rho_i(\theta^{t+1}_i-\theta^{t+1}),\\
    \end{cases}
\end{align}
where $\mathcal{L}_i(\theta_i,y_i,\theta)\doteq w_i L_i(\phi_i(\theta_i),\mathcal{D}^{q}_i)+\langle y _i,\theta_i-\theta\rangle+\frac{\rho_i}{2}\Vert \theta_i-\theta\Vert^2$. It decomposes Problem~(\ref{equi_prob}) into a set of subproblems that can be solved in parallel, all while computing $R_h(\theta,\theta_p)$ and $L_i(\phi_i(\theta_i),\mathcal{D}^{q}_i)$ separately. Building on this, to fully take advantage of the combined computation power of the local clients and the server, we provide the following alternating updating strategy: 1) updating $\theta$ at the server and 2) updating $\theta_i,y_i$ at the clients in a distributed manner. This way, the computation corresponding to the regularizer can be decoupled from the client side to the server side, avoiding exposing the PM's information to clients. Yet, attaining the exact solution to each subproblem is computationally costly, especially with complex function approximations. To tackle this challenge, we devise an inexact-ADMM-based algorithm, namely \name{}, in what follows.

Specifically, in the communication round $t=0$, the server randomly initializes the global model parameter, denoted as $\theta^0$, and sends it to all clients; each client locally initializes the dual variable, $y^{-1}_i$. Then, \name{} iterates in two repeating steps:\footnote{In what follows, we use $\theta^t_i$ to denote the `local model' that is updated locally on client $i$, and use  $\theta^t$ to denote the `global model' that is aggregated from all local models after each round of aggregation. It is worth noting that the local model and global model are both meta-models, i.e., learning model initialization.} 

\textbf{1) Local update of $\bm{\theta_i}$ and $\bm{y_i}$.} After receiving the global model $\theta^t$ from the server at communication round $t$, each client $i\in\mathcal{I}$ does the following updates:
    
\begin{itemize}
    \item \textit{Update client-specific model $\phi_i$:} Based on local dataset $\mathcal{D}^{s}_i$, $\phi^t_i$ is updated as
    \begin{align}
        \label{eq:update_phi}
        \phi^t_i=\theta^t-\alpha\nabla L_i(\theta^t,\mathcal{D}^{s}_{i}).
    \end{align}
    
    \item  \textit{Update local parameter $\theta_i$:} Based on \cref{eq:ori_ADMM}, given the global model (meta-model) $\theta^t$ and the local dual variable $y_i^{t-1}$ from last communication round, the local parameter $\theta_i$ is updated as
    \begin{align}
    \label{eq:update_thetai}
        \theta^t_i&=\mathop{\arg\min}_{\theta_i}\{w_i L_i(\phi_i(\theta_i),\mathcal{D}^{q}_i)+ \langle y^{t-1}_i,\theta_i-\theta^t\rangle \nonumber\\
      &+\frac{\rho_i}{2} \Vert\theta_i-\theta^t\Vert^2\}.
    \end{align}
    To simplify the computation, we use linear approximation (\ie, first-order Taylor expansion) around $\theta^t$ to relax the above subproblem, \ie,
    \begin{align}
    \label{eq:linear_approx}
    \nonumber
    \theta^t_i&=\mathop{\arg\min}_{\theta_i}\{w_i L_i(\phi^t_i,\mathcal{D}^{q}_i) + \langle w_i(I-\alpha\nabla^2 L_i(\theta^t,\mathcal{D}^{s}_i))\\
    &\cdot\nabla  L_i(\phi^t_i,\mathcal{D}^{q}_i)+y^{t-1}_i,\theta_i-\theta^t\rangle+\frac{\rho_i}{2} \Vert \theta_i-\theta^t\Vert^2\}
    \end{align}
    where $\phi^t_i$ is from \cref{eq:update_phi}. Nevertheless, \cref{eq:linear_approx} is still insufficient since the computational complexity of the Hessian-gradient product $\nabla^2 L_i(\theta^t,\mathcal{D}^{s}_i)\nabla L_i(\phi^t_i,\mathcal{D}^{q}_i)$ remains $\mathcal{O}(n^2)$. To further alleviate the computational cost, we replace the Hessian-gradient product with a first-order estimator, \ie,
    \begin{align}
    \label{estimatehessian}
    g^t_i\doteq \frac{\nabla L_i(\theta^t+\delta_{i,t} r^t_i,\mathcal{D}^{s}_i)-\nabla L_i(\theta^t-\delta_{i,t} r^t_i,\mathcal{D}^{s}_i)}{2\delta_{i,t}}
    \end{align}
    where we represent $r^t_i\doteq\nabla L_i(\phi^t_i,\mathcal{D}^{q}_i)$ and $\delta_{i,t}>0$ as the degree of freedom to capture the estimation accuracy. Replacing $\nabla^2 L_i(\theta^t,\mathcal{D}^{s}_i)\nabla L_i(\phi^t_i,\mathcal{D}^{q}_i)$ with $g^t_i$ in \cref{eq:linear_approx} and solving the approximated problem yields the update of the local model parameter $\theta_i^t$ as follows:
    \begin{align}
    \label{eq:update_thetai_inex}
    \theta^t_i=\theta^t-\frac{y^{t-1}_i+w_i(\nabla L_i(\phi^t_i,\mathcal{D}^q_{i})-\alpha g^t_i)}{\rho_i}.
    \end{align}    
    \item \textit{Update dual variable $y_i$:} Based on the global model $\theta^t$ and the updated local parameter $\theta_i^t$, the auxiliary dual variable $y_i^t$ is next updated according to
    \begin{align}
        \label{eq:update_y}
        y^t_i=& y^{t-1}_i + \rho_i(\theta^t_i - \theta^t).
    \end{align}
\end{itemize}

\textbf{2) Global aggregation towards meta-model $\bm{\theta}$.} Each client $i$ sends the updated local parameters $\theta^t_i$ and $y^t_i$ to the server. With the pretrained model $\theta_p$, the server performs a global update of the meta-model $\theta$ based on
\begin{align}\label{eq:update_theta}
    \theta^{t+1}=\frac{\sum_{i\in\mathcal{I}}(y^t_i+\rho_i\theta^t_i)-\lambda\nabla_{\theta^t}R_h(\theta^t,\theta_p)}{\sum_{i\in\mathcal{I}}\rho_i}.
\end{align}
Akin to \cref{eq:update_thetai_inex}, \cref{eq:update_theta} is derived from the optimality of the linearized $\mathcal{L}(\{\theta^t_i,y^t_i\},\theta)$ around $\theta^t$. After that, the server sends $\theta^{t+1}$ back to each client and starts the next round.

After the training phase, the server distributes the learned global meta-model $\theta^T$ (with $T$ the total number of rounds) to all clients. Based on this, each client can perform one or a few steps of stochastic gradient descent using the local dataset to obtain a personalized model suited for its specific domain. 


In summary, the procedure of \name{} is outlined in \cref{alg}. Note that due to linearizing all decomposed subproblems and estimating Hessian by its first-order estimation, we enable the computation complexity of \name{} to be $\mathcal{O}(n)$ in each round, which maintains the lowest among all existing federated meta-learning approaches.

\begin{algorithm}[t]
	\caption{\name{}}
	\label{alg}
	\LinesNumbered
	\KwIn{$\theta_p$, $\alpha$, $\lambda$, $\mathcal{D}^{s}_{1:|\mathcal{I}|}$, $\mathcal{D}^q_{1:|\mathcal{I}|}$, $\rho_{1:|\mathcal{I}|}$}
	Each client $i\in\mathcal{I}$ initializes $y^{-1}_i$\;
	Server initializes $\theta^0$ and sends it to all clients\;
	\For{$t=0$ \KwTo $T$}{
	    \For{$i=1$ \KwTo $|\mathcal{I}|$}{
	        Compute $\phi^t_i\leftarrow\theta^t-\alpha\nabla L_i(\theta^t,\mathcal{D}^{s}_i)$\;
	        Compute $\theta^t_i$ by \cref{eq:update_thetai_inex}\;
	        Compute $ y^t_i\leftarrow y^{t-1}_i + \rho_i(\theta^t_i - \theta^t)$\;
	        Send $\theta^t_i$ and $ y^t_i$ back to the server\;
	        }
	    Server computes $\theta^{t+1}$ by \cref{eq:update_theta} and sends it back to all clients\;
		}
		Server distributes $\theta^T$ to clients for fast adaptation\;
\end{algorithm}

\section{Performance Analysis}
\label{sec:analysis}

In this section, we establish performance guarantees for the proposed method. First, we study the convergence properties and characterize the communication complexity for \name{}. Then, we analyze the adaptation performance and quantify the benefit of knowledge transfer.

\subsection{Convergence Properties}

For convenience, we denote the objective function of \cref{prim_prob} as $F(\theta)$ in this section:
\begin{align}
    \label{F_definition}
    F(\theta)\doteq \sum_{i\in\mathcal{I}}w_i L_i(\phi_i(\theta),\mathcal{D}^{q}_i)+\lambda R_h(\theta,\theta_p)
\end{align}
with $\phi_i(\theta)=\theta-\alpha\nabla L_i(\theta,\mathcal{D}^{s}_i)$. Due to the potential non-convexity of $F(\cdot)$, we characterize the convergence and communication complexity of \name{} for finding a \textit{First-Order Stationary Point} (FOSP) of $F(\theta)$. Formally, the definition of an $\epsilon$-FOSP is given as follows.

\begin{definition}[$\epsilon$-FOSP]
Given $\epsilon>0$, a solution $\theta\in\mathbb{R}^n$ is called an $\epsilon$-approximate first-order stationary point ($\epsilon$-FOSP) of Problem (\ref{prim_prob})-(\ref{eq:phi}), if $\Vert\nabla F(\theta)\Vert\le\epsilon$.
\end{definition}

The above definition implies that if a solution $\theta$ obtained by an algorithm is an $\epsilon$-FOSP, then the gradient norm of the objective function is bounded above by $\epsilon$. 

Note that the first-order estimator of Hessian introduced in \cref{eq:update_thetai_inex} inevitably complicates the convergence analysis of \name{}, making the existing analysis methods of ADMM~\cite{barber2024convergence} not suitable here. Before developing a new technical path to establish the convergence guarantee, we impose the following assumptions. 

\begin{assumption}[Bounded objective]
    \label{lowerbounded}
    $F(\theta)$ is lower-bounded, \ie,  $F(\theta)>-\infty$, for all $\theta\in\mathbb{R}^n$.
\end{assumption}

\begin{assumption}[Smoothness and bounded gradients]
\label{Lsmooth}
Given any $i\in\mathcal{I}$, $\theta_p\in\mathbb{R}^m$ and $\mathcal{D}^{s}_i$, both $L_i(\cdot,\mathcal{D}^{s}_i)$ and $R_h(\cdot,\theta_p)$ are twice continuously differentiable and smooth, \ie, for any $\theta_1,\theta_2\in\mathbb{R}^n$, there exist $\mu_i>0$ and $\mu_r>0$ such that
\begin{align}
    \label{Lsmoothinequality}
    \Vert \nabla L_i(\theta_1,\mathcal{D}^{s}_i)-\nabla L_i(y,\mathcal{D}^{s}_i)\Vert&\le \mu_i\Vert \theta_1-\theta_2\Vert,\\
    \Vert \nabla_{\theta_1} R_h(\theta_1,\theta_p)-\nabla_{\theta_2} R_h(\theta_2,\theta_p)\Vert&\le \mu_r\Vert \theta_1-\theta_2\Vert.
\end{align}
Besides, for each $i\in\mathcal{I}$, the gradient norm of $L_i(\cdot,\mathcal{D}^{s}_i)$ is bounded by a positive constant $\beta_i>0$, \ie, for any $\theta\in\mathbb{R}^n$, the following fact holds:
\begin{align}
    \label{gradientbound}
    &\Vert\nabla L_i(\theta,\mathcal{D}^{s}_i)\Vert\leq \beta_i.
\end{align}
\end{assumption}
\begin{assumption}[Lipschitz continuous Hessian]
\label{HessianLipschitz}
Given $i\in\mathcal{I}$ and $\mathcal{D}^{s}_i$, the Hessian of $L_i(\cdot,\mathcal{D}^{s}_i)$ is $\zeta_i$-Lipschitz continuous, \ie, for any $\theta_1,\theta_2\in\mathbb{R}^n$, we have
\begin{align}
    \label{HessianLipschitzinequality}
    \Vert \nabla^2 L_i(\theta_1,\mathcal{D}^{s}_i)-\nabla^2 L_i(\theta_2,\mathcal{D}^{s}_i)\Vert\le \zeta_i\Vert \theta_1-\theta_2\Vert.
\end{align}
\end{assumption}

\begin{assumption}[Bounded variances]
\label{bounded_var}
For each $i\in\mathcal{I}$ and any $\theta\in\mathbb{R}^n$, the stochastic gradient $\nabla l_i(\theta,(\mathbf{x},\mathbf{y}))$ and Hessian $\nabla^2 l_i(\theta,(\mathbf{x},\mathbf{y}))$ with respect to $(\mathbf{x},\mathbf{y})\in\mathcal{X}_i\times\mathcal{Y}_i$ have bounded variances, \ie,
\begin{align}
    \mathbb{E}_{(\mathbf{x},\mathbf{y})\sim P_i}\{\Vert\nabla l_i(\theta,(\mathbf{x},\mathbf{y}))-\nabla L_i(\theta)\Vert^2\}&\le(\sigma^g_i)^2,\\
    \mathbb{E}_{(\mathbf{x},\mathbf{y})\sim P_i}\{\Vert\nabla^2 l_i(\theta,(\mathbf{x},\mathbf{y}))-\nabla^2 L_i(\theta)\Vert^2\}&\le(\sigma^h_i)^2,
\end{align}
for some positive constants $\sigma^g_i>0$ and $\sigma^h_i>0$.
\end{assumption}

Assumptions \ref{lowerbounded}--\ref{bounded_var} are standard and commonly used in the literature on the analysis of FL algorithms~\cite{lin2020collaborative,fallah2020personalized,zhang2021fedpd}. In particular, \cref{gradientbound} is critical for analyzing the convergence as it allows characterizing the estimation error of Hessian. \cref{HessianLipschitz} implies the high-order smoothness of $L_i(\cdot,\mathcal{D}^{s}_i)$ for dealing with the second-order information in the update steps of \cref{alg}. \cref{bounded_var} provides the upper bounds of the variances of the gradient and Hessian estimations. Next, we impose the assumption on hyper-parameters.

\begin{assumption}
\label{parameterassumption}
For all $i\in\mathcal{I}$, $\rho_i$ is sufficiently large such that the following facts hold:
\begin{align}
    \label{parameter_eq1}
    &\frac{\rho_i}{2}-4w_i\nu_i>0,\\
    \label{parameter_eq2}
    &\frac{\rho_i}{2}-2w^2_i\nu^2_i(\frac{4w_i\nu_i}{\rho^2_i}+\frac{1}{\rho_i})-\frac{\lambda\mu_r}{2|\mathcal{I}|}>0,\\
    \label{parameter_eq3}
    &\rho_i-3\nu_i>0,
\end{align}
where $\nu_i$ is a smooth scalar defined in \cref{lipschtz_nu}. In addition, for each $i\in\mathcal{I}$, the degree of freedom parameter $\{\delta_{i,t}\}$ for the approximation of Hessian-gradient products is chosen to be a monotonically non-increasing positive sequence and satisfies $\sum^\infty_{t=1}\delta_{i,t}<\infty$.
\end{assumption}

\cref{parameterassumption} constrains the values of penalty parameter $\rho_i$ and the degree of freedom parameter $\delta_{i,t}$. Intuitively, a large value of $\rho_i$ is required to balance the error caused by the linear approximation and the Hessian estimation in \cref{eq:update_thetai_inex}.


Denote $F_i(\theta)\doteq L_i(\phi_i(\theta),\mathcal{D}^q_i)$. Based on Assumptions \ref{Lsmooth} and \ref{HessianLipschitz}, we have the following result on the smoothness of $F_i$.

\begin{restatable}{lemma}{fsmoo}
\label{fsmoothlemma}
Given Assumptions \ref{Lsmooth} and \ref{HessianLipschitz}, for each $i\in\mathcal{I}$, $F_i$ is proper and $\nu_i$-smooth, \ie, 
\begin{align}
    \label{fsmooth}
    \Vert \nabla F_i(\theta_1)-\nabla F_i(\theta_2)\Vert\le \nu_i\Vert \theta_1-\theta_2\Vert,~\forall \theta_1,\theta_2\in\mathbb{R}^n
\end{align}
where $\nu_i$ is defined as follows:
\begin{align}
\label{lipschtz_nu}
    \nu_i\doteq (1+\alpha\mu_i)(1+\mu_i)\mu_i+\alpha\beta_i\zeta_i.
\end{align}
\end{restatable}

\begin{proof}
The proof can be found in Appendix.
\end{proof}

For conciseness, denote $F_i(\theta)\doteq L_i(\theta-\alpha\nabla L_i(\theta,\mathcal{D}^{s}_i),\mathcal{D}^q_i)$, $f_i(\theta)\doteq L_i(\theta,\mathcal{D}^{s}_i)$, and $f^q_i(\phi_i)\doteq L_i(\phi_i,\mathcal{D}^q_i)$. Building upon \cref{fsmoothlemma}, we next bound the variations of $y^t_i$ via the variations of $\theta^t$ in the following lemma.

\begin{lemma}
\label{dualbound}
Suppose that Assumptions \ref{lowerbounded}--\ref{HessianLipschitz} are satisfied. Then, the following fact holds:
\begin{align}
    \label{dualbound_ineq_1}
    \Vert y^{t+1}_i- y^t_i\Vert\le w_i\nu_i\Vert\theta^{t+1}-\theta^t\Vert + (\delta_{i,t}+\delta_{i,t+1})\alpha w_i\zeta_i\beta^2_i.
\end{align}
\end{lemma}
\begin{proof}
\label{proof_dualbound}
The proof can be found in Appendix.
\end{proof}

To bound the successive difference of the augmented Lagrangian function $\mathcal{L}(\{\theta^t_i,y^t_i\},\theta^t)$ (defined in \cref{eq:lagrangian}), we first bound the successive difference of $\mathcal{L}_i(\theta,\theta^{t+1}, y^t_i)$, defined as 
\begin{align*}
    \mathcal{L}_i(\theta,\theta^{t+1}, y^t_i)\doteq w_i F_i(\theta)+\langle y^t _i,\theta-\theta^{t+1}\rangle+\frac{\rho_i}{2}\Vert \theta-\theta^{t+1}\Vert^2,
\end{align*}
in the following lemma.

\begin{lemma}
\label{l_descent}
Suppose that Assumptions \ref{lowerbounded}-\ref{HessianLipschitz} are satisfied. The following fact holds:
\begin{align*}
    \nonumber
    &\mathcal{L}_i(\theta^{t+1}_i,\theta^{t+1}, y^t_i)- \mathcal{L}_i(\theta^t_i,\theta^{t+1}, y^t_i)\\
     \le\,& \frac{2(1+w_i)\nu_i}{\rho^2_i}\Vert  y^{t+1}_i- y^t_i\Vert^2
     -\frac{\rho_i-(3+4w_i)\nu_i}{2}\Vert\theta^{t+1}_i-\theta^t_i\Vert^2\nonumber\\
     +\,&\frac{2\alpha w_i\zeta_i\beta^2_i\delta_{i,t+1}}{\rho_i}\Vert  y^{t+1}_i- y^t_i\Vert+\alpha w_i\zeta_i\beta^2_i\delta_{i,t+1}\Vert \theta^t_i-\theta^{t+1}_i\Vert.
\end{align*}
\end{lemma}
\begin{proof}
\label{proof_l_descent}
The proof can be found in Appendix.
\end{proof}

Based on Lemma \ref{l_descent}, we proceed to derive the successive difference of the augmented Lagrangian function $\mathcal{L}(\{\theta^t_i,y^t_i\},\theta^t)$ in the following lemma. Notably, due to the error induced by linear approximation and first-order Hessian estimation, the \textit{sufficient descent} property does not hold therein. 

\begin{lemma}
\label{L_descent}
Under Assumptions \ref{lowerbounded}--\ref{HessianLipschitz}, the following fact holds:
\begin{align}
    \label{L_descent_ineq}
    \nonumber
    &\,\mathcal{L}(\{\theta^{t+1}_i,y^{t+1}_i\},\theta^{t+1})-\mathcal{L}(\{\theta^t_i,y^t_i\},\theta^t)\\
    \nonumber
    \le&-\sum\nolimits_{i\in\mathcal{I}}(a_{i,e}\Vert\theta^{t+1}_i-\theta^t_i\Vert^2+a_{i,p}\Vert\theta^{t+1}-\theta^t\Vert^2\\
    &-b^{t+1}_{i,e}\Vert\theta^{t+1}_i-\theta^t_i\Vert-b^{t+1}_{i,p}\Vert\theta^{t+1}-\theta^t\Vert-c^{t+1}_i),
\end{align}
where $a_{i,e}$ and $a_{i,p}$ are defined in \cref{parameter_eq1,parameter_eq2}, respectively. $b^{t+1}_{i,e}$, $b^{t+1}_{i,e}$, and $c^{t+1}_i$ are defined as follows:
\begin{align}
    \label{parameter_bp}
    b^{t+1}_{i,p}\doteq \;&\frac{2\alpha\nu_i w^2_i\zeta_i\beta^2_i\delta_{i,t+1}}{\rho_i},\\
    \label{parameter_c}
    c^{t+1}_i\doteq\;&2(\delta_{i,t}+\delta_{i,t+1})^2(\alpha w_i\zeta_i\beta^2_i)^2(4w_i\nu_i/\rho^2_i+\frac{1}{\rho_i})\nonumber\\
    &+\frac{2(\alpha w_i\zeta_i\beta^2_i)^2\delta_{i,t+1}}{\rho_i}(\delta_{i,t}+\delta_{i,t+1}),\\
    \label{parameter_be}
    b^{t+1}_{i,e}\doteq\;& \alpha w_i\zeta_i\beta^2_i\delta_{i,t+1}.
\end{align}
\end{lemma}
\begin{proof}
\label{proof_L_descent}
The proof can be found in Appendix.
\end{proof}

In the next lemma, we show that the augmented Lagrangian function $\mathcal{L}(\{\theta^t_i,y^t_i\},\theta^t)$ is lower bounded for $t\in\mathbb{N}$.

\begin{lemma}
\label{lagrangian_lowerbound}
Suppose that Assumptions \ref{lowerbounded}--\ref{parameterassumption} hold. Under \cref{alg}, the augmented Lagrangian function defined in \cref{eq:lagrangian} is lower bounded.
\end{lemma}
\begin{proof}
\label{proof_lagrangian_lowerbound}
The proof can be found in Appendix.
\end{proof}

Now, we are ready to establish the convergence and characterize the communication complexity for \cref{alg}.

\begin{restatable}[Convergence]{theorem}{conv}
\label{convergence}
Under Assumptions \ref{lowerbounded}--\ref{parameterassumption}, the following facts hold for \cref{alg}:
\begin{enumerate}[(i)]
    \item The sequence of $\theta_t$ has at least one limit point and each limit point $\theta^*$ is a stationary solution of Problem (\ref{prim_prob})-(\ref{eq:phi}), that is, $\Vert \nabla F(\theta^*)\Vert=0$.
    \item \cref{alg} finds an $\epsilon$-FOSP of Problem (\ref{prim_prob})-(\ref{eq:phi}) after at most $\mathcal{O}(1/\epsilon^2)$ communication rounds.
\end{enumerate}
\end{restatable}

\begin{proof}
First, we prove part (i). Recall that the RHS of \cref{L_descent_ineq} is the sum of independent quadratic functions of $\Vert\theta^{t+1}_i-\theta^t_i\Vert$ and $\Vert\theta^{t+1}-\theta^t\Vert$. Due to \cref{parameterassumption} and \cref{L_descent}, for each $i\in\mathcal{I}$, based on the form of the roots of quadratic functions, it is easy to see that there exist $\sigma^{t+1}_i$ and $\gamma^{t+1}_i$ such that
\begin{align}
\label{convergence_eq1}
    \lim_{t\rightarrow\infty}\sigma^{t+1}_i=0,\quad\lim_{t\rightarrow\infty}\gamma^{t+1}_i=0.
\end{align}
When $\Vert\theta^{t+1}_i-\theta^t_i\Vert>\sigma^{t+1}_i$, we have
\begin{align}
\label{convergence_eq11}
    a_{i,e}\Vert\theta^{t+1}_i-\theta^t_i\Vert^2-b^{t+1}_{i,e}\Vert\theta^{t+1}_i-\theta^t_i\Vert-c^{t+1}_i>0;
\end{align}
and when $\Vert\theta^{t+1}-\theta^t\Vert>\gamma^{t+1}_i$, the following holds:
\begin{align}
\label{convergence_eq12}
    a_{i,p}\Vert\theta^{t+1}-\theta^t\Vert^2-b^{t+1}_{i,p}\Vert\theta^{t+1}-\theta^t\Vert>0.
\end{align}
We show $\lim_{t\rightarrow\infty}\Vert\theta^{t+1}_i-\theta^t_i\Vert=0$, $\lim_{t\rightarrow\infty}\Vert\theta^{t+1}-\theta^t\Vert=0$ in the following two steps.
\begin{enumerate}[1)]
    \item Suppose that there exists $T\ge 0$ such that for all $t\ge T$, the following is true: 
    \begin{align*}
        \sum\nolimits_{i\in\mathcal{I}}(a_{i,e}\Vert\theta^{t+1}_i-\theta^t_i\Vert^2+a_{i,p}\Vert\theta^{t+1}-\theta^t\Vert^2-c^{t+1}_i\nonumber\\
        -b^{t+1}_{i,e}\Vert\theta^{t+1}_i-\theta^t_i\Vert-b^{t+1}_{i,p}\Vert\theta^{t+1}-\theta^t\Vert)>0.
    \end{align*}
    Then, under \cref{parameterassumption}, using Lemmas \ref{L_descent}--\ref{lagrangian_lowerbound}, the value of the augmented Lagrangian function $\mathcal{L}(\{\theta^{t+1}_i,y^{t+1}_i\},\theta^{t+1})$ will monotonically decrease and converges. Thus, we have
    \begin{align*}
        \lim_{t\rightarrow\infty}\sum\nolimits_{i\in\mathcal{I}}(a_{i,e}\Vert\theta^{t+1}_i-\theta^t_i\Vert^2-b^{t+1}_{i,e}\Vert\theta^{t+1}_i-\theta^t_i\Vert-c^{t+1}_i\\
        +a_{i,p}\Vert\theta^{t+1}-\theta^t\Vert^2-b^{t+1}_{i,p}\Vert\theta^{t+1}-\theta^t\Vert)=0,
    \end{align*}
    which implies that $\Vert\theta^{t+1}_i-\theta^t_i\Vert$ and $\Vert\theta^{t+1}-\theta^t\Vert$ converge to the positive roots of corresponding quadratic functions, \ie, LHS of \cref{convergence_eq11,convergence_eq12}; otherwise, the limitation will not be 0. Due to \cref{convergence_eq1}, the positive roots of the above quadratic function converge to 0, which implies
    \begin{align}
        \label{convergence_eq2}
        &\lim_{t\rightarrow\infty}\Vert\theta^{t+1}_i-\theta^t_i\Vert=0,~\forall~i\in\mathcal{I},\\
        \label{convergence_eq3}
        &\lim_{t\rightarrow\infty}\Vert\theta^{t+1}-\theta^t\Vert=0.
    \end{align}
    By \cref{dualbound} and \cref{eq:update_y}, we obtain
    \begin{align}
        \label{convergence_eq4}
        \lim_{t\rightarrow\infty}\Vert y^{t+1}_i- y^t_i\Vert&=0,~\forall~i\in\mathcal{I},\\
        \label{convergence_eq5}
        \lim_{t\rightarrow\infty}\Vert\theta^{t+1}_i-\theta^{t+1}\Vert&=0,~\forall~i\in\mathcal{I}.
    \end{align}
    \item Suppose that there exists a sequence $\{t_j\mid j\in\mathbb{N}\}$ such that
    \begin{align}
        \label{convergence_eq14}
        \sum\nolimits_{i\in\mathcal{I}}(a_{i,e}\Vert\theta^{t_j+1}_i-\theta^{t_j}_i\Vert^2+a_{i,p}\Vert\theta^{t_j+1}-\theta^{t_j}\Vert^2-c^{t_j+1}_i\nonumber\\
        -b^{t_j+1}_{i,e}\Vert\theta^{t_j+1}_i-\theta^{t_j}_i\Vert-b^{t_j+1}_{i,p}\Vert\theta^{t_j+1}-\theta^{t_j}\Vert)\le0.
    \end{align}
    Due to \cref{parameterassumption}, the minimum value of the above quadratic function converges to 0, which implies
    \begin{align*}
        \lim_{t\rightarrow\infty}\sum\nolimits_{i\in\mathcal{I}}(a_{i,e}\Vert\theta^{t_j+1}_i-\theta^{t_j}_i\Vert^2+a_{i,p}\Vert\theta^{t_j+1}-\theta^{t_j}\Vert^2\\
        -c^{t_j+1}_i-b^{t_j+1}_{i,e}\Vert\theta^{t_j+1}_i-\theta^{t_j}_i\Vert-b^{t_j+1}_{i,p}\Vert\theta^{t_j+1}-\theta^{t_j}\Vert)=0.
    \end{align*}
    Analogously to \cref{convergence_eq2,convergence_eq3}, for all $i\in\mathcal{I}$, we have 
    \begin{align*}
        \lim_{t\rightarrow\infty}\Vert\theta^{t_j+1}_i-\theta^{t_j}_i\Vert=0,\quad\lim_{t\rightarrow\infty}\Vert\theta^{t_j+1}-\theta^{t_j}\Vert=0.
    \end{align*}
    Slightly abusing notation, we define the auxiliary  sequence $\{t_q\mid q\in\mathbb{N}\}\doteq\mathbb{N}-\{t_j\mid j\in\mathbb{N}\}$. Note that
    \begin{align}
        \label{convergence_eq15}
        \sum\nolimits_{i\in\mathcal{I}}(a_{i,e}\Vert\theta^{t_q+1}_i-\theta^{t_q}_i\Vert^2+a_{i,p}\Vert\theta^{t_q+1}-\theta^{t_q}\Vert^2-c^{t_q+1}_i\nonumber\\
        -b^{t_q+1}_{i,e}\Vert\theta^{t_q+1}_i-\theta^{t_q}_i\Vert-b^{t_q+1}_{i,p}\Vert\theta^{t_q+1}-\theta^{t_q}\Vert)>0.
    \end{align}
    Similar to 1), for all $i\in\mathcal{I}$, we have
    \begin{align*}
        \lim_{t\rightarrow\infty}\Vert\theta^{t_q+1}-\theta^{t_q}\Vert=0,\quad\lim_{t\rightarrow\infty}\Vert\theta^{t_q+1}_i-\theta^{t_q}_i\Vert=0.
    \end{align*}
    Based on the above observations, for any $\eta>0$, there exists $k\ge 0$ such that when $j>k$ and $q>k$, the following facts hold for all $i\in\mathcal{I}$:
    \begin{align*}
        &\Vert\theta^{t_j+1}_i-\theta^{t_j}_i\Vert\le\eta,\quad
        \Vert\theta^{t_j+1}-\theta^{t_j}\Vert\le\eta,\\
        &\Vert\theta^{t_q+1}_i-\theta^{t_q}_i\Vert\le\eta,\quad\Vert\theta^{t_q+1}-\theta^{t_q}\Vert\le\eta.
    \end{align*}
    Thus, for any $t>t_k$ and $i\in\mathcal{I}$, we can write
    \begin{align*}
        \Vert\theta^{t+1}_i-\theta^t_i\Vert\le\eta,\quad\Vert\theta^{t+1}-\theta^t\Vert\le\eta,
    \end{align*}
    which implies that 
    \begin{align*}
        \lim_{t\rightarrow\infty}\Vert\theta^{t+1}_i-\theta^t_i\Vert=0,\quad\lim_{t\rightarrow\infty}\Vert\theta^{t+1}-\theta^t\Vert=0,
    \end{align*}
    and therefore \cref{convergence_eq2,convergence_eq5} hold.
\end{enumerate}

Using the optimality condition of \cref{eq:update_thetai_inex} leads to
\begin{align}
    \label{convergence_eq6}
    w_i\tilde{\nabla}F_i(\theta^{t+1})+ y^{t+1}_i=0,
\end{align}
where $\tilde{\nabla}F_i(\theta^{t+1})\doteq \nabla f^q_i(\phi^{t+1}_i)-\alpha g^{t+1}_i$. For each $i\in\mathcal{I}$, we derive an upper bound of $\Vert\nabla_{\theta^{t+1}_i}\mathcal{L}(\{\theta^{t+1}_i,y^{t+1}_i\},\theta^{t+1})\Vert$ as
\begin{align}
    \label{convergence_eq7}
    \nonumber
    &\Vert\nabla_{\theta^{t+1}_i}\mathcal{L}(\{\theta^{t+1}_i,y^{t+1}_i\},\theta^{t+1})\Vert\\
    \nonumber
    =\,&\Vert w_i\nabla F_i(\theta^{t+1}_i)+ y^{t+1}_i+\rho_i(\theta^{t+1}_i-\theta^{t+1})\Vert\\
    \nonumber
    \le\,&\Vert w_i\nabla  F_i(\theta^{t+1}_i)-w_i\tilde{\nabla}F_i(\theta^{t+1})\Vert+\rho_i\Vert\theta^{t+1}_i-\theta^{t+1}\Vert\\
    \nonumber
    &+\Vert w_i\tilde{\nabla}F_i(\theta^{t+1})+ y^{t+1}_i\Vert\\
    \nonumber
    \le\,& \Vert w_i\nabla  F_i(\theta^{t+1}_i)-w_i\nabla F_i(\theta^{t+1})\Vert+\Vert w_i\nabla F_i(\theta^{t+1})\\
    \nonumber
    &-w_i\tilde{\nabla}F_i(\theta^{t+1})\Vert+\rho_i\Vert\theta^{t+1}_i-\theta^{t+1}\Vert\\
    \nonumber
    \le\,&\Vert w_i\nabla  F_i(\theta^{t+1}_i)-w_i\nabla F_i(\theta^{t+1})\Vert+\rho_i\Vert\theta^{t+1}_i-\theta^{t+1}\Vert\\
    &+\alpha w_i\zeta_i\beta^2_i\delta_{i,t+1}.
\end{align}
Taking limitation of $t\rightarrow\infty$ on both sides of \cref{convergence_eq7} and using \cref{parameterassumption} and \cref{convergence_eq5} yields
\begin{align}
    \label{convergence_eq8}
    \Vert\nabla_{\theta^*_i}\mathcal{L}(\{\theta^*_i,y^*_i\},\theta^*)\Vert=0,~\forall~i\in\mathcal{I}.
\end{align}
Note that
\begin{align}
\label{convergence_eq9}
    \Vert\nabla_{\theta^{t+1}}\mathcal{L}(\{\theta^{t+1}_i,y^{t+1}_i\},\theta^{t+1})\Vert\le \sum\nolimits_{i\in\mathcal{I}}(\rho_i\Vert\theta^{t+1}_i-\theta^t_i\Vert\nonumber\\
    +w_i\nu_i\Vert\theta^{t+1}-\theta^t\Vert+2\alpha w_i\zeta_i\beta^2_i\delta_{i,t}).
\end{align}
Similarly, we obtain
\begin{align}
\label{convergence_eq10}
    \nabla_{\theta^*}\mathcal{L}(\{\theta^*_i,y^*_i\},\theta^*)=0.
\end{align}
Finally, we bound $\Vert \sum_{i\in\mathcal{I}}w_i\nabla F_i(\theta^{t+1})+\lambda \nabla R_h(\theta^{t+1},\theta_p)\Vert$:
\begin{align}
    \label{convergence_eq13}
    \nonumber
    &\Vert \sum\nolimits_{i\in\mathcal{I}}w_i\nabla F_i(\theta^{t+1})+\lambda \nabla R_h(\theta^{t+1},\theta_p)\Vert\\
    \nonumber
    \le\,& \Vert \sum\nolimits_{i\in\mathcal{I}}w_i\nabla F_i(\theta^{t+1})+\lambda \nabla R_h(\theta^{t+1},\theta_p)\\
    \nonumber
    &-\sum\nolimits_{i\in\mathcal{I}}\nabla_{\theta^{t+1}_i}\mathcal{L}(\{\theta^{t+1}_i,y^t_i\},\theta^{t+1})\Vert\\
    \nonumber
    &+\Vert\sum\nolimits_{i\in\mathcal{I}}\nabla_{\theta^{t+1}_i}\mathcal{L}(\{\theta^{t+1}_i,y^t_i\},\theta^{t+1}) \Vert\\
    \nonumber
    =\,&\Vert\sum\nolimits_{i\in\mathcal{I}}w_i(\nabla F_i(\theta^{t+1})-\nabla F_i(\theta^{t+1}_i))+\lambda \nabla R_h(\theta^{t+1},\theta_p)\\
    \nonumber
    &-\sum\nolimits_{i\in\mathcal{I}}(y^t_i+\rho_i(\theta^{t+1}_i-\theta^{t+1}))\Vert\\
    \nonumber
    &+\Vert\sum\nolimits_{i\in\mathcal{I}}\nabla_{\theta^{t+1}_i}\mathcal{L}(\{\theta^{t+1}_i,y^t_i\},\theta^{t+1}) \Vert\\
    \nonumber
    \le\,& \sum\nolimits_{i\in\mathcal{I}}w_i\nu_i\Vert\theta^{t+1}-\theta^{t+1}_i\Vert+\Vert \nabla_{\theta^{t+1}}\mathcal{L}(\{\theta^{t+1}_i,y^t_i\},\theta^{t+1})\Vert\\
    &+\sum\nolimits_{i\in\mathcal{I}}\Vert\nabla_{\theta^{t+1}_i}\mathcal{L}(\{\theta^{t+1}_i,y^t_i\},\theta^{t+1})\Vert.
\end{align}
Taking limitation on the both sides of \cref{convergence_eq13} by $t\rightarrow\infty$ and combining \cref{convergence_eq5,convergence_eq8,convergence_eq10} yields part (i).

Next, we proceed to prove part (ii). Summing up \cref{L_descent_ineq} from $t=0$ to $T$ and taking a limitation on $T$, there exist some positive constants $a_2$ and $a_1$ corresponding to $\rho_i$ such that
\begin{align}
    \sum\nolimits^\infty_{t=0}z^t\le \mathcal{L}(\{\theta^0_i,y^0_i\},\theta^0)-\mathcal{L}(\{\theta^*_i,y^*_i\},\theta^*)<\infty
\end{align}
where $z^t$ is denoted by 
\begin{align*}
    z^t&\doteq \underbrace{a_2\sum\nolimits_{i\in\mathcal{I}}(\Vert\theta^{t+1}_i-\theta^t_i\Vert^2+\Vert\theta^{t+1}-\theta^t\Vert^2)}_{z^t_2}\\
    &-\underbrace{a_1\sum\nolimits_{i\in\mathcal{I}}(\delta_{i,t+1}(\Vert\theta^{t+1}_i-\theta^t_i\Vert+\Vert\theta^{t+1}-\theta^t\Vert)+\delta^2_{i,t})}_{z^t_1}.
\end{align*}
$z^t_2$, $z^t_1$ are denoted as the first and second sum terms, respectively. From \cref{parameterassumption} and \cref{convergence}, it is easy to see that there exists some positive constant $a_3$ such that the following holds:
\begin{align*}
    \sum^{\infty}_{t=0}z^t_1    &=a_1\sum^{\infty}_{t=0}\sum_{i\in\mathcal{I}}(\delta_{i,t+1}(\Vert\theta^{t+1}_i-\theta^t_i\Vert+\Vert\theta^{t+1}-\theta^t\Vert)+\delta^2_{i,t})\\
    &\le a_1\sum_{i\in\mathcal{I}}(\sum^{\infty}_{t=0}2\delta_{i,t+1}+\sum^{\infty}_{t=0}\delta^2_{i,t})+a_3\\
    &< \infty.
\end{align*}
Hence, we have
\begin{align}
    \sum^{\infty}_{t=1}z^t_2\le b<\infty,~\text{for some constant}~b>0.
\end{align}
Due to $\sum^{\infty}_{t=0}z^t_1<\infty$, it is clear that the augmented Lagrangian function is upper bounded and $\theta^t$ is finite, implying that $\{\theta^t\}$ has at least one limit point. Denote
\begin{align*}
    T^2(\epsilon)&\doteq\min\{t\mid \Vert\theta^{t+1}-\theta^t\Vert^2\le\epsilon,t\ge0\},\\
    T^2_i(\epsilon)&\doteq\min\{t\mid \Vert\theta^{t+1}_i-\theta^t_i\Vert^2\le\epsilon,t\ge0\}.
\end{align*}
Then, we have
\begin{align}
    \label{complexity_eq1}
    &a_2 T^2(\epsilon)\epsilon\le\sum^{\infty}_{t=1}z^t_2\le b,\\
    \label{complexity_eq2}
    &a_2 T^2_i(\epsilon)\epsilon\le\sum^{\infty}_{t=1}z^t_2\le b.
\end{align}
That is, $T^2(\epsilon)=\mathcal{O}(1/\epsilon)$ and $T^2_i(\epsilon)=\mathcal{O}(1/\epsilon)$ hold. Further, we represent
\begin{align*}
    T(\epsilon)&\doteq\min\{t\mid\Vert\theta^{t+1}-\theta^t\Vert\le\epsilon,t\ge0\},\\
    T_i(\epsilon)&\doteq\min\{t\mid\Vert\theta^{t+1}_i-\theta^t_i\Vert\le\epsilon,t\ge0\}.
\end{align*}
Based on \cref{complexity_eq1,complexity_eq2}, we have $T(\epsilon)=\mathcal{O}(1/\epsilon^2)$ and $T_i(\epsilon)=\mathcal{O}(1/\epsilon^2)$. Due to \cref{parameterassumption}, combining \cref{dualbound_ineq_1,convergence_eq7} yields
\begin{align*}
    \Vert\nabla_{\theta^{t+1}_i}\mathcal{L}(\{\theta^{t+1}_i,y^{t+1}_i\},\theta^{t+1})\Vert
    \le(w_i\nu_i+\rho_i)w_i\nu_i/\rho_i\\
    \cdot\Vert\theta^{t+1}-\theta^t\Vert+(2w_i\nu_i+3\rho_i)\alpha w_i\zeta_i\beta^2_i\delta_{i,t}/\rho_i.
\end{align*}
Similarly, it is easy to see that the convergence rate of $\delta_{i,t}$ is $\mathcal{O}(1/\epsilon)$. Therefore, for any $\epsilon>0$, \cref{alg} finds a solution $(\{\theta_i\},\theta, y)$ with $\Vert\nabla_{\theta_i}\mathcal{L}(\{\theta_i,y_i\},\theta)\Vert\le\epsilon$ after at most $\mathcal{O}(1/\epsilon^2)$ communication rounds. In the same way, it can be shown that 
\begin{align*}
     \Vert\nabla_{\theta^{t+1}}\mathcal{L}(\{\theta^{t+1}_i,y^{t+1}_i\},\theta^{t+1})\Vert\le \sum\nolimits_{i\in\mathcal{I}}(\rho_i\Vert\theta^{t+1}_i-\theta^t_i\Vert\\
     +w_i\nu_i\Vert\theta^{t+1}-\theta^t\Vert+2\alpha w_i\zeta_i\beta^2_i\delta_{i,t}).
\end{align*}
It implies the complexity w.r.t.  $\Vert\nabla_{\theta^{t+1}}\mathcal{L}(\{\theta^{t+1}_i,y^{t+1}_i\},\theta^{t+1})\Vert$ is consistent with  $\Vert\nabla_{\theta^{t+1}_i}\mathcal{L}(\{\theta^{t+1}_i,y^{t+1}_i\},\theta^{t+1})\Vert$. Moreover, it can be easily shown that $\Vert\nabla_{y_i}\mathcal{L}(\{\theta_i,y_i\},\theta)\Vert$ also enjoys the same communication complexity. Combining with \cref{convergence_eq13}, it completes the proof.
\end{proof}

\cref{convergence} indicates that \cref{alg} always converges to a stationary point of Problem (\ref{prim_prob})-(\ref{eq:phi}), and it requires $\mathcal{O}(1/\epsilon^2)$ communication rounds between the clients and the server to find an $\epsilon$-FOSP of the problem. It is worth noting that in contrast to the previous methods \cite{lin2020collaborative,fallah2020personalized,fallah2020convergence}, \name{} can converge under mild conditions, \ie, not depending on the similarity assumptions (\cref{similarity}) across different clients. This implies that \cref{alg} is well-suited for unbalanced local datasets, revealing its great potential in dealing with the inherent heterogeneity in FL.

Further, to characterize the impact of sample sizes on the expected performance, we provide the following corollary.

\begin{corollary}
\label{coro:data_impact}
Given Assumptions \ref{lowerbounded}--\ref{parameterassumption}, the $\epsilon$-FOSP solution $\theta_{\epsilon}$ found by \cref{alg} satisfies that
\begin{align}
    \nonumber
    &\mathbb{E}\Big\{\Big\Vert \sum\nolimits_{i\in\mathcal{I}}w_i L_i(\theta_{\epsilon}-\alpha\nabla L_i(\theta_{\epsilon}))+\lambda R_h(\theta_{\epsilon},\theta_p)\Big\Vert\Big\}\\
    &\le\epsilon+\sum\nolimits_{i\in\mathcal{I}}w_i\sigma^g_i(\frac{\alpha\mu_i}{\sqrt{D^{s}_i}}+\frac{1}{\sqrt{D^{q}_i}})
\end{align}
where $w_i=1/|\mathcal{I}|$ and $L_i(\cdot)$ is the expected loss denoted by:
\begin{align}
    \label{eq:expected_loss}
    L_i(\theta)\doteq\mathbb{E}_{(\mathbf{x},\mathbf{y})\sim P_i}\left\{l_i(\theta,(\mathbf{x},\mathbf{y}))\right\}.
\end{align}
\end{corollary}

\begin{proof}
First, we prove the following result
\begin{align}
    \label{eq:coro_1}
    &\mathbb{E}\{\Vert \nabla L_i(\theta-\alpha\nabla L_i(\theta))-\nabla L_i(\theta-\alpha\nabla L_i(\theta,\mathcal{D}^s_i),\mathcal{D}^q_i)\Vert\}\nonumber\\
    &\le \frac{\alpha\mu_i\sigma^g_i}{\sqrt{D^s_i}}+\frac{\sigma^g_i}{\sqrt{D^q_i}}.
\end{align}
We can write
{\small
\begin{align}
    \label{eq:eg_1}
    &\,\mathbb{E}\{\Vert \nabla L_i(\theta-\alpha\nabla L_i(\theta))-\nabla L_i(\theta-\alpha\nabla L_i(\theta,\mathcal{D}^s_i),\mathcal{D}^q_i)\Vert\}\nonumber\\
    \le&\underbrace{\mathbb{E}\{\Vert \nabla L_i(\theta-\alpha\nabla L_i(\theta))-\nabla L_i(\theta-\alpha\nabla L_i(\theta,\mathcal{D}^s_i))\Vert\}}_\text{(a)}\nonumber\\
    +&\underbrace{\mathbb{E}\{\Vert \nabla L_i(\theta-\alpha\nabla L_i(\theta,\mathcal{D}^s_i))-\nabla L_i(\theta-\alpha\nabla L_i(\theta,\mathcal{D}^s_i),\mathcal{D}^q_i)\Vert\}}_\text{(b)}.
\end{align}}

For (b), due to \cref{bounded_var}, we have
\begin{align}
    \label{eq:eg_2}
    \text{(c)}&\le\sqrt{\mathbb{E}\{\Vert \nabla L_i(\phi_i)-\frac{1}{D^q_i}\sum\nolimits^{D^q_i}_{j=1}\nabla l(\phi_i,(\mathbf{x}^j_i,\mathbf{y}^j_i))\Vert^2\}}\nonumber\\
    &=\sqrt{\frac{1}{(D^q_i)^2}\sum\nolimits^{D^q_i}_{j=1}\mathbb{E}\{\Vert \nabla L_i(\phi_i)-\nabla l(\phi_i,(\mathbf{x}^j_i,\mathbf{y}^j_i))\Vert^2\}}\nonumber\\
    &\le\frac{\sigma^g_i}{\sqrt{D^q_i}},
\end{align}
where $\phi_i=\theta-\alpha\nabla L_i(\theta,\mathcal{D}^s_i)$. 

For (a), based on \cref{Lsmooth}, it can be bounded by
\begin{align}
    \label{eq:eg_3}
    \text{(b)}\le\mathbb{E}\{\alpha\mu_i\Vert \nabla L_i(\theta,\mathcal{D}^s_i)-\nabla L_i(\theta)\Vert\}
    \le\frac{\alpha\mu_i\sigma^g_i}{\sqrt{D^s_i}}.
\end{align}
Plugging \cref{eq:eg_2,eq:eg_3} into \eqref{eq:eg_1}, \cref{eq:coro_1} holds. Based on \cref{eq:coro_1}, the following fact holds:
\begin{align}
    \label{eq:convergence_expected}
    \mathbb{E}\{\Vert \sum\nolimits_{i\in\mathcal{I}} w_i \nabla L_i(\theta_{\epsilon}-\alpha\nabla L_i(\theta_{\epsilon}))+\lambda \nabla R_h(\theta_{\epsilon},\theta_p)\Vert\}\nonumber\\
    \le \sum\nolimits_{i\in\mathcal{I}}w_i \mathbb{E}\{\Vert \nabla L_i(\theta_{\epsilon}-\alpha\nabla L_i(\theta_{\epsilon}))\nonumber\\
    -\nabla L_i(\theta_{\epsilon}-\alpha\nabla L_i(\theta_{\epsilon},\mathcal{D}^s_i),\mathcal{D}^q_i)\Vert\}\nonumber\\
    +\mathbb{E}\{\Vert \sum\nolimits_{i\in\mathcal{I}} w_i \nabla L_i(\theta_{\epsilon}-\alpha\nabla L_i(\theta_{\epsilon},\mathcal{D}^s_i),\mathcal{D}^q_i)\nonumber\\
    +\lambda \nabla R_h(\theta_{\epsilon},\theta_p)\Vert\}\nonumber\\
    \le\epsilon+ \sum\nolimits_{i\in\mathcal{I}}w_i\sigma^g_i(\frac{\alpha\mu_i}{\sqrt{D^s_i}}+\frac{1}{\sqrt{D^q_i}}),
\end{align}
thereby completing the proof.
\end{proof}

\cref{coro:data_impact} implies that while $\epsilon$-FOSP of the deterministic loss \cref{prim_prob} can be obtained within $\mathcal{O}(1/\epsilon^2)$ rounds, performance degradation may occur due to small sample sizes.

\subsection{Adaptation Performance}

In this section, we quantify the adaptation performance of the learned global meta-model at the client side. We represent $i'$ as a target client with $\mathcal{D}_{i'}$ the local dataset ($i'$ does not have to be included in $\mathcal{I}$). Analogously, $\mathcal{D}_{i'}$ is divided into support set $\mathcal{D}^s_{i'}$ and query set $\mathcal{D}^s_{i'}$. Before proceeding, we impose the following assumption on task similarity.

\begin{assumption}[Task similarity]
\label{similarity}
There exist positive constants $\psi^g_i>0$ and $\psi^h_i>0$ such that for any $i\in\mathcal{I}$ and $\theta\in\mathbb{R}^n$, the following holds:
\begin{align}
    \Vert\nabla L_{i'}(\theta)-\nabla L_i(\theta)&\Vert\le \psi^g_i\\
    \Vert\nabla^2 L_{i'}(\theta)-\nabla^2 L_i(\theta)&\Vert\le \psi^h_i
\end{align}
where the expected loss $L_i(\cdot)$ is defined in \cref{eq:expected_loss}, and the same applies to $L_{i'}(\cdot)$.
\end{assumption}

\cref{similarity} indicates that the variations of the gradients between clients are bounded by some constants, which quantifies the heterogeneity across clients corresponding to non-IID data and hold for many practical loss functions, such as logistic regression and hyperbolic tangent functions \cite{zhang2021fedpd}. In particular, $\psi^g_i$ and $\psi^h_i$ can be roughly seen as a distance between data distributions $P_{i'}$ and $P_i$ \cite{fallah2020convergence}.

Next, we present the quantitative result on the adaptation performance.

\begin{restatable}[Adaptation performance]{theorem}{perf}
\label{performance}
Suppose that Assumptions \ref{lowerbounded}--\ref{similarity} hold for all $i\in\mathcal{I}$, and Assumptions \ref{Lsmooth} and \ref{bounded_var} hold for $i'$. For any $\epsilon>0$, the $\epsilon$-FOSP solution $\theta_{\epsilon}$ found by \cref{alg} satisfies that
\begin{align}
\label{performance_ineq}
\mathbb{E}\{\Vert \nabla F_{i'}(\theta_{\epsilon})\Vert\}&\le\epsilon+\alpha\beta_{i'}\sum_{i\in\mathcal{I}}w_i\psi^h_i+(\alpha\mu+1)^2\sum_{i\in\mathcal{I}}w_i\psi^g_i\nonumber\\
\nonumber
&+\alpha\mu(\alpha\mu+1)\sum_{i\in\mathcal{I}}w_i\sigma^g_i(\frac{1}{\sqrt{D^{q}_{i'}}}+\frac{1}{\sqrt{D^{q}_i}})\\
\nonumber
&+(\alpha\mu+1)\sum_{i\in\mathcal{I}}w_i\sigma^g_i(\frac{1}{\sqrt{D^{s}_{i'}}}+\frac{1}{\sqrt{D^{s}_i}})\\
&+\alpha\beta_{i'}\sum_{i\in\mathcal{I}}w_i\sigma^h_i(\frac{1}{\sqrt{D^{s}_{i'}}}+\frac{1}{\sqrt{D^{s}_i}})
\end{align}
where $F_{i'}(\theta)\doteq L_{i'}(\theta-\alpha\nabla L_{i'}(\theta,\mathcal{D}^{s}_{i'}),\mathcal{D}^q_{i'})+\lambda R_h(\theta,\theta_p)$ for any $\mathcal{D}^{s}_{i'}$ and $\mathcal{D}^q_{i'}$ w.r.t. distribution $P_{i'}$, and $\mu\doteq\max_{i\in\mathcal{I}}\{\mu_i\}$.
\end{restatable}

\begin{proof}
Since $\theta_{\epsilon}$ denotes the $\epsilon$-FOSP, \ie, 
\begin{align}
    \label{performance_eq2}
    &\Vert \sum\nolimits_{i\in\mathcal{I}}w_i\nabla F_i(\theta_{\epsilon})+ \lambda \nabla R_h(\theta_{\epsilon},\theta_p)\Vert\le\epsilon,
\end{align}
$\mathbb{E}\{\Vert \nabla F_{i'}(\theta_{\epsilon})+\lambda\nabla R_h(\theta_{\epsilon},\theta_p)\Vert\}$ can be upper bounded by
\begin{align}
    \label{performance_eq4}
    \nonumber
    &\mathbb{E}\{\Vert \nabla F_{i'}(\theta_{\epsilon})+\lambda\nabla R_h(\theta_{\epsilon},\theta_p)\Vert\}\\
    \nonumber
    =\,&\mathbb{E}\{\Vert\sum\nolimits_{i\in\mathcal{I}}w_i\nabla F_i(\theta_{\epsilon})+ \lambda \nabla R_h(\theta_{\epsilon},\theta_p)\\
    \nonumber
    &+\sum\nolimits_{i\in\mathcal{I}}w_i(\nabla F_{i'}(\theta_{\epsilon})-\nabla F_i(\theta_{\epsilon}))\Vert\}\\
    \le\,&\epsilon+\mathbb{E}\{\underbrace{\Vert\nabla F_{i'}(\theta_{\epsilon})-\sum\nolimits_{i\in\mathcal{I}}w_i\nabla F_i(\theta_{\epsilon})\Vert}_{\text{(a)}}\}.
\end{align}
Due to \cref{similarity}, for $i\in\mathcal{I}\cup\{i'\}$ and $\mathcal{D}_i$ with respect to $P_i$, we can write
\begin{align}
    \label{performance_9}
    \mathbb{E}\{\Vert \nabla L_i(\theta,\mathcal{D}_i)-\nabla L_i(\theta)\Vert\}\le\frac{\sigma^g_i}{\sqrt{D_i}}.
\end{align}
Based on \cref{performance_9}, we observe that
\begin{align}
    \label{performance_10}
    \nonumber
    &\mathbb{E}\{\Vert \nabla L_{i'}(\theta,\mathcal{D}_{i'})-\nabla L_i(\theta,\mathcal{D}_i)\Vert\}\\
    \nonumber
    \le\,&\mathbb{E}\{\Vert \nabla L_{i'}(\theta,\mathcal{D}_{i'})-\nabla L_{i'}(\theta)\Vert\}\\
    \nonumber
    &+\mathbb{E}\{\Vert \nabla L_{i'}(\theta)-\nabla L_i(\theta)\Vert\}\\
    \nonumber
    &+\mathbb{E}\{\Vert \nabla L_i(\theta,\mathcal{D}_i)-\nabla L_i(\theta)\Vert\}\\
    \le\,&\psi^g_i+\frac{\sigma^g_i}{\sqrt{D_{i'}}}+\frac{\sigma^g_i}{\sqrt{D_i}}.
\end{align}
Similarly, we can show that
\begin{align}
    \label{performance_11}
    \mathbb{E}\{\Vert \nabla^2 L_{i'}(\theta,\mathcal{D}_{i'})-\nabla^2 L_i(\theta,\mathcal{D}_i)\Vert\}\le\psi^h_i+\frac{\sigma^h_i}{\sqrt{D_{i'}}}+\frac{\sigma^h_i}{\sqrt{D_i}}.
\end{align}
Thus, for (a), we have
\begin{align}
    \nonumber
    &\Vert\nabla F_{i'}(\theta)-\sum\nolimits_{i\in\mathcal{I}}w_i\nabla F_i(\theta)\Vert\\
    \nonumber
    =\,&\Vert \nabla_\theta L_{i'}(\theta-\alpha\nabla L_{i'}(\theta,\mathcal{D}^{s}_{i'}),\mathcal{D}^q_{i'})\\
    \nonumber
    &-\sum\nolimits_{i\in\mathcal{I}}w_i\nabla_\theta L_i(\theta-\alpha\nabla L_i(\theta,\mathcal{D}^{s}_i),\mathcal{D}^q_i)\Vert\\
    \nonumber
    \le\,&\Vert\nabla L_{i'}(\theta-\alpha\nabla L_{i'}(\theta,\mathcal{D}^{s}_{i'}),\mathcal{D}^q_{i'})\\
    \nonumber
    &\underbrace{-\sum\nolimits_{i\in\mathcal{I}}w_i\nabla L_i(\theta-\alpha\nabla L_i(\theta,\mathcal{D}^{s}_i),\mathcal{D}^q_i)\Vert}_\text{(b)}\\
    \nonumber
    &+\alpha\Vert\nabla^2 L_{i'}(\theta,\mathcal{D}^{s}_{i'})\nabla L_{i'}(\theta-\alpha\nabla L_{i'}(\theta,\mathcal{D}^{s}_{i'}),\mathcal{D}^q_{i'})\\
    &\underbrace{-\sum\nolimits_{i\in\mathcal{I}}w_i\nabla^2 L_i(\theta,\mathcal{D}^{s}_i)\nabla L_i(\theta-\alpha\nabla L_i(\theta,\mathcal{D}^{s}_i),\mathcal{D}^q_i)\Vert}_\text{(c)}.
    \label{performance_eq5}
\end{align}
Based on Assumptions \ref{Lsmooth} and \ref{similarity}, we have
\begin{align}
    \label{performance_eq6}
    \nonumber
    \mathbb{E}\{\text{(b)}\}\le\;&\mathbb{E}\{\Vert\nabla L_{i'}(\theta-\alpha\nabla L_{i'}(\theta,\mathcal{D}^{s}_{i'}),\mathcal{D}^q_{i'})\\
    \nonumber
    &-\sum\nolimits_{i\in\mathcal{I}}w_i\nabla L_i(\theta-\alpha\nabla L_{i'}(\theta,\mathcal{D}^{s}_{i'}),\mathcal{D}^q_i)\Vert\\
    \nonumber
    &+\Vert\sum\nolimits_{i\in\mathcal{I}}w_i\nabla L_i(\theta-\alpha\nabla L_{i'}(\theta,\mathcal{D}^{s}_{i'}),\mathcal{D}^q_i)\\
    \nonumber
    &-\sum\nolimits_{i\in\mathcal{I}}w_i\nabla L_i(\theta-\alpha\nabla L_i(\theta,\mathcal{D}^{s}_i),\mathcal{D}^q_i)\Vert\}\\
    \nonumber
    \le\;& \sum\nolimits_{i\in\mathcal{I}}w_i(\psi^g_i+\frac{\sigma^g_i}{\sqrt{D^{q}_{i'}}}+\frac{\sigma^g_i}{\sqrt{D^{q}_i}})\\
    &\underbrace{+\alpha\mu\sum\nolimits_{i\in\mathcal{I}}w_i(\psi^g_i+\frac{\sigma^g_i}{\sqrt{D^{s}_{i'}}}+\frac{\sigma^g_i}{\sqrt{D^{s}_i}})}_\text{(d)}.
\end{align}
Similarly, for (c), it follows that
\begin{align}
    \label{performance_eq7}
    \nonumber
    \mathbb{E}\{\text{(c)}\}
    \le\alpha\mathbb{E}\{\Vert\nabla^2 L_{i'}(\theta,\mathcal{D}^{s}_{i'})\nabla L_{i'}(\theta-\alpha\nabla L_{i'}(\theta,\mathcal{D}^{s}_{i'}),\mathcal{D}^q_{i'})\\
    \nonumber
    -\sum\nolimits_{i\in\mathcal{I}}w_i\nabla^2 L_i(\theta,\mathcal{D}^{s}_i)\nabla L_{i'}(\theta-\alpha\nabla L_{i'}(\theta,\mathcal{D}^{s}_{i'}),\mathcal{D}^q_{i'})\Vert\}\\
    \nonumber
    +\alpha\mathbb{E}\{\Vert\sum\nolimits_{i\in\mathcal{I}}w_i\nabla^2 L_i(\theta,\mathcal{D}^{s}_i)\nabla L_{i'}(\theta-\alpha\nabla L_{i'}(\theta,\mathcal{D}^{s}_{i'}),\mathcal{D}^q_{i'})\\
    \nonumber
    -\sum\nolimits_{i\in\mathcal{I}}w_i\nabla^2 L_i(\theta,\mathcal{D}^{s}_i)\nabla L_i(\theta-\alpha\nabla L_i(\theta,\mathcal{D}^{s}_i),\mathcal{D}^q_i)\Vert\}\\
    \le \alpha\beta_{i'}\sum\nolimits_{i\in\mathcal{I}}w_i(\psi^h_i+\frac{\sigma^h_i}{\sqrt{D^{s}_{i'}}}+\frac{\sigma^h_i}{\sqrt{D^{s}_i}})+\alpha\mu\cdot\text{(d)}.
\end{align}
Plugging \cref{performance_eq6,performance_eq7} in \cref{performance_eq5} yields
\begin{align}
    \label{performance_eq8}
    \nonumber
    \mathbb{E}\{\text{(a)}\}\le\;&\alpha\beta_{i'}\sum\nolimits_{i\in\mathcal{I}}w_i(\psi^h_i+\frac{\sigma^h_i}{\sqrt{D^{s}_{i'}}}+\frac{\sigma^h_i}{\sqrt{D^{s}_i}})\\
    \nonumber
    &+(\alpha\mu+1)(\alpha\mu)\sum\nolimits_{i\in\mathcal{I}}w_i(\psi^g_i+\frac{\sigma^g_i}{\sqrt{D^{q}_{i'}}}+\frac{\sigma^g_i}{\sqrt{D^{q}_i}})\\
    \nonumber
    &+(\alpha\mu+1)\sum\nolimits_{i\in\mathcal{I}}w_i(\psi^g_i+\frac{\sigma^g_i}{\sqrt{D^{s}_{i'}}}+\frac{\sigma^g_i}{\sqrt{D^{s}_i}})\\
    \nonumber
    =\;&\alpha\beta_{i'}\sum\nolimits_{i\in\mathcal{I}}w_i\psi^h_i+(\alpha\mu+1)^2\sum\nolimits_{i\in\mathcal{I}}w_i\psi^g_i\\
    \nonumber
    &+\alpha\beta_{i'}\sum\nolimits_{i\in\mathcal{I}}w_i(\frac{\sigma^h_i}{\sqrt{D^{s}_{i'}}}+\frac{\sigma^h_i}{\sqrt{D^{s}_i}})\\
    \nonumber
    &+(\alpha\mu+1)(\alpha\mu)\sum\nolimits_{i\in\mathcal{I}}w_i(\frac{\sigma^g_i}{\sqrt{D^{q}_{i'}}}+\frac{\sigma^g_i}{\sqrt{D^{q}_i}})\\
    &+(\alpha\mu+1)\sigma^g_i\sum\nolimits_{i\in\mathcal{I}}w_i(\frac{\sigma^g_i}{\sqrt{D^{s}_{i'}}}+\frac{\sigma^g_i}{\sqrt{D^{s}_i}}).
\end{align}
Plugging \cref{performance_eq8} in \cref{performance_eq4}, we obtain
\begin{align}
    &\mathbb{E}\{\Vert \nabla F_{i'}(\theta_{\epsilon})+\lambda\nabla R_h(\theta_{\epsilon},\theta_p)\Vert\}\nonumber\\
    \nonumber
    \le\,&\epsilon+\alpha\beta_{i'}\sum\nolimits_{i\in\mathcal{I}}w_i\psi^h_i+(\alpha\mu+1)^2\sum\nolimits_{i\in\mathcal{I}}w_i\psi^g_i\\
    \nonumber
    +\;&\alpha\beta_{i'}\sum\nolimits_{i\in\mathcal{I}}w_i(\frac{\sigma^h_i}{\sqrt{D^{s}_{i'}}}+\frac{\sigma^h_i}{\sqrt{D^{s}_i}})\\
    \nonumber
    +\;&(\alpha\mu+1)(\alpha\mu)\sum\nolimits_{i\in\mathcal{I}}w_i(\frac{\sigma^g_i}{\sqrt{D^{q}_{i'}}}+\frac{\sigma^g_i}{\sqrt{D^{q}_i}})\\
    +\;&(\alpha\mu+1)\sum\nolimits_{i\in\mathcal{I}}w_i(\frac{\sigma^g_i}{\sqrt{D^{s}_{i'}}}+\frac{\sigma^g_i}{\sqrt{D^{s}_i}}),
\end{align}
thereby completing the proof. 
\end{proof}

Theorem \ref{performance} sheds light on the performance of adaptation with the pretrained knowledge, which depends on the sample sizes, the variance of stochastic gradient and Hessian, and the similarity between the target client and training clients. In particular, if $D^{s}_j=\mathcal{O}(\epsilon^{-2})$ and $D^{q}_j=\mathcal{O}(\epsilon^{-2})$ for $j\in\mathcal{I}\cup\{i'\}$, then an $\mathcal{O}(\epsilon+\sum_{i\in\mathcal{I}}w_i(\psi^h_i+\psi^g_i))$-FOSP can be obtained at the target client. However, it is clear that the larger the data of training clients dissimilar to the target client is, the worse the rapid adaptation performs. In the next subsection, we will show these issues can be alleviated via regularization with a good pretrained model.

\subsection{Knowledge Transfer}

In this section, we aim to quantify the benefit of knowledge transfer from the pretrained model. Specifically, we consider a special case, where the regularizer is the squared Euclidean distance, \ie, $R_h(\theta,\theta_p)=\Vert\theta-\theta_p\Vert^2$. Based on that, we first derive an upper bound of $R_h(\theta,\theta_p)$ via the following lemma.

\begin{lemma}
\label{lem:forgetting}
Given Assumptions \ref{lowerbounded}--\ref{parameterassumption}, for any $\epsilon>0$, the $\epsilon$-FOSP solution $\theta_{\epsilon}$ obtained by Algorithm \ref{alg} satisfies that
\begin{align}
    \label{eq:forgetting_general_ex}
    \mathbb{E}\{R_h(\theta_{\epsilon},\theta_p)\}\le&\;\frac{1}{\lambda}(\epsilon+\sum_{i\in\mathcal{I}}w_i(\beta_i+\frac{\alpha\mu_i\sigma^g_i}{\sqrt{D^s_i}}+\frac{\sigma^g_i}{\sqrt{D^q_i}}))\nonumber\\
    \cdot&\;\Vert\theta_{\epsilon}-\theta_p\Vert.
\end{align}
Particularly, suppose $R_h(\theta,\theta_p)$ is strongly convex with respect to $\theta$, \ie, there exists $M>0$ such that for $\theta_1,\theta_1\in\mathbb{R}^n$, the following holds:
\begin{align}
    \label{eq:strongly_convex_primal_ex}
    \left\langle \nabla R_h(\theta_1,\theta_p)-\nabla R_h(\theta_2,\theta_p),x-y\right\rangle\ge M\Vert \theta_1-\theta_2\Vert^2.
\end{align}
Then, \cref{eq:forgetting_general_ex} can be written as
\begin{align}
    \label{eq:forgetting_general_strong_ex}
    \mathbb{E}\{R_h(\theta_{\epsilon},\theta_p)\}\le\frac{1}{M\lambda^2}(\epsilon+\sum_{i\in\mathcal{I}}w_i(\beta_i+\frac{\alpha\mu_i\sigma^g_i}{\sqrt{D^s_i}}+\frac{\sigma^g_i}{\sqrt{D^q_i}}))^2.
\end{align}
\end{lemma}

\begin{proof}
The proof can be found in Appendix.
\end{proof}

For a current model parameter $\theta$, we define the \textit{forgetting cost} of $\theta$ on the pretraining task $p$ as $L_p(\theta)$ as in \cite{krishnan2020meta}, where $L_p(\cdot)$ is the expected loss over the data distribution of task $p$ (defined by \cref{eq:expected_loss}). Based on \cref{lem:forgetting}, we next characterize the forgetting cost of the $\epsilon$-FOSP solution. 

\begin{theorem} 
\label{thm:forgetting} Suppose $L_p(\cdot)$ is $\mu_p$-smooth, $\Vert\nabla L_p(\theta_p)\Vert\le\epsilon_p$ for some $\epsilon_p>0$, and $R_h(\theta,\theta_p)=\Vert\theta-\theta_p\Vert^2$. Under Assumptions \ref{lowerbounded}--\ref{parameterassumption}, for the $\epsilon$-FOSP solution $\theta_{\epsilon}$, the following fact holds: 
\begin{align} \label{eq:forgeting_eucliean} \Vert\nabla L_p(\theta_{\epsilon})\Vert\le\epsilon_p+\frac{\nu_p}{2\lambda}(\epsilon+\sum_{i\in\mathcal{I}}w_i(\beta_i+\frac{\alpha\mu_i\sigma^g_i}{\sqrt{D^s_i}}+\frac{\sigma^g_i}{\sqrt{D^q_i}})). 
\end{align} 
\end{theorem}

\begin{proof}
The result can be directly obtained by \cref{lem:forgetting}, and thus we omit for brevity.
\end{proof}

Based on \cref{thm:forgetting} and \cref{coro:data_impact}, it is clear that by selecting a suitable $\lambda$, the regularizer enables the meta-model to learn on the current task while maintaining good performance on the pretraining task. On the other hand, combined with \cref{performance}, \cref{thm:forgetting} also implies that owing to the independence of \cref{eq:forgeting_eucliean} and the similarity conditions, a good pretrained model (\ie, with a relatively small $\Vert \nabla L_{i'}(\theta_p)\Vert$) can effectively alleviate the significant performance degradation caused by the dissimilarity between the training clients and the target client, which has been deemed as a significant limitation of the existing federated meta-learning methods~\cite{fallah2020personalized}.

\section{Practical Instantiation}
\label{sec:instantiation}

Based on contrastive representation distillation (CRD) \cite{tian2019crd}, this section provides an instantiation for the regularization used in the knowledge transfer, capable of transferring the valuable structural knowledge from the pretrained model to the meta-model. 


Specifically, denote $\mathcal{D}_s=\{(\mathbf{x}^j,\mathbf{y}^j)\}^{D_s}_{j=1}$ as a labeled dataset at the server. Given an input $\mathbf{x}\in\mathcal{D}_s$, we denote $c = g^{\theta_p}(\mathbf{x})$ and $s = g^{\theta}(\mathbf{x})$ as the representations at the penultimate layers (before logits) of pretrained model $\theta_p$ and training model $\theta$, respectively. For an input pair $(\mathbf{x}_i,\mathbf{x}_j)$, if $i=j$, we regard it as a positive pair and otherwise as a negative pair. The core idea of contrastive learning is to push closer $g^{\theta_p}(\mathbf{x}_j)$ and $g^{\theta}(\mathbf{x}_j)$ for positive pairs while pushing apart $g^{\theta_p}(\mathbf{x}_j)$ and $g^{\theta}(\mathbf{x}_i)$ for negative pairs via maximizing the \textit{mutual information} thereof.



Let $p(c,s)$ represent the joint distribution of $(c,s)$. Define a distribution $q$ with latent variable $z$ as follows:
\begin{align*}
    q(c,s|z=1)=p(c,s),\quad q(c,s|z=0)=p(c)p(s),
\end{align*}
where $p(c)$ and $p(s)$ represent the marginals. We select $N+1$ sample pairs from $\mathcal{D}_s$, consisting of 1 positive pair and $N$ negative pairs, to build a training batch. In the batch, the priors of the latent variable $z$ follow
\begin{align*}
    &q(z=1)=\frac{1}{N+1},\quad q(z=0)=\frac{N}{N+1}.
\end{align*}
Based on Bayes' rule, the posterior of $z=1$ is expressed as
\begin{align}
    \label{eq:posteriors_of_Z}
    &q(z=1|c,s)=\frac{p(c,s)}{p(c,s)+Np(c)p(s)}.
\end{align}
The mutual information between $c$ and $s$ is defined by
\begin{align*}
    I(c,s)\doteq\mathbb{E}_{p(c,s)}\left[\log{\frac{p(c,s)}{p(c)p(s)}}\right],
\end{align*}
which captures the correlation between $c$ and $s$. 
Define $h(c,s)$:
\begin{align*}
    h(c,s)\doteq\frac{\exp({f^c(c)'f^s(s)/\tau})}{\exp({f^c(c)'f^s(s)}/\tau)+N/D_s},
\end{align*}
where $\tau$ is the temperature, $D_s$ is the cardinality of $\mathcal{D}_s$, and $f$ is a mapping function that maps $c$ and $s$ to the same dimension. By taking expectation w.r.t $p(c,s)$ and taking logarithm on \cref{eq:posteriors_of_Z}, the mutual information can be lower bounded by
\begin{align}
    \label{eq:mutual_bounded}
    I(c,s)&\geq \log{N}+\mathbb{E}_{q(c,s|z=1)}[\log{q(z=1|c,s)}]\nonumber\\
    &=\log{N}-D^*_h,
\end{align}
where $D^*_h$ is denoted as
\begin{align*}
    D^*_h&\doteq \min_h -N\mathbb{E}_{q(c,s|z=0)}[\log{(1-h(c,s))}]\nonumber\\
    &-\mathbb{E}_{q(c,s|z=1)}[\log{h(c,s)}].
\end{align*}
Therefore, based on \cref{eq:mutual_bounded}, we can instantiate $R_h$ as
\begin{align*}
    \nonumber
    R_h(\theta,\theta_p) = -\mathbb{E}_{q(g^\theta(\mathbf{x}_i),g^{\theta_p}(\mathbf{x}_j)|z=1)}\left[\log{h(g^\theta(\mathbf{x}_i),g^{\theta_p}(\mathbf{x}_j))}\right]\\
-N\mathbb{E}_{q(g^\theta(\mathbf{x}_i),g^{\theta_p}(\mathbf{x}_j)|z=0)}\left[\log{(1-h(g^\theta(\mathbf{x}_i),g^{\theta_p}(\mathbf{x}_j)))}\right]
\end{align*}
to maximize the lower bound of mutual information, facilitating the structural knowledge transfer from pretrained model $\theta_p$ to learning model $\theta$.

\section{Experiments}
\label{sec:experiment}

In this section, we use experimental studies to evaluate our proposed method by answering the following questions:

\textit{1) How does \name{} (without knowledge transfer) perform on various benchmarks compared to baseline methods?}

\textit{2) How is the scalability of \name{}?}

\textit{3) How does \name{} perform given different sample sizes?}

\textit{4) What is the impact of $\rho$ on \name{}?}

\textit{5) Can PM effectively enhance FL by \name{}?}

\subsection{Experimental Setup}


\textbf{Datasets and models.} We evaluate \name{} on \textcolor{black}{five} commonly used benchmarks: CIFAR-10~\cite{krizhevsky2009learning}, CIFAR-100~\cite{krizhevsky2009learning}, Fashion-MNIST~\cite{xiao2017fashion}, EMNIST~\cite{cohen2017emnist} and Tiny-ImageNet~\cite{chrabaszcz2017downsampled}. Following~\cite{lin2020collaborative}, we consider a challenging setup with limited and heterogeneous data. The data is distributed among $|\mathcal{I}|$ clients as follows: 1) Each client has samples from two random classes; 2) the number of samples per client follows a discrete uniform distribution, \ie, $D_i\sim U(M,2M)$. For each client, we divide the local dataset into a support set and a query set, each with 50\% of local data. We set $50$ clients for Fashion-MNIST, Tiny-ImageNet and CIFAR-10, and $100$ for CIFAR-100 and EMNIST, where we randomly select 80\% and 20\% clients as the training and testing clients respectively. We set the meta-stepsize as $\alpha=0.03$ and the penalty parameters as $\rho=0.7$. We set the weight for regularization as $\lambda=5.0$ and the degree of freedom parameter as $\delta = 1/(10t + 100)$ with $t=1,2,\dots$. We exploit ResNet $8\times4$ as client models for all benchmarks and ResNet $32\times4$ as the PM. The PM for each benchmark is trained on 10,000 images from whole classes, not overlapping with the client data. The hyperparameters used in the experiments is summarized in \cref{table:parameters}.

\begin{table}[H]
    \renewcommand\arraystretch{1.05}
    \centering
    \caption{Hyperparameters.}
    \begin{tabular}{l|l} 
        Hyperparameter                             & Value \\ 
        \hline
        Meta-stepsize ($\alpha$)         & 3e-2  \\
        Optimizer                             & Adam  \\
        Penalty parameter ($\rho$)            & 0.7  \\
        Regularization weight ($\lambda$)     & 5.0  \\
        Batchsize for knowledge transfer                      & 128   \\ 
        \hline
    \end{tabular}
    \label{table:parameters}
\end{table}


\textbf{Baselines.} 
For PM-aided-FL, we consider two recent baselines: 1) FedLP~\cite{nguyen2022begin}, which initializes client models with a PM, and 2) FedPCL~\cite{tan2022federated}, which exploits a frozen pretrained encoder on each client and applies prototype contrastive learning for projector training. Note that both FedPCL and FedLP require deploying the PM on local clients. For personalized FL, we consider four baseline methods: 1) FedAvg~\cite{mcmahan2017communication}, the vanilla FL method; 2) Per-FedAvg ~\cite{fallah2020personalized}, a federated meta-learning algorithm building on FedAvg and MAML; 3) FedProto~\cite{tan2022fedproto}, which aggregates local prototypes from clients into a global prototype and then redistributes it back to clients to standardize the local training; 4) FedCP~\cite{zhang2023fedcp}, which introduces an additional Conditional Policy Network to split global and personalized features for each client.


\textbf{Implementation.} We implement code using PyTorch 1.9.0 and run experiments on a server with an Intel Xeon Gold 6330 CPU and an NVIDIA RTX 3090 24G GPU. The code can be found at \href{https://github.com/zerui981105/AugFL.git}{https://github.com/zerui981105/AugFL.git}.



\subsection{Experimental Results}

To demonstrate the impact of the inexact-ADMM-based technique and the knowledge transfer from the PM respectively, we first remove the PM-based knowledge distillation component and compare \name{}'s performance against the personalized FL baselines. Then, we compare \name{} with two PM-aided-FL approaches and provide ablation properties.

\begin{table*}
    \centering
    \renewcommand\arraystretch{1.25}
    \caption{Comparative results of different methods.}
    \resizebox{\linewidth}{!}{
    \begin{tabular}{@{}llrrrrrrr@{}} 
        \toprule
        ~~\textbf{Dataset}         &\multicolumn{1}{c}{\textbf{FedAvg}}  & \multicolumn{1}{c}{\textbf{Per-FedAvg}}  &\multicolumn{1}{c}{\textbf{FedProto}}  &\multicolumn{1}{c}{\textbf{FedCP}}   &\multicolumn{1}{c}{\textbf{FedLP}}  &\multicolumn{1}{c}{\textbf{FedPCL}}  & \multicolumn{1}{c}
        {\textbf{\name{} (w/o PM)}} & \multicolumn{1}{c}
        {\textbf{\name{} (w/ PM)}}\\ \midrule
        ~~CIFAR-10                   & $54.41\%\pm0.64\%$ & $60.96\%\pm7.19\%$ & $59.67\%\pm4.81\%$ & $68.67\%\pm3.34\%$ & $56.36\%\pm1.56\%$ & $59.71\%\pm0.41\%$ & $\bm{74.49\%\pm2.16\%}$ & $\bm{76.53\%\pm2.0\%}$~~    \\
        ~~CIFAR-100           &$46.18\%\pm2.92\%$ & $48.25\%\pm8.27\%$ & $49.54\%\pm7.15\%$ & $53.07\%\pm5.31\%$ & $48.46\%\pm1.90\%$ & $58.95\%\pm1.70\%$ & $\bm{57.35\%\pm1.76\%}$ & $\bm{61.06\%\pm3.30\%}$~~    \\
        ~~EMNIST               &  $54.34\%\pm1.62\%$ & $69.74\%\pm1.58\%$ & $72.13\%\pm1.57\%$ & $73.68\%\pm3.19\%$ & $79.45\%\pm0.24\%$ & $84.67\%\pm1.55\%$ & $\bm{79.89\%\pm1.69\%}$ & $\bm{89.44\%\pm0.42\%}$~~    \\
        ~~Fashion-MNIST              & $73.95\%\pm3.68\%$ & $90.23\%\pm3.23\%$ & $88.32\%\pm1.73\%$ & $90.24\%\pm1.10\%$ & $91.23\%\pm0.96\%$ & $91.07\%\pm0.76\%$ & $\bm{94.04\%\pm1.16\%}$ & $\bm{97.30\%\pm0.70\%}$~~    \\
        ~~Tiny-ImageNet              & $47.66\%\pm2.04\%$ & $52.41\%\pm1.68\%$ & $48.67\%\pm4.23\%$ & $55.50\%\pm0.43\%$ & $50.19\%\pm4.08\%$ & $57.84\%\pm2.92\%$ & $\bm{58.11\%\pm2.63\%}$ & $\bm{64.45\%\pm1.01\%}$~~    \\
         \bottomrule
    \end{tabular}}
    \label{table:performance}
\end{table*}

\begin{figure*}[t]
    \centering
    \begin{minipage}{0.199\linewidth}
    \centering
    \includegraphics[width=1.0\linewidth]{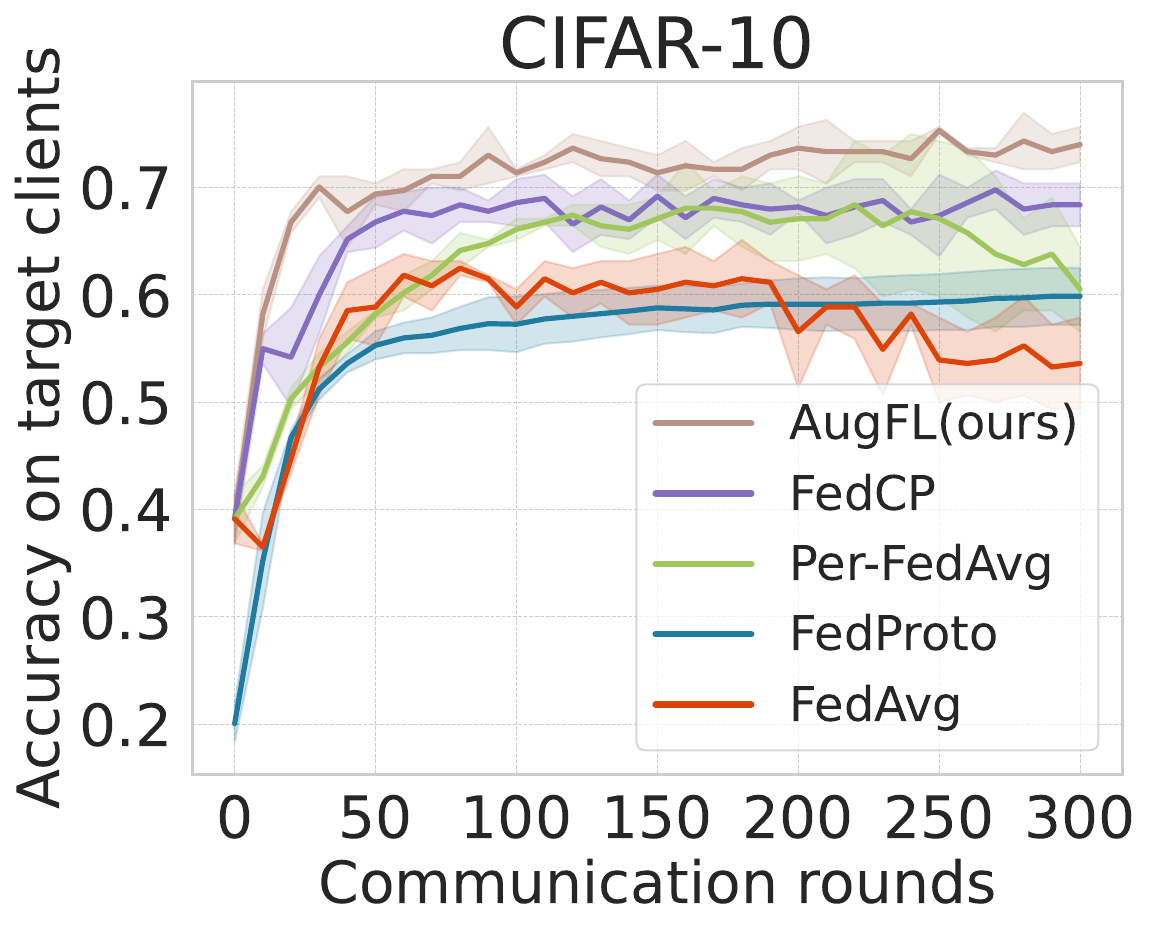}
    \end{minipage}
    \begin{minipage}{0.19\linewidth}
    \centering
    \includegraphics[width=1.0\linewidth]{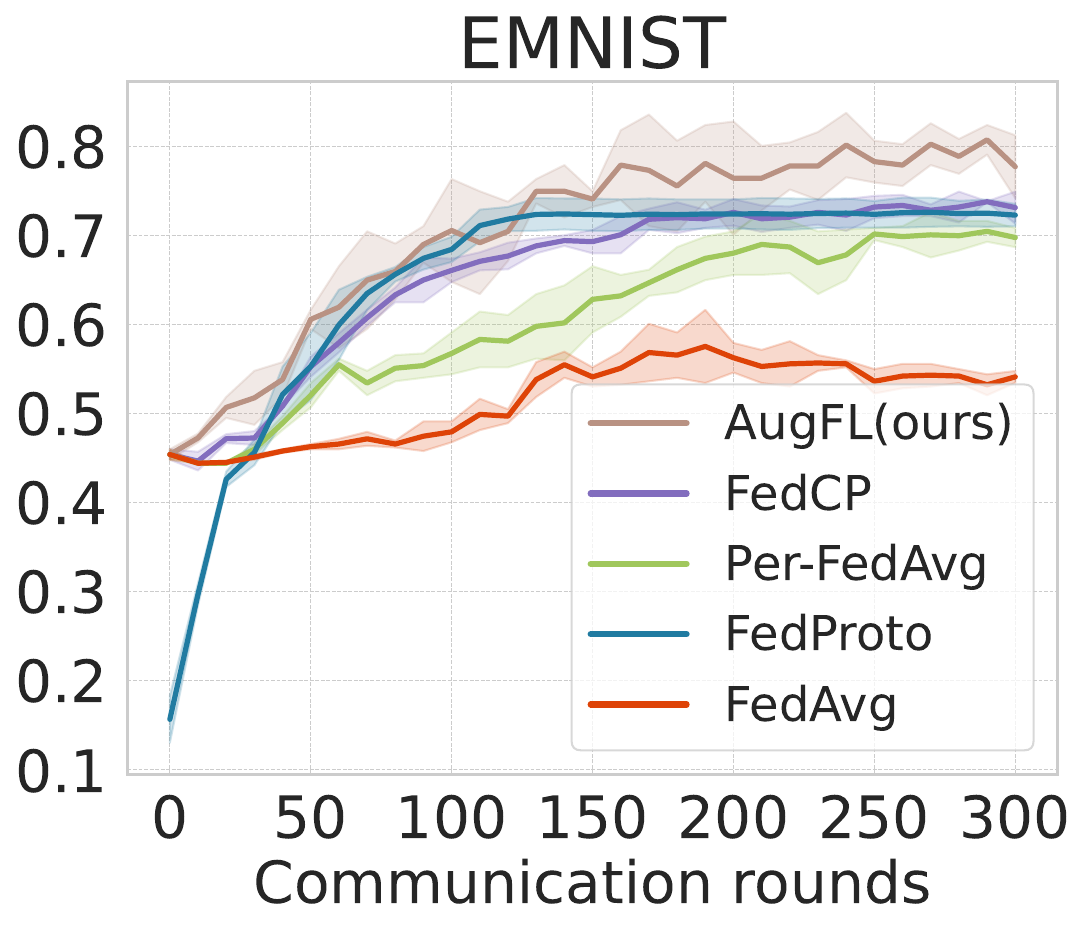}
    \end{minipage}
    \begin{minipage}{0.19\linewidth}
    \centering
    \includegraphics[width=1.0\linewidth]{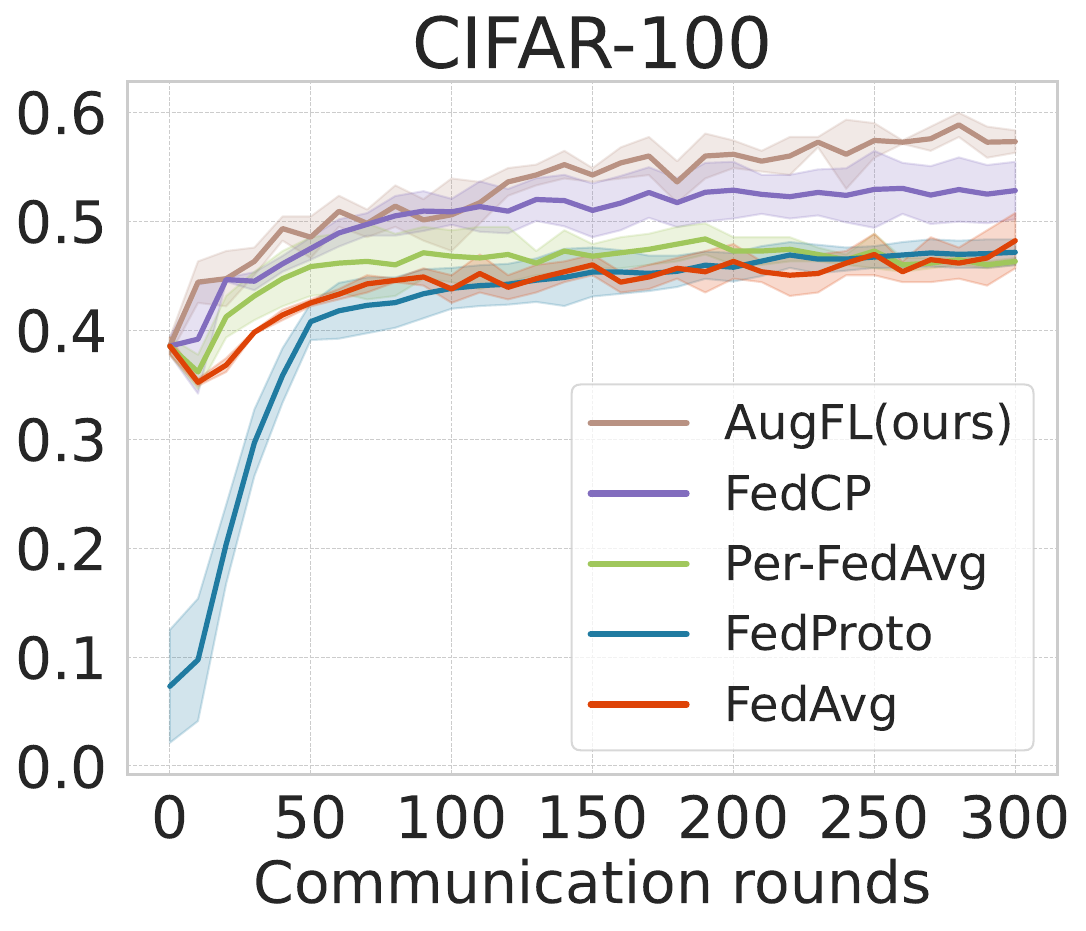}
    \end{minipage}
    \begin{minipage}{0.19\linewidth}
    \centering
    \includegraphics[width=1.0\linewidth]{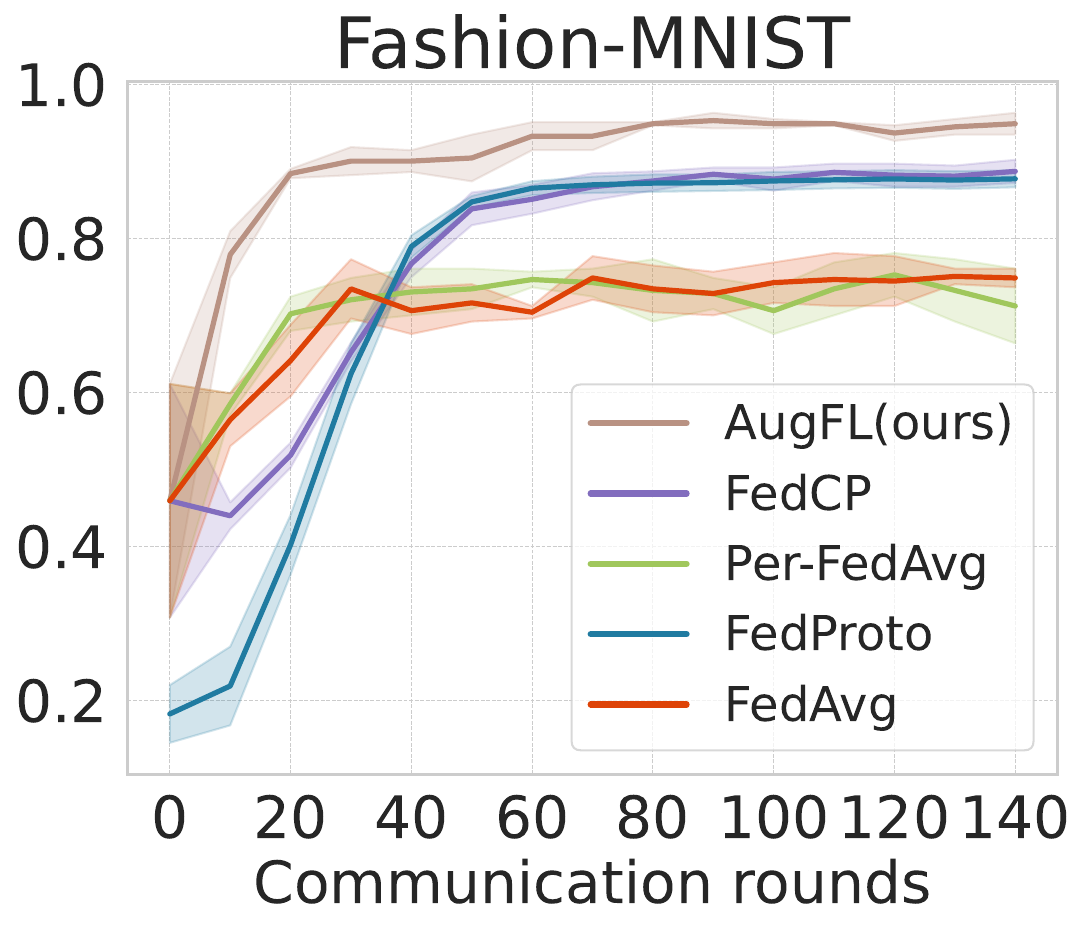}
    \end{minipage}
    \begin{minipage}{0.19\linewidth}
    \centering
    \includegraphics[width=1.0\linewidth]{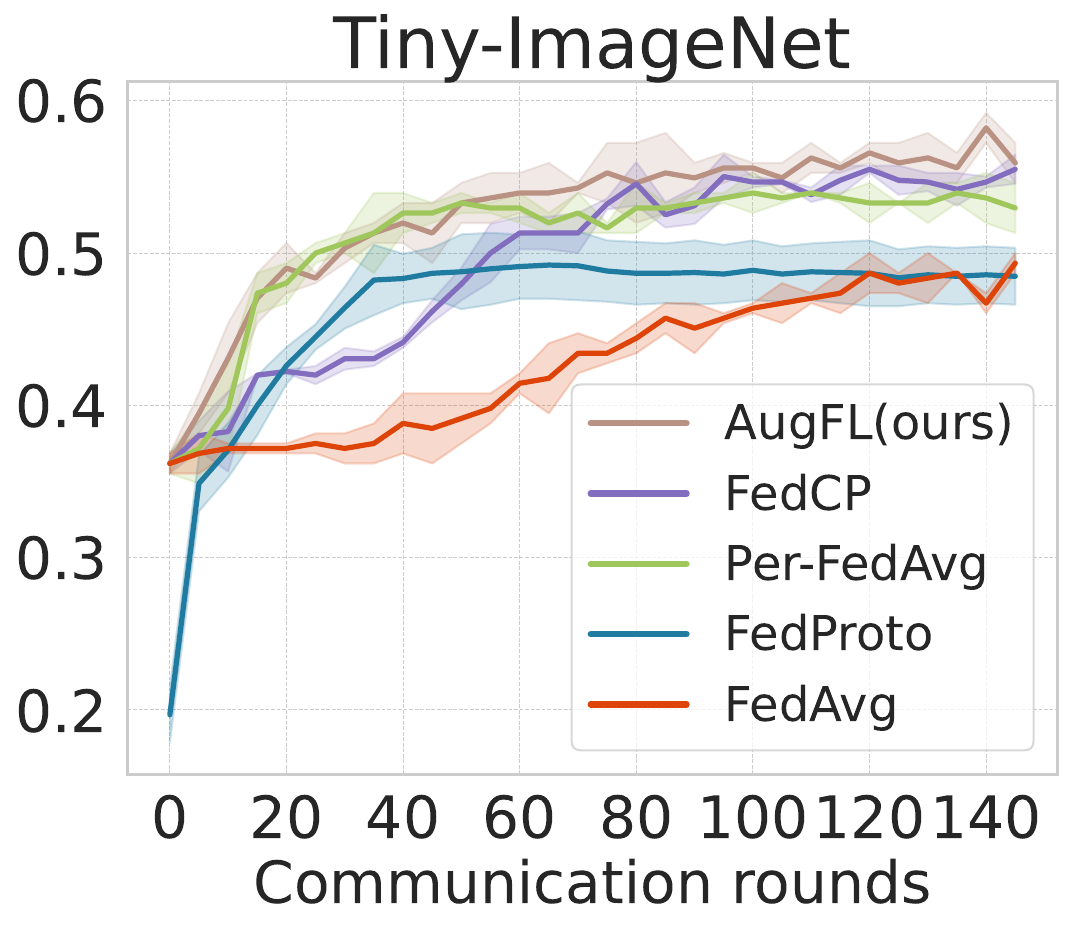}
    \end{minipage}
    \caption{Comparative results between different personalized FL methods.}
    \label{fig:convergence}
    \vspace{-0.2in}
\end{figure*}

\textbf{Results of personalized FL.} We first compare \name{} with personalized FL approaches. To be fair,  we set $\lambda=0$ to remove the impact of using the PM. We set $M=20$ and repeat the experiments 5 times. As shown in \cref{table:performance}, \name{} outperforms the baselines, validating its superior capability in capturing the data heterogeneity. In addition, we depict the learning curves in \cref{fig:convergence}. \name{} converges fast in most benchmarks, indicating it can achieve great communication efficiency.



\begin{figure}[t]
	\centering
	\begin{minipage}{0.49\linewidth}
		\centering
		\includegraphics[width=1.0\linewidth]{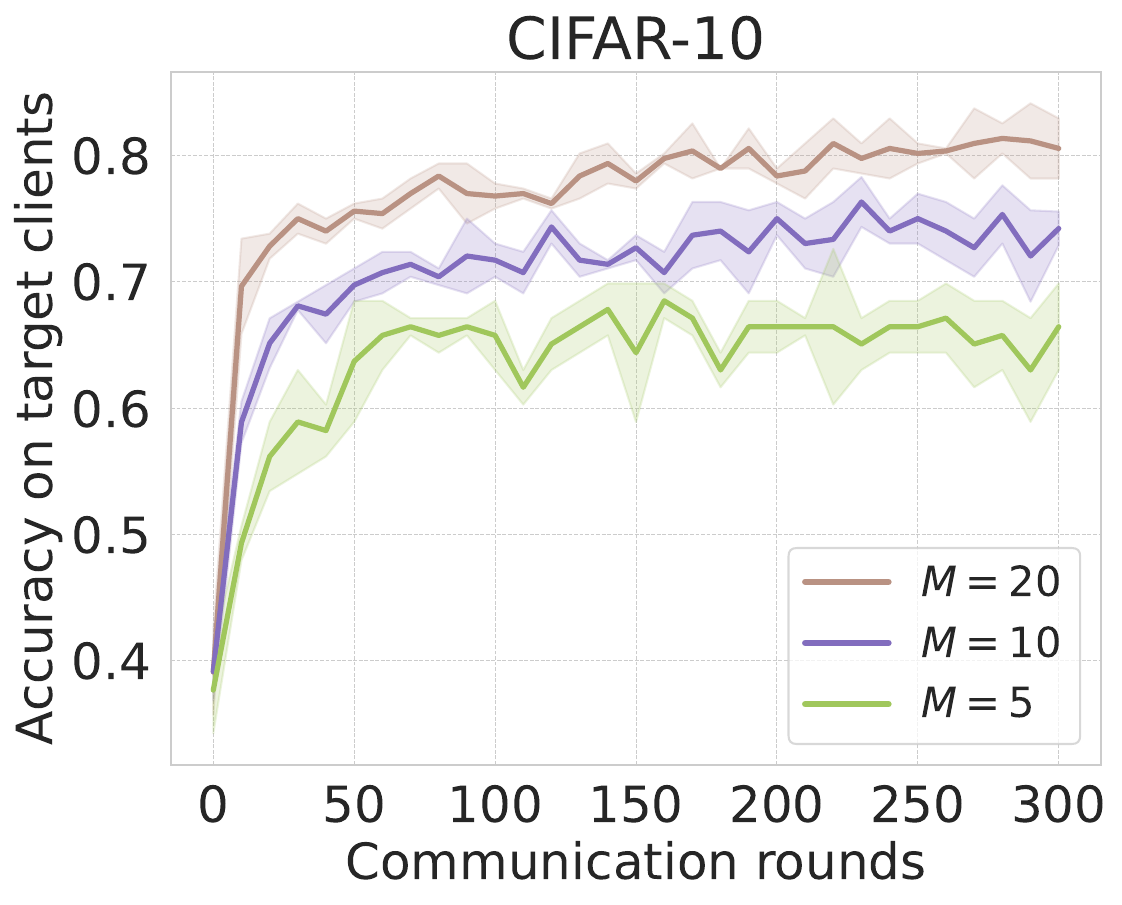}
	\end{minipage}
	\begin{minipage}{0.49\linewidth}
		\centering
		\includegraphics[width=1.0\linewidth]{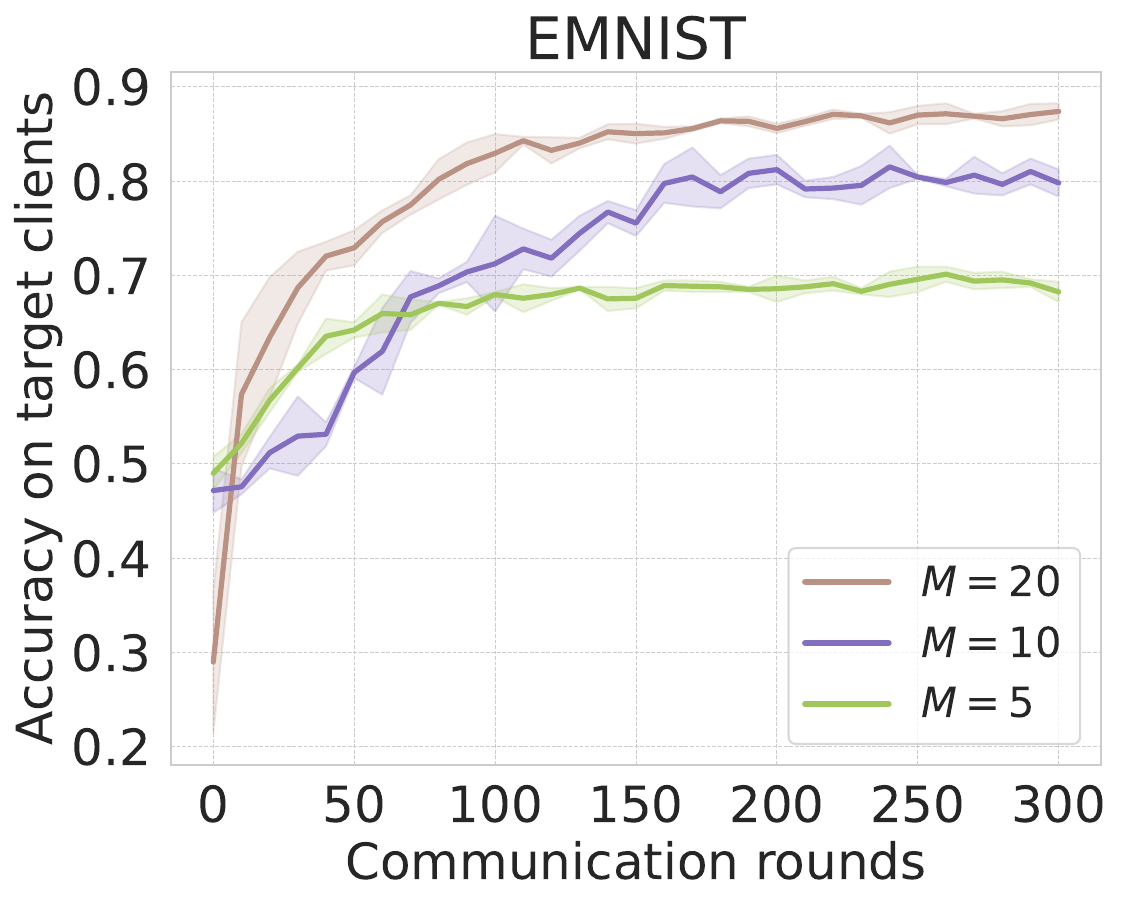}
	\end{minipage}
	
	\begin{minipage}{0.49\linewidth}
		\centering
		\includegraphics[width=1.0\linewidth]{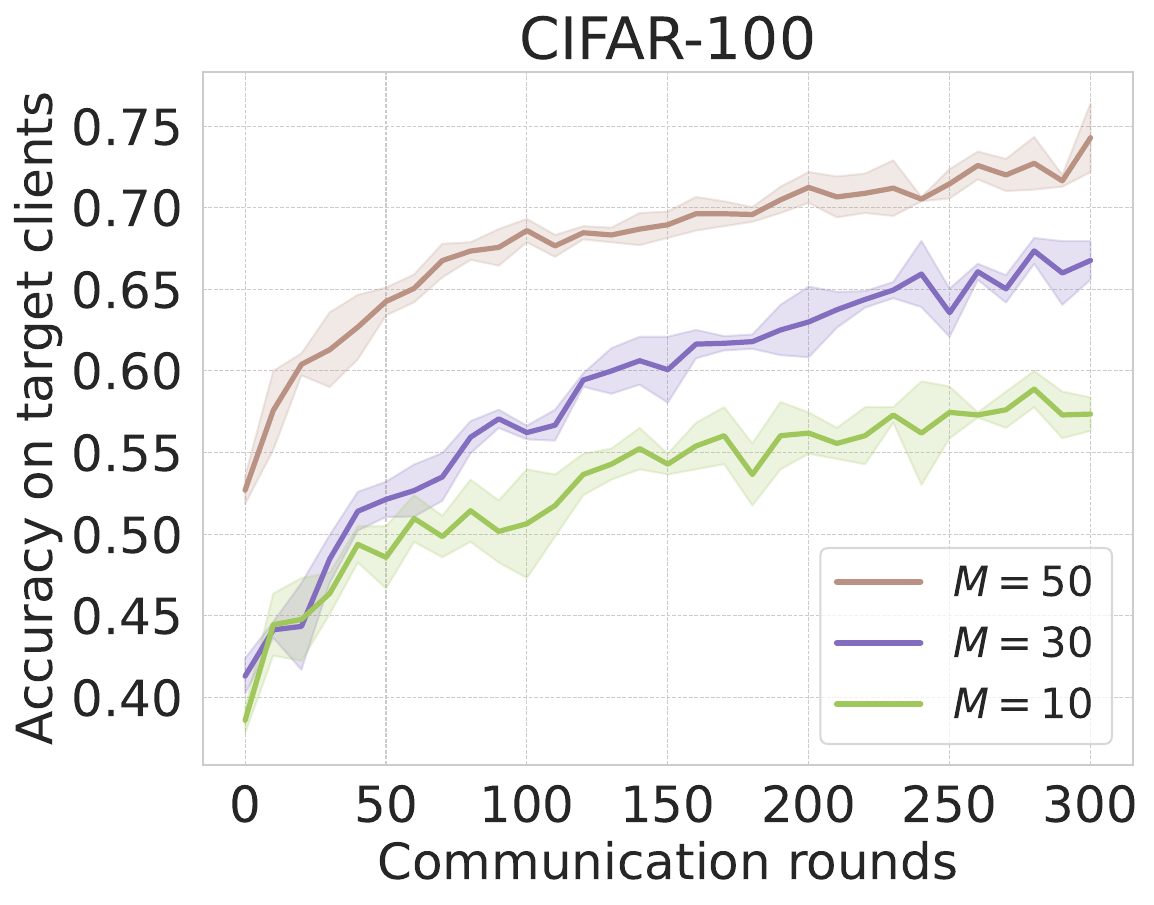}
	\end{minipage}
	\begin{minipage}{0.49\linewidth}
		\centering
		\includegraphics[width=1.0\linewidth]{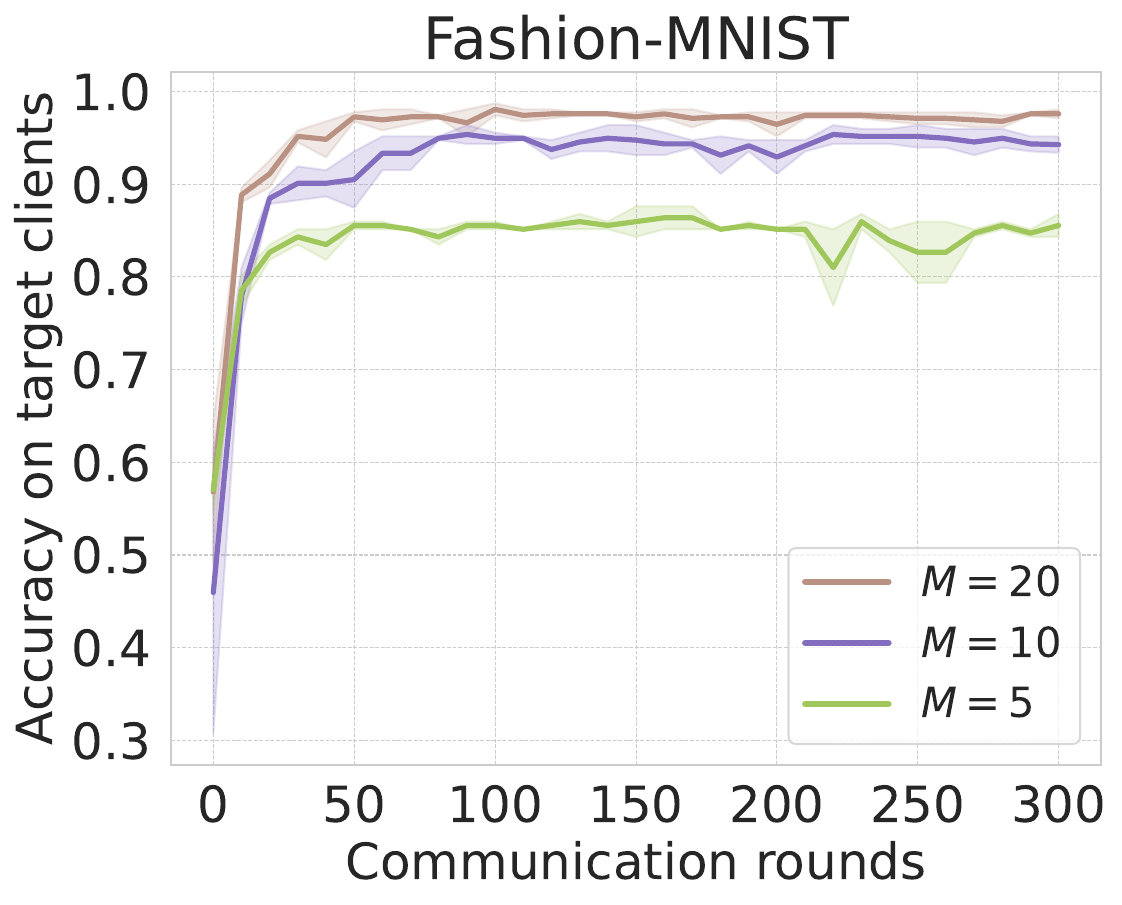}
	\end{minipage}
        \caption{Impact of dataset sizes on performance.}
        \label{fig:datasize}
        
\end{figure}

\textbf{Results under different dataset sizes.} To investigate the impact of dataset sizes on the performance, we fix the number of clients to 40 for Fashion-MNIST, CIFAR-10, and 80 for EMNIST, CIFAR-100. We vary parameter $M$ in the uniform distribution $U(M,2M)$, ranging from 5 to 20 for EMNIST, CIFAR-10, and Fashion-MNIST and 0 to 50 for CIFAR-100. Not surprisingly, the performance of \name{} consistently improves as the amount of local data increases. In particular, albeit with scarce local data, \name{} achieves good performance, demonstrating its great sample efficiency.


	
        

\begin{figure}[t]
	\centering
	\begin{minipage}{0.49\linewidth}
		\centering
		\includegraphics[width=1.0\linewidth]{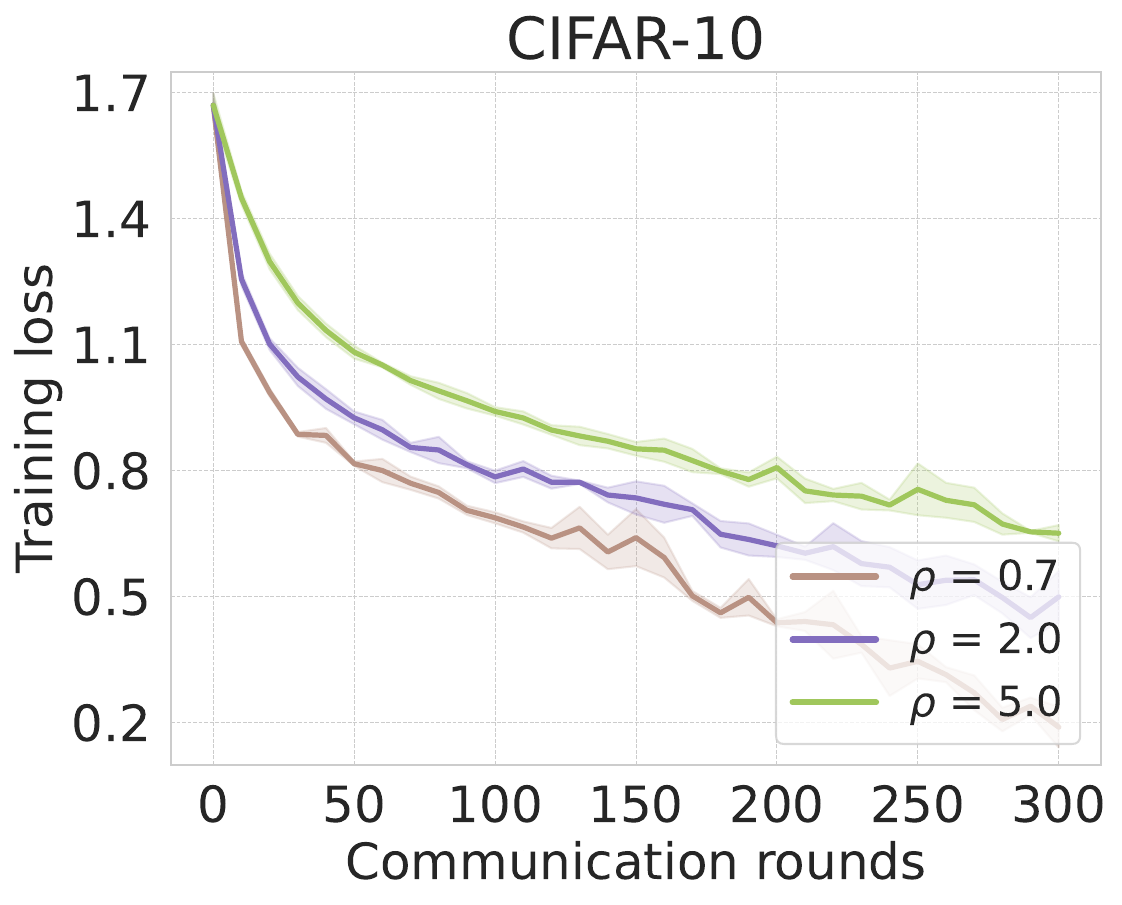}
	\end{minipage}
	\begin{minipage}{0.49\linewidth}
		\centering
		\includegraphics[width=1.0\linewidth]{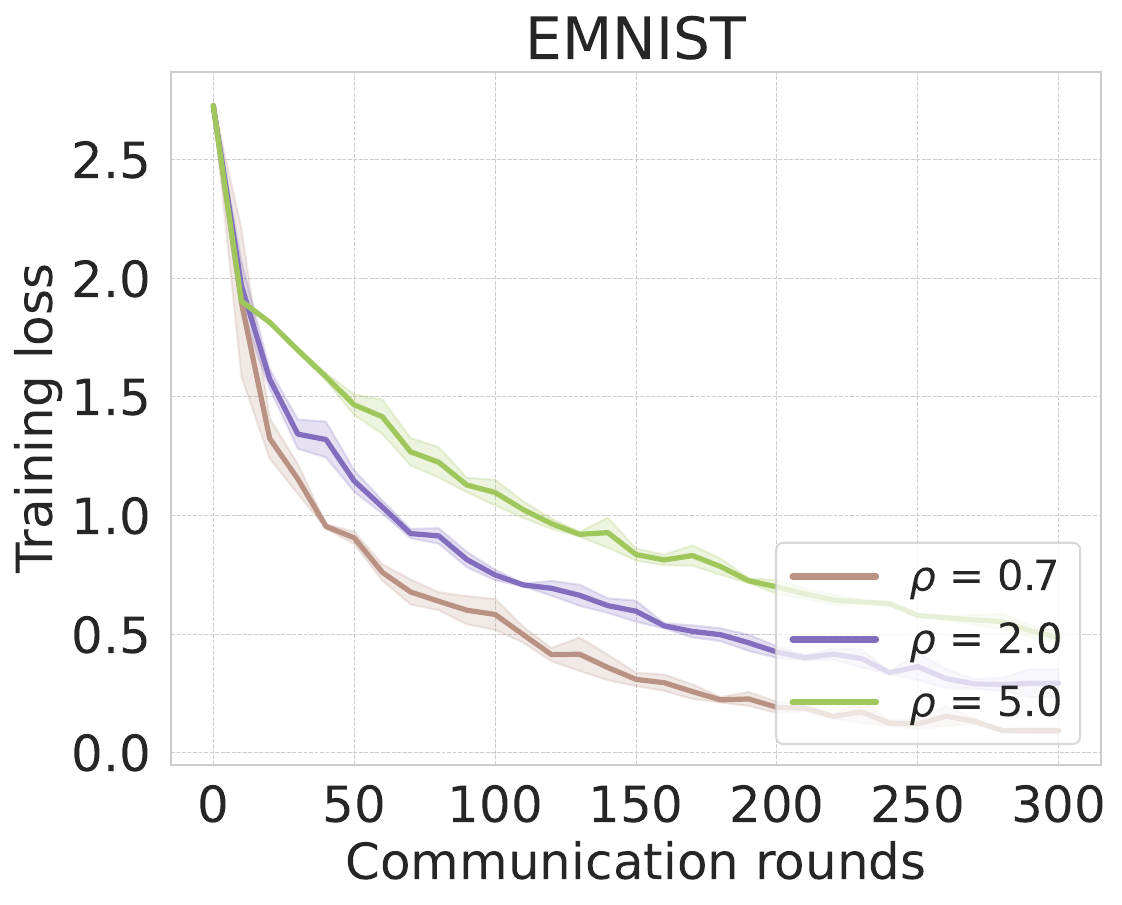}
	\end{minipage}
	
	\begin{minipage}{0.49\linewidth}
		\centering
		\includegraphics[width=1.0\linewidth]{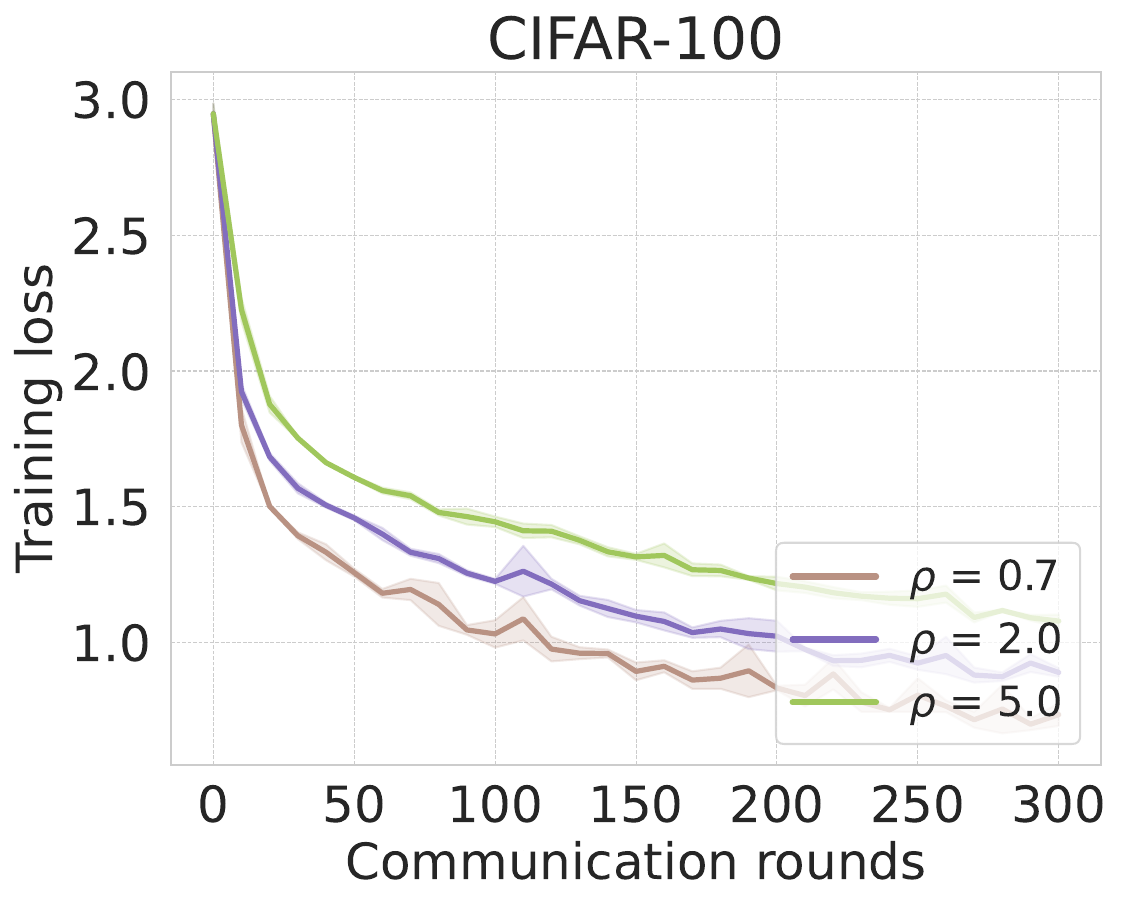}
	\end{minipage}
	\begin{minipage}{0.49\linewidth}
		\centering
		\includegraphics[width=1.0\linewidth]{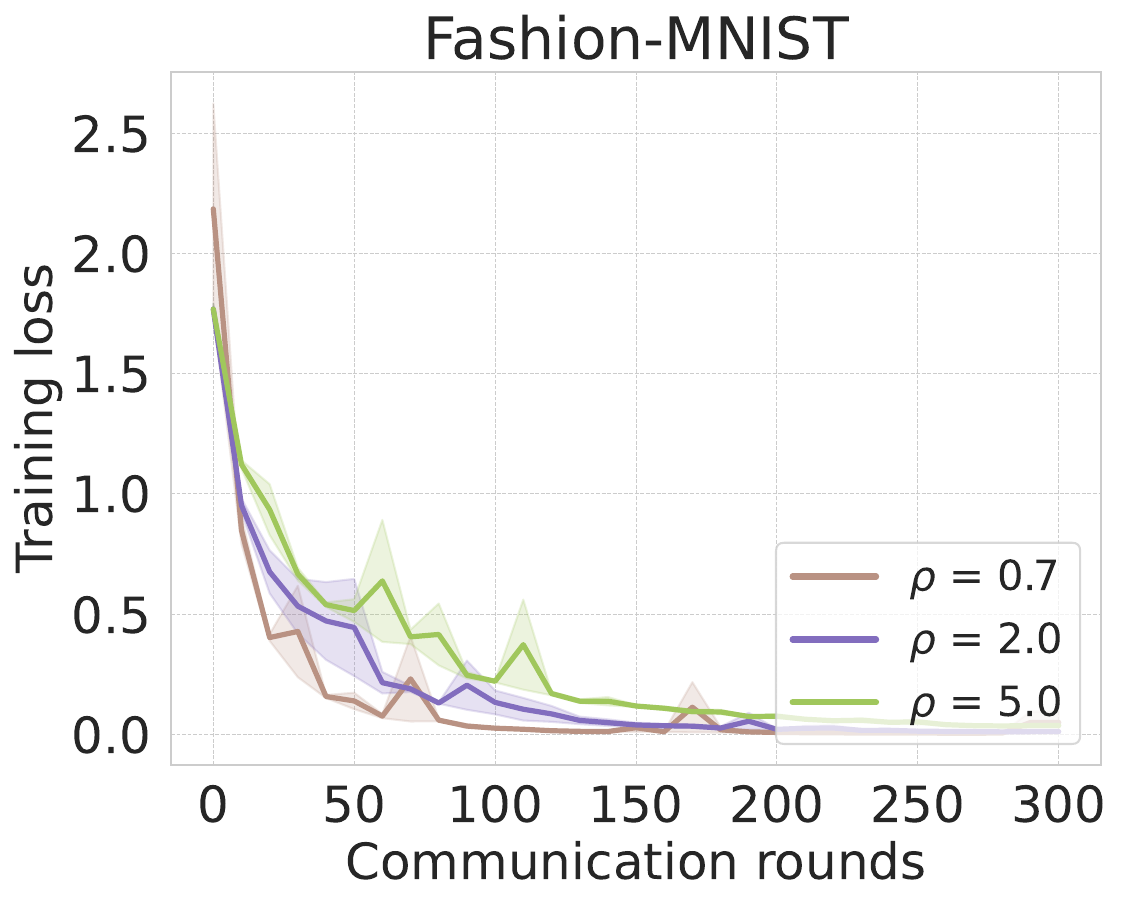}
	\end{minipage}
        \caption{Impact of penalty parameter $\rho$.}
        \label{fig:rho}
        
\end{figure}

\textbf{Impact of penalty parameter $\bm{\rho}$.} We investigate the impact of the penalty parameter $\rho$ on the convergence of \name{} ($\rho_i$ is identical across different clients). It can be seen from \cref{fig:rho} that \name{} has a relatively faster convergence speed with a smaller $\rho$ in terms of the training loss, aligning with our theoretical findings. In particular, a small change of $\rho$ does not greatly affect the convergence properties of the algorithm, which implies that \name{} is robust to the penalty parameter.

\textbf{Results under different numbers of clients.} To show the scalability of \name{}, we fix the dataset size of clients and vary the number of clients from 10 to 80 across EMNIST and CIFAR-10. The results, as depicted in \cref{fig:node_number}, demonstrate that the performance of our algorithm improves with an increasing number of clients. This improvement is attributed to \name{}'s effective aggregation of distributed heterogeneous information, corroborating its scalability in unbalanced environments.


\begin{figure*}[t]
    \centering
    \begin{minipage}{0.199\linewidth}
    \centering
    \includegraphics[width=1.0\linewidth]{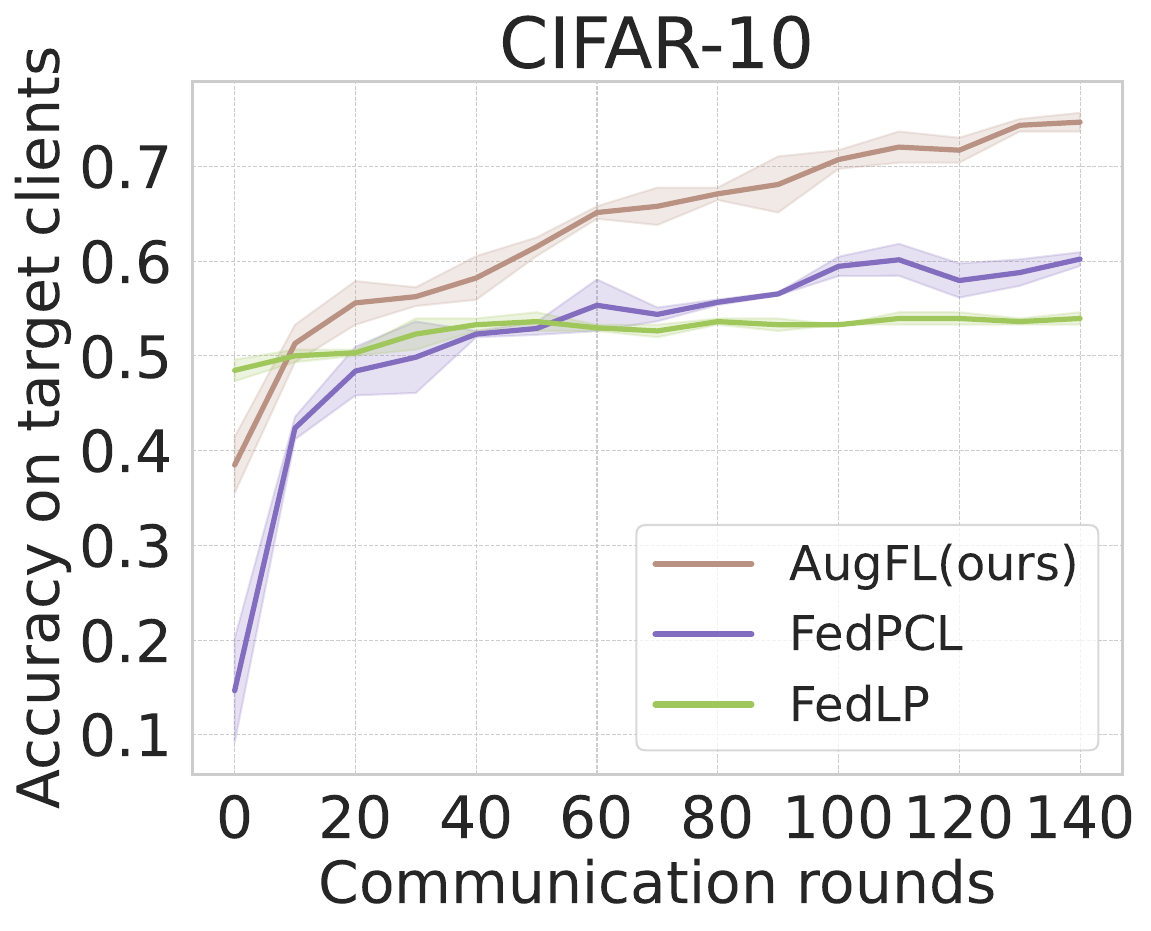}
    \end{minipage}
    \begin{minipage}{0.19\linewidth}
    \centering
    \includegraphics[width=1.0\linewidth]{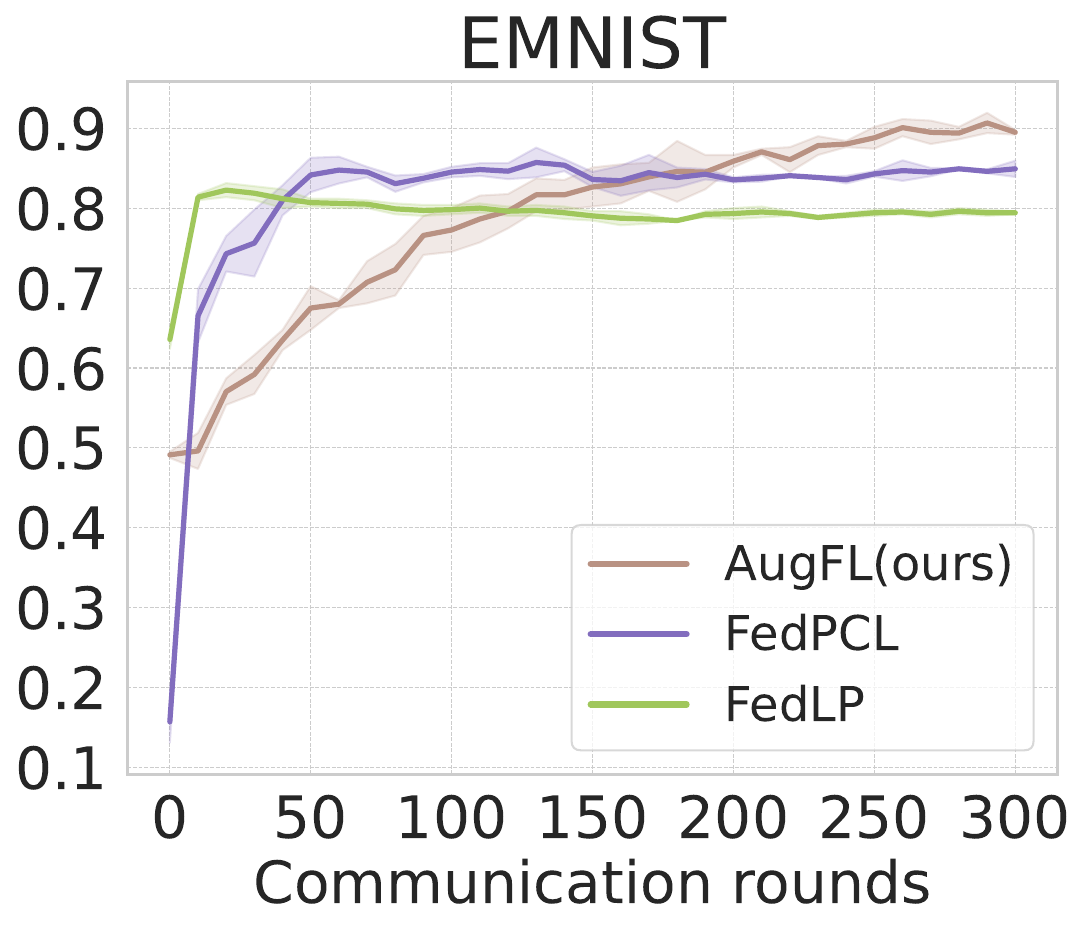}
    \end{minipage}
    \begin{minipage}{0.19\linewidth}
    \centering
    \includegraphics[width=1.0\linewidth]{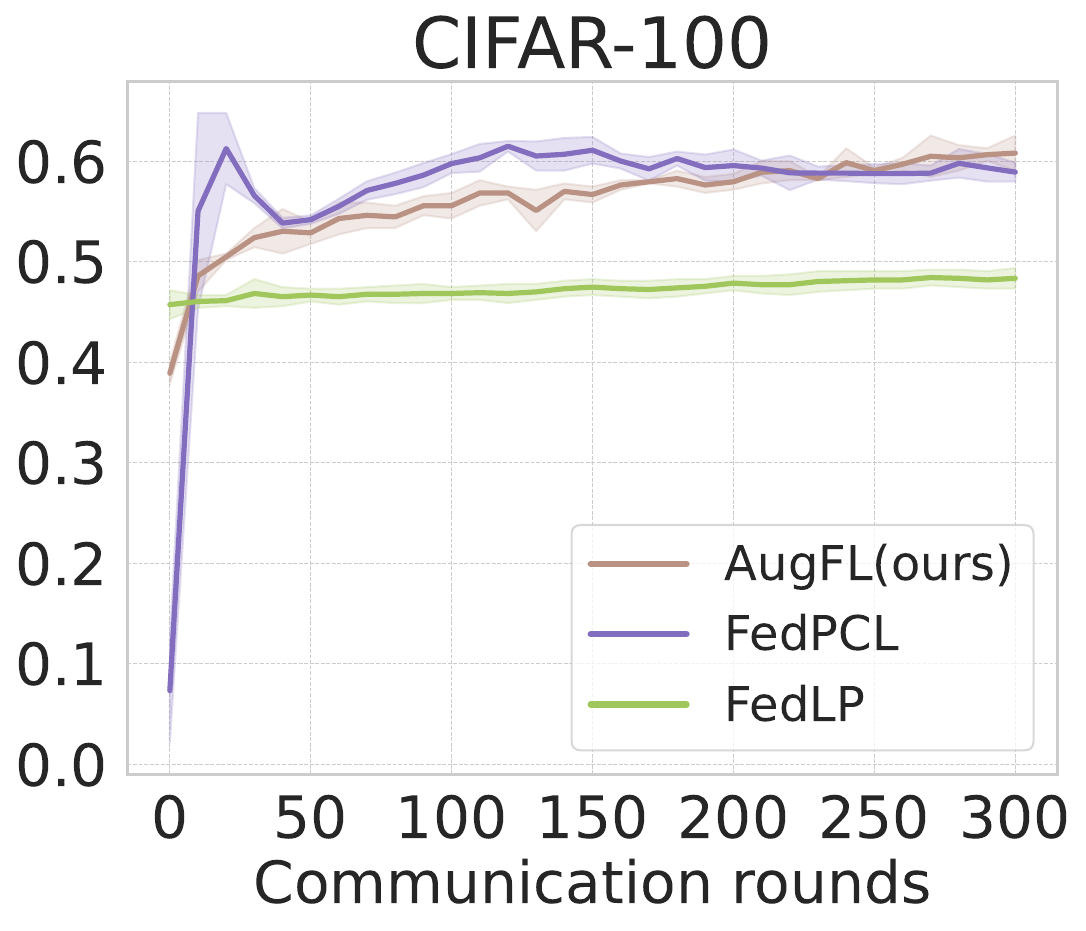}
    \end{minipage}
    \begin{minipage}{0.19\linewidth}
    \centering
    \includegraphics[width=1.0\linewidth]{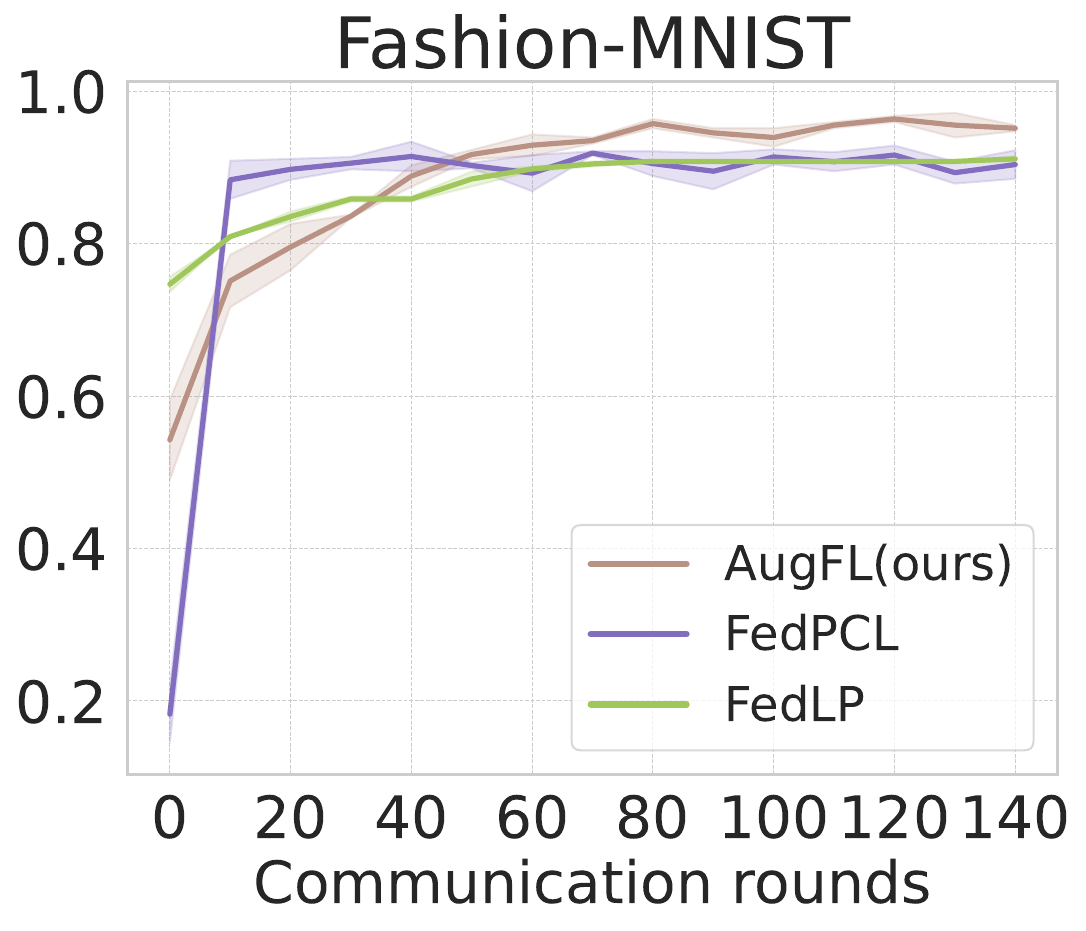}
    \end{minipage}
    \begin{minipage}{0.19\linewidth}
    \centering
    \includegraphics[width=1.0\linewidth]{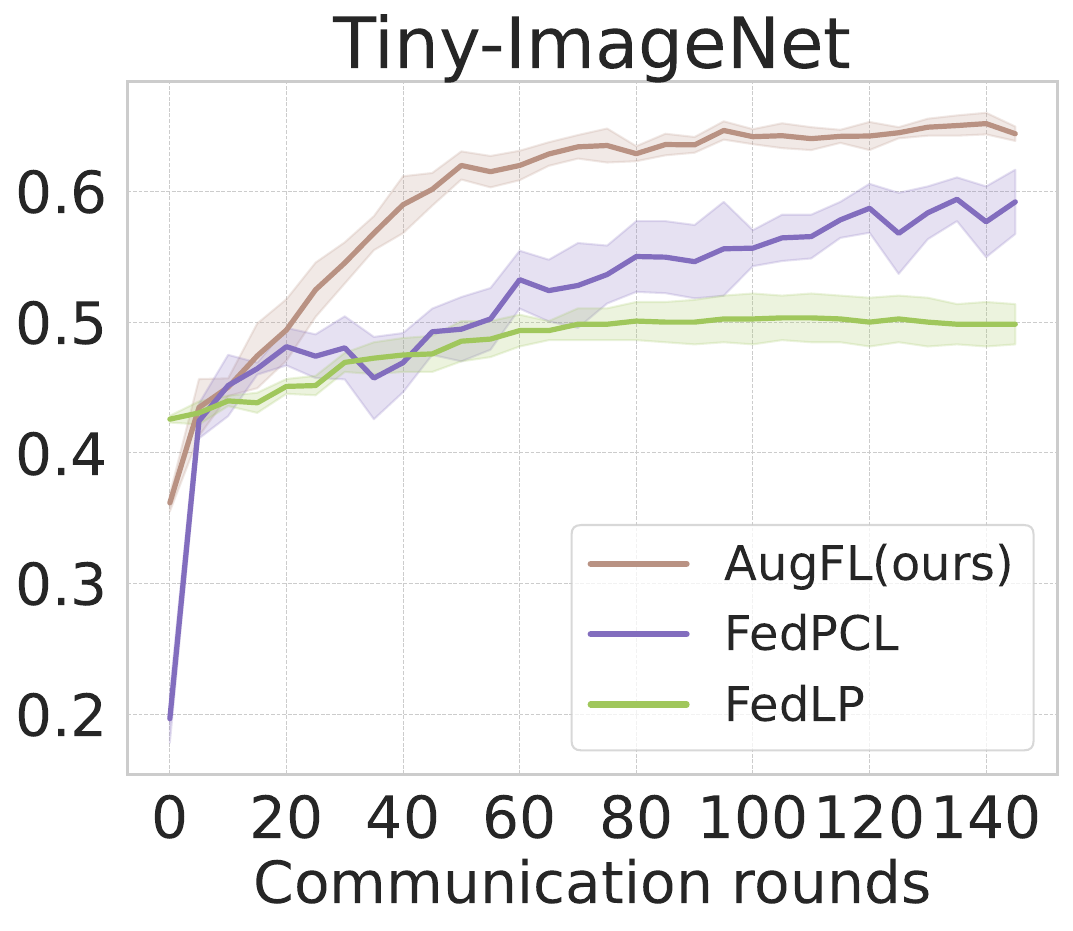}
    \end{minipage}
    \caption{Comparative results between different PM-aided FL methods.}
    \label{fig:AugFL_PM_baselines}
    \vspace{-0.2in}
\end{figure*}

\textbf{Results of PM-aided FL.} For fair comparison, we deploy a pretrained ResNet8x4 as the initial client model for FedLP and FedPCL, while \name{} uses the same PM only on the server for knowledge distillation. The results are shown in \cref{fig:AugFL_PM_baselines}. While FedLP initially performs well due to the PM initialization, the continual federated learning causes the model to easily forget this initial knowledge. FedPCL also shows subpar performance, as its frozen local encoders limit its ability to capture client-specific domain knowledge. Furthermore, deploying the PM on clients would introduce potentially high computational and storage costs and may compromise the ownership of the PM.


\begin{figure}[t]
	\centering
	\begin{minipage}{0.49\linewidth}
		\centering
		\includegraphics[width=1.0\linewidth]{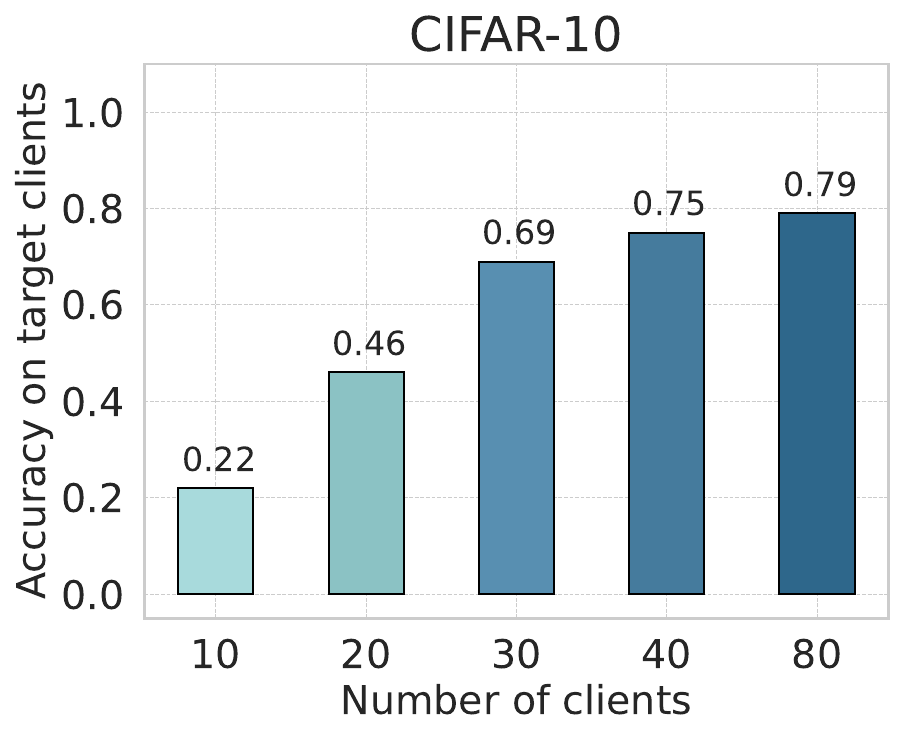}
	\end{minipage}
	\begin{minipage}{0.49\linewidth}
		\centering
		\includegraphics[width=1.0\linewidth]{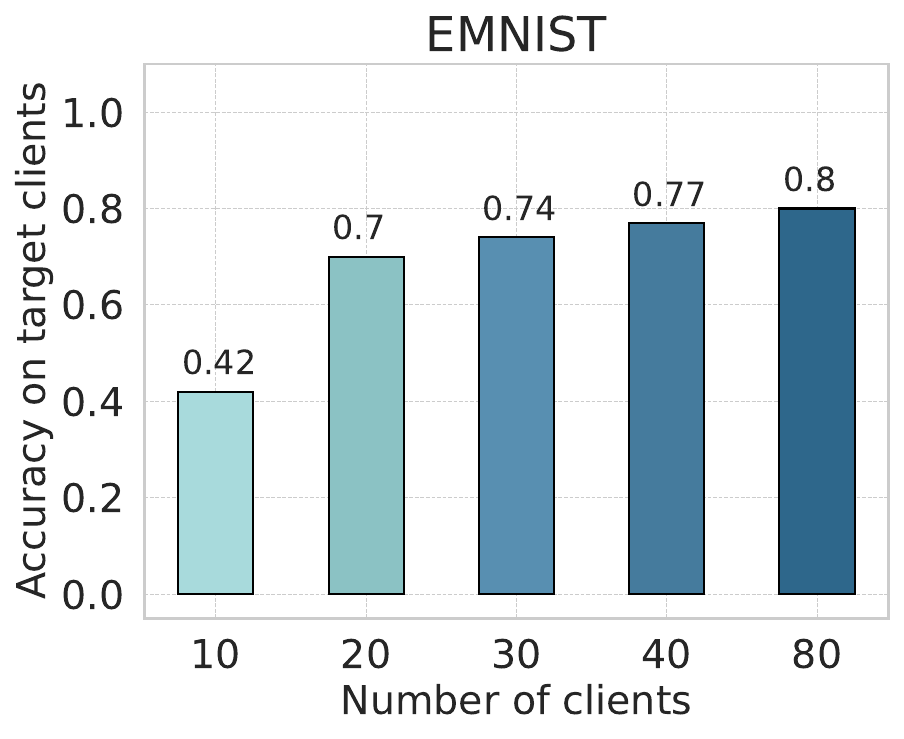}
	\end{minipage}
    \caption{Impact of network scales on \name{}.}
    \label{fig:node_number}
    \vspace{-0.2in}
\end{figure}

\begin{figure}[t]
    \centering
    \vspace{0.35em}
    \begin{minipage}{0.49\linewidth}
    \centering
    \includegraphics[width=1.0\linewidth]{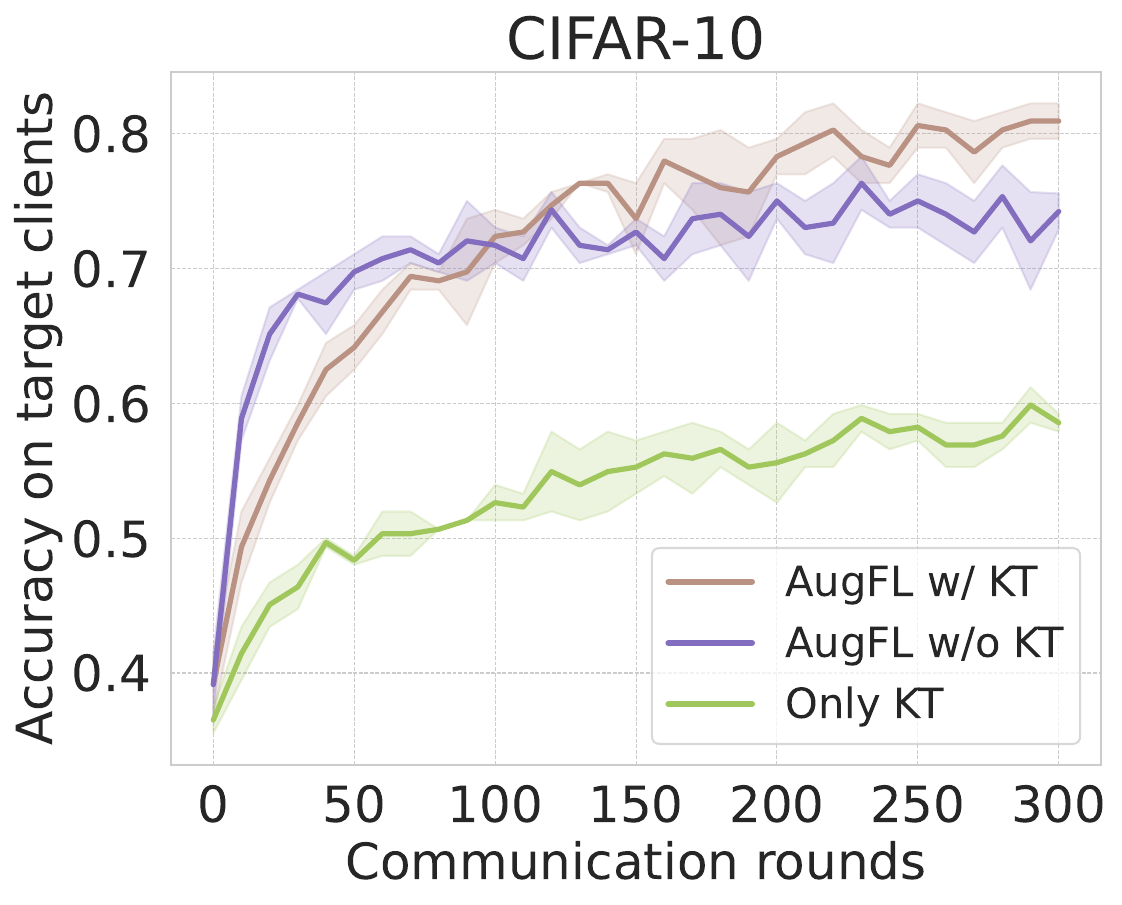}
    \end{minipage}
    \begin{minipage}{0.49\linewidth}
    \centering
    \includegraphics[width=1.0\linewidth]{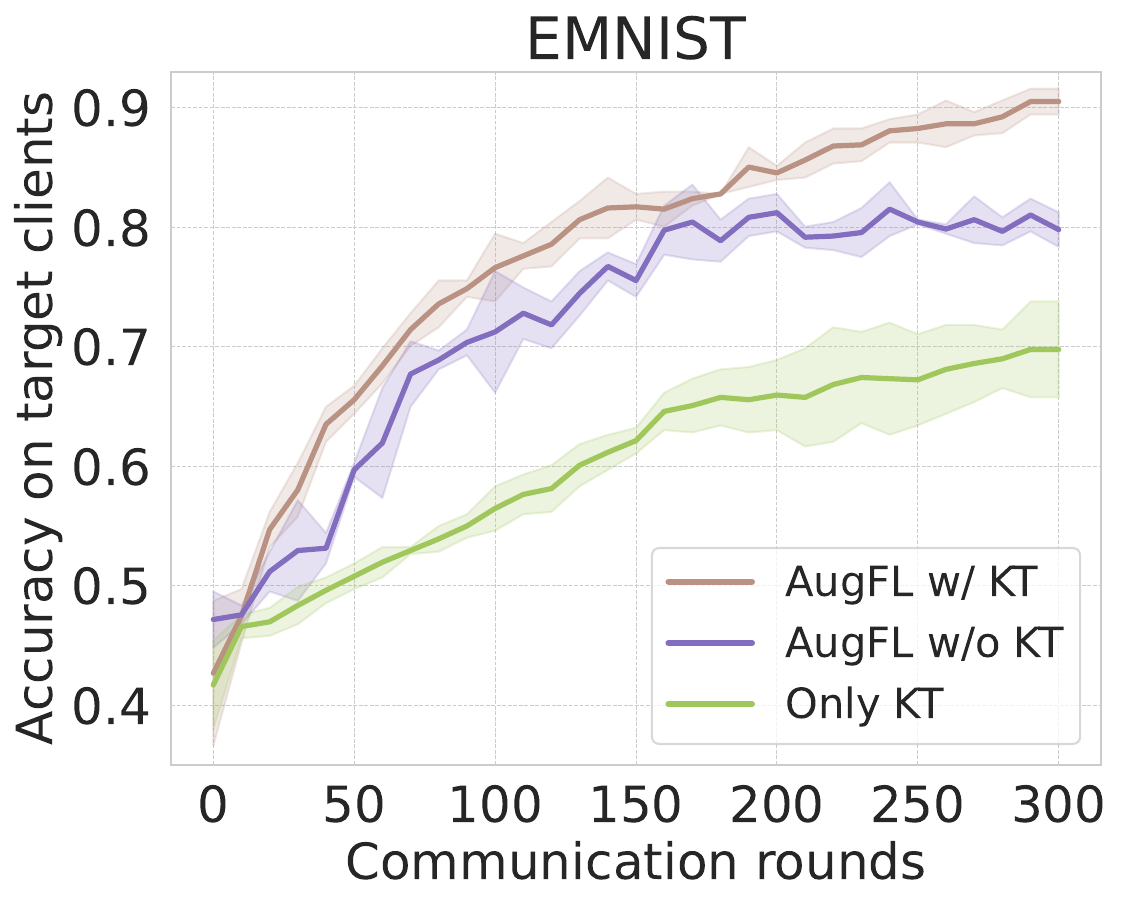}
    \end{minipage}
    
    \begin{minipage}{0.49\linewidth}
    \centering
    \includegraphics[width=1.0\linewidth]{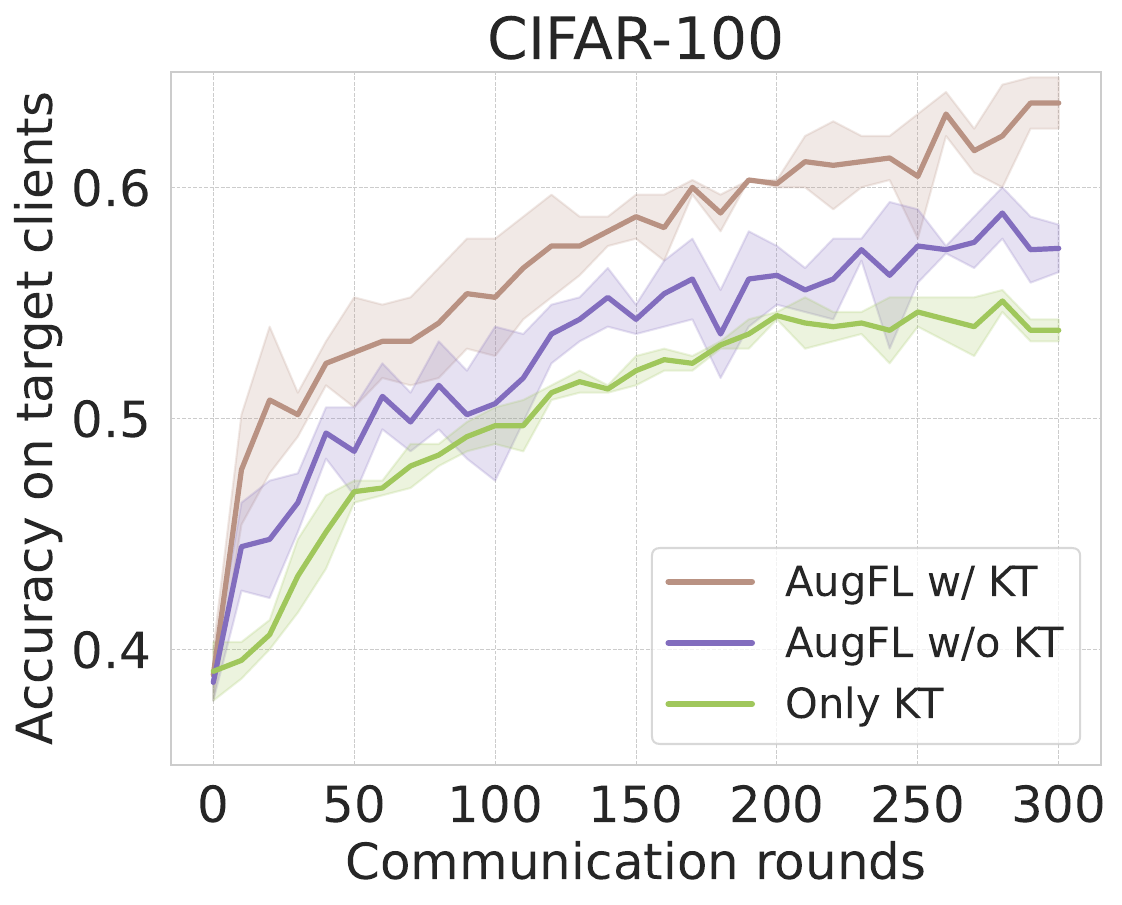}
    \end{minipage}
    \begin{minipage}{0.49\linewidth}
    \centering
    \includegraphics[width=1.0\linewidth]{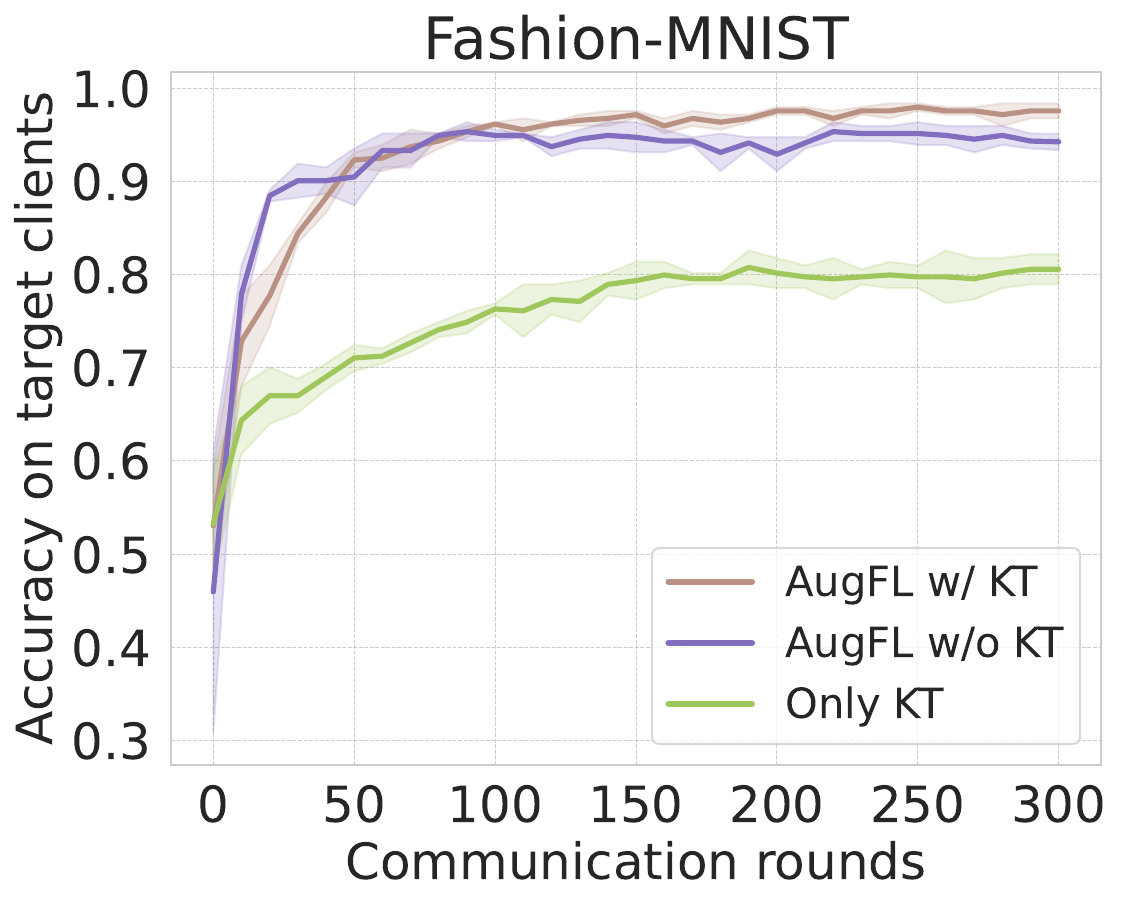}
    \end{minipage}
    \caption{Impact of knowledge transfer.}
    \label{fig:regularization}
    \vspace{-0.2in}
\end{figure}

\textbf{Benefit of the pretrained model.} To showcase the impact of the PM on performance enhancement, we partition the benchmark data into two parts. The first 10,000 samples are dedicated to knowledge transfer on the server, while clients' local data are sampled from the remaining part. The results, depicted in \cref{fig:regularization} (where KT denotes knowledge transfer), show that the knowledge transfer from the PM significantly improves the performance. This is because the PM enhances the structural knowledge of the meta-model, enabling client models to better understand the input figures. In addition, we ablate the FL and use $R_h(\theta,\theta_p)$ solely for knowledge distillation. As illustrated in \cref{fig:regularization}, only knowledge distillation works not well due to a lack of domain-specific knowledge.
\begin{figure}[t]
	\centering
	\begin{minipage}{0.49\linewidth}
		\centering
		\includegraphics[width=1.0\linewidth]{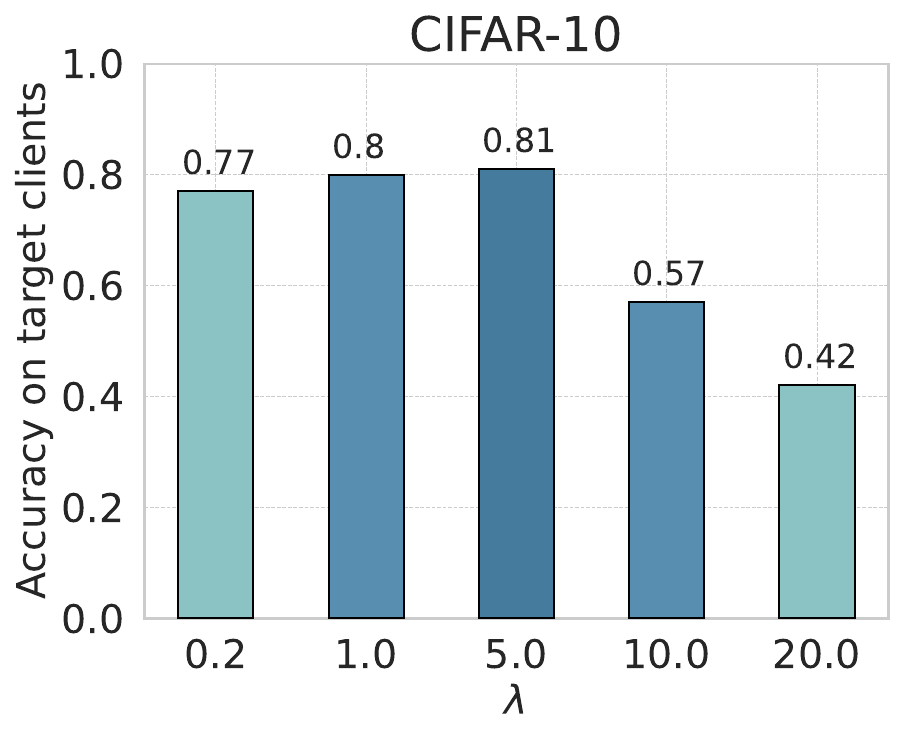}
	\end{minipage}
	\begin{minipage}{0.49\linewidth}
		\centering
		\includegraphics[width=1.0\linewidth]{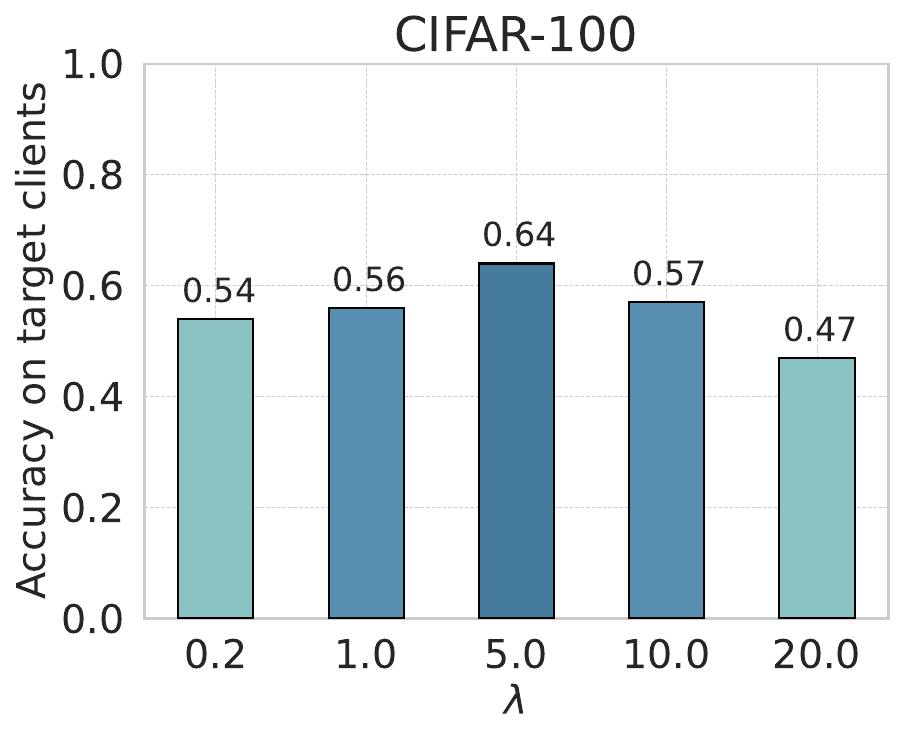}
	\end{minipage}
	\caption{Performance under different values of $\lambda$.}
    \label{fig:impact_lambda}
    \vspace{-0.1in}
\end{figure}

\textbf{Impact of $\lambda$.} We conduct hyper-parameter experiments to showcase the impact of the hyper-parameter $\lambda$. As indicated in \cref{fig:impact_lambda}, when $\lambda$ is set too low, \name{} struggles to fully utilize the PM. On the contrary, a high value of $\lambda$ can limit the model to learn the knowledge in client data. We found that $\lambda=5$ consistently yielded good performance, and hence we applied this value in our experiments.

\begin{figure}[h]
\centering
	\subfloat{\includegraphics[width = 0.5\linewidth]{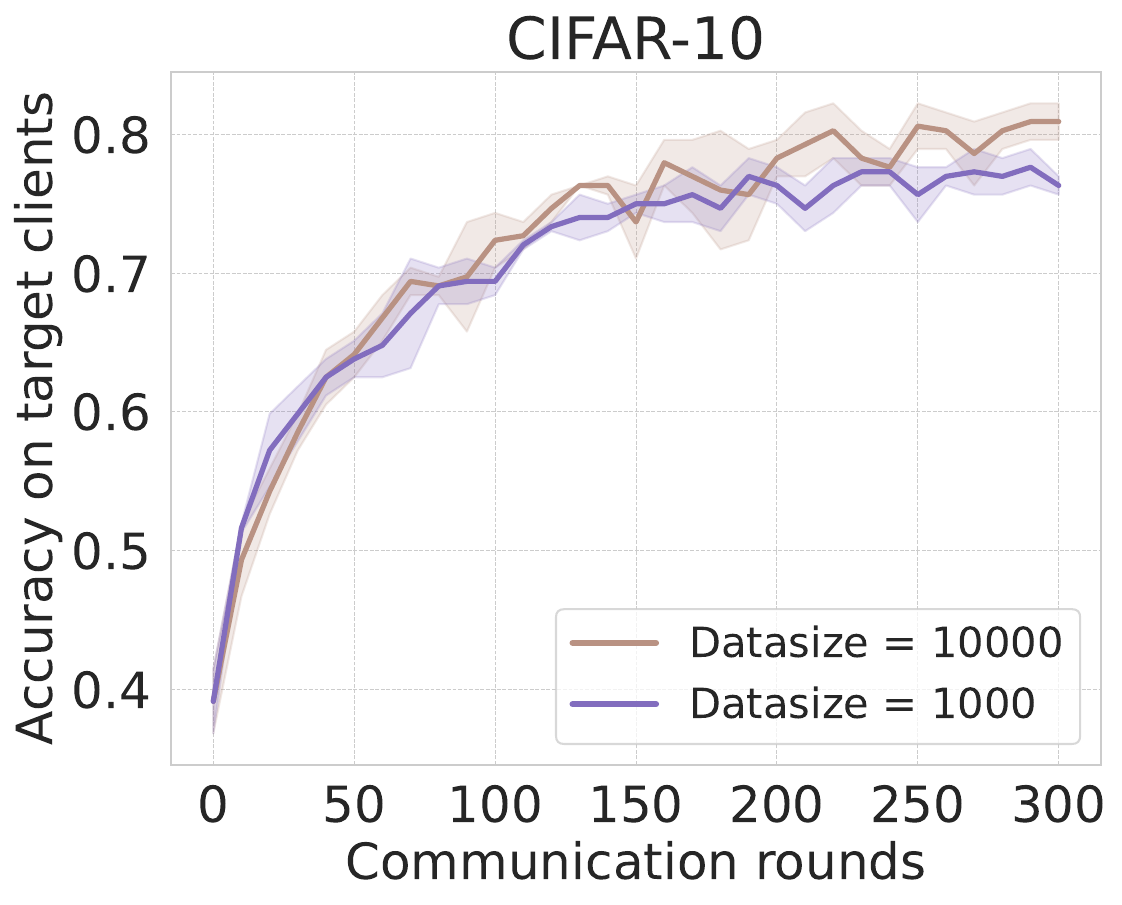}
        \includegraphics[width = 0.5\linewidth]{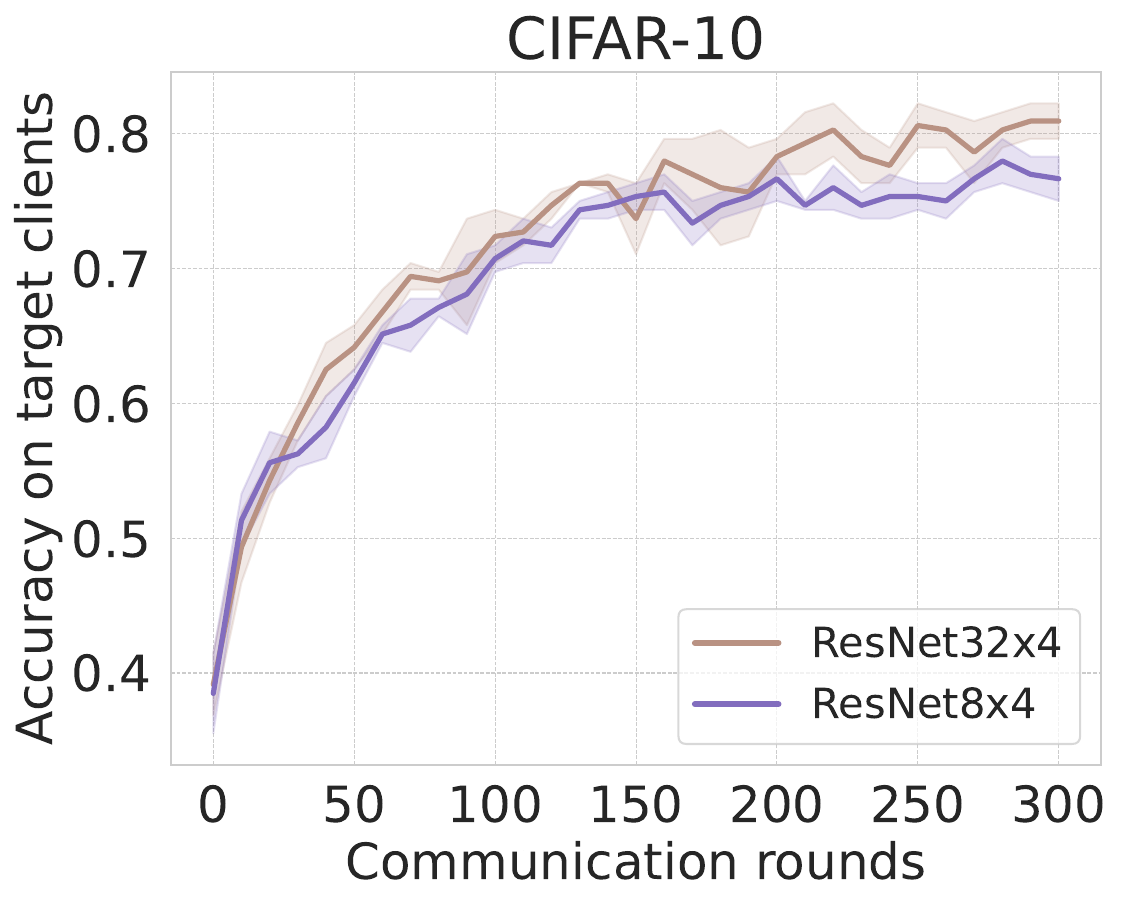}}
	\hfill
	\subfloat{\includegraphics[width = 0.5\linewidth]{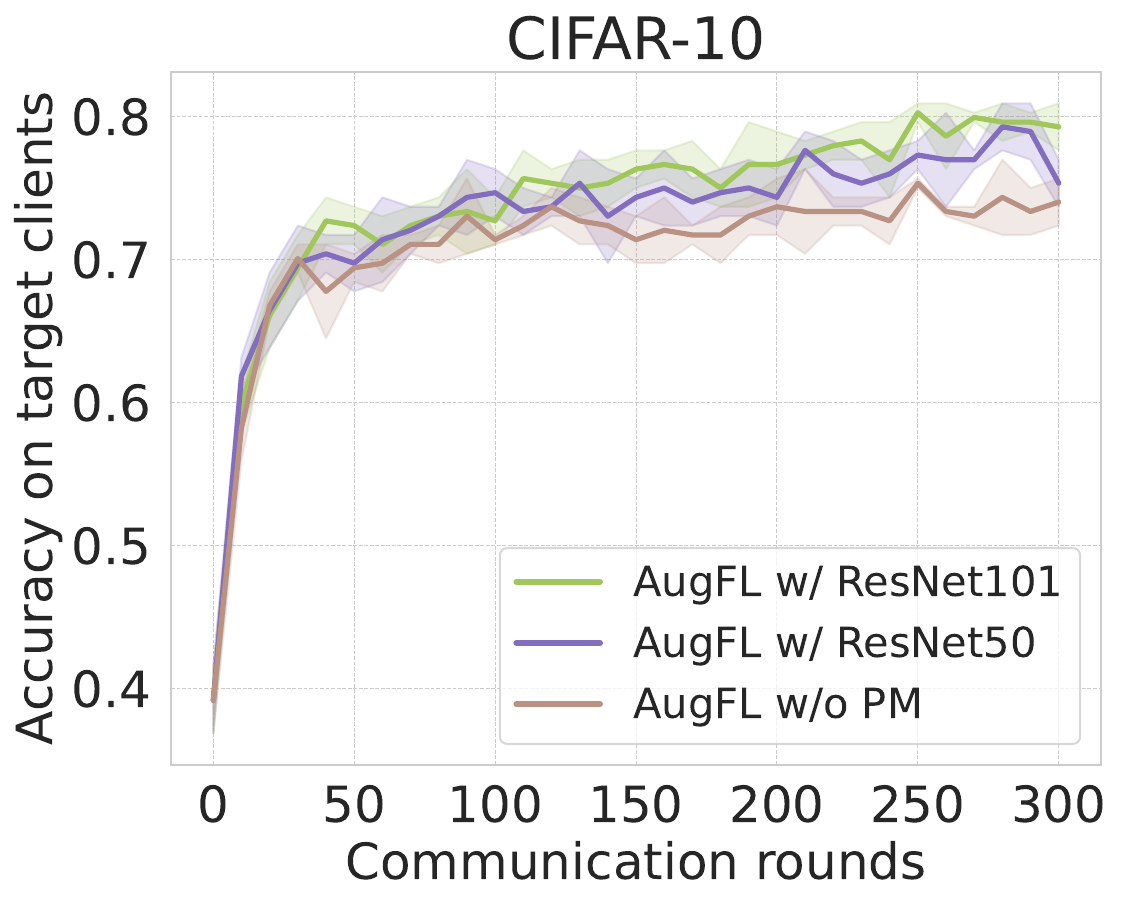}
        \includegraphics[width = 0.5\linewidth]{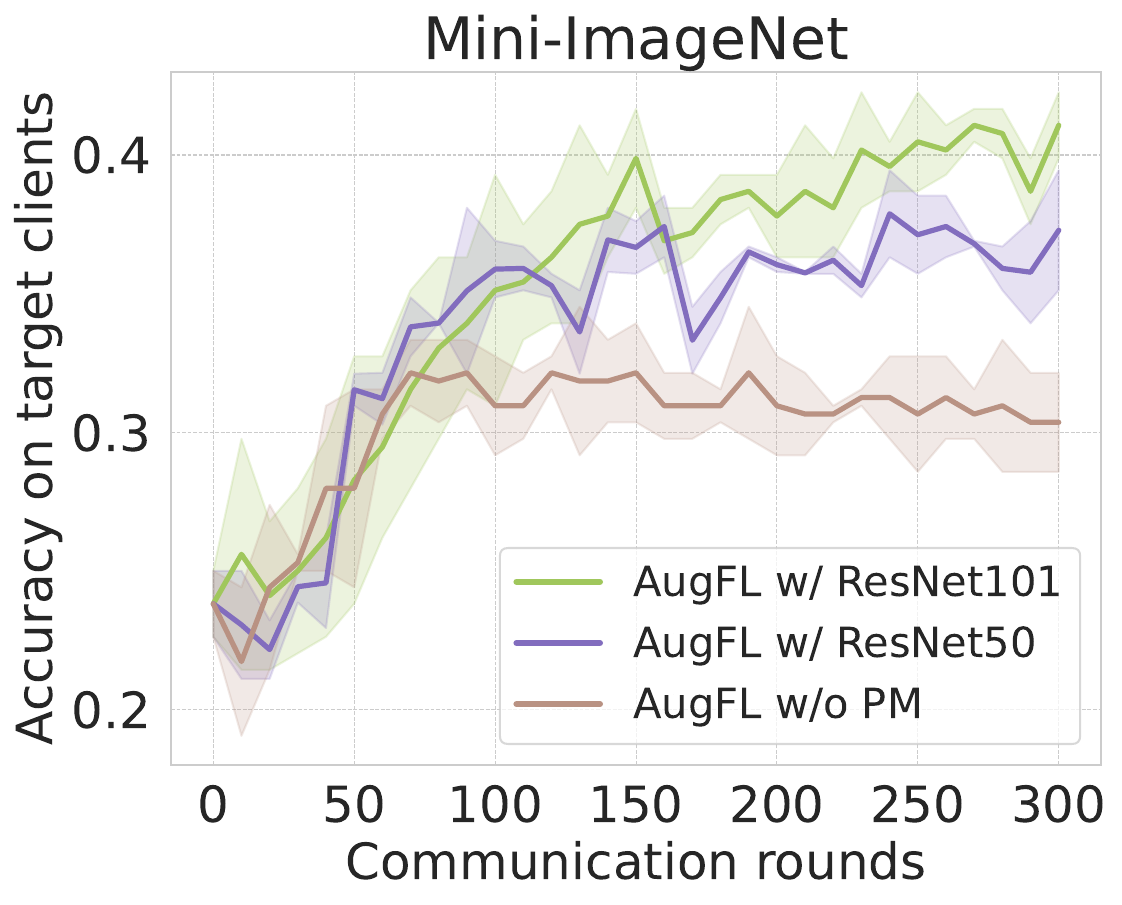}} 
\caption{\textcolor{black}{Performance of \name{} with different PMs}}
\label{fig:diff_PMs}
\end{figure}

\textcolor{black}{\textbf{Results under different pretrained models.} 
We explore the impact of PM on performance by varying the model size and the pretraining dataset. As demonstrated in \cref{fig:diff_PMs}, larger pretrained model and pretraining dataset lead to improved performance. This implies that \name{} can effectively utilize the capability of the PM. To further demonstrate the adaptability of our proposed framework, we incorporate commonly used pretrained ResNet models - \texttt{resnet50} and \texttt{resnet101} pretrained on ImageNet from PyTorch \cite{resnet_pytorch}. We conduct complementary experiments on CIFAR-10 and Mini-ImageNet~\cite{vigneswaran2021feature}. As shown in \cref{fig:diff_PMs}, \name{}'s performance improves when using these ResNet models as the PM, with the gains becoming more significant as the PM size increases. In particular, the similarity between the pretraining domain and the target domain lead to more pronounced performance improvement on Mini-ImageNet. The results demonstrate \name{}'s generalizability across a variety of PMs, aligning with our underlying motivation and theoretical results.}

\section{Conclusion}
\label{sec:conclustion}

In this paper, we introduced an inexact-ADMM-based PM-aided personalized FL framework that enjoys low computational cost with no need to expose the information of the PM to clients. We rigorously analyzed the convergence property and adaptation performance of the proposed method and corroborated its superiority over existing baselines across a suite of commonly used benchmarks. In future work, we plan to deploy the introduced framework in practical recommendation systems with knowledge transfer from well-established LLMs. Besides, we will explore enhancing \name{}'s privacy protection by addressing Byzantine attacks and the potential privacy leakage in aggregation.

\section*{Acknowledgments}
This research was supported in part by the National Key R\&D Program of China (No. 2022YFB2703303), the National Natural Science Foundation of China (Nos. 62432004, 62411540034, and 62402267), and a grant from the Guoqiang Institute, Tsinghua University. We also thank Prof. Xiaochun Cao for his valuable guidance and suggestions throughout this research.



\bibliographystyle{IEEEtran}
\bibliography{reference}

\begin{thebibliography}{10}
\providecommand{\url}[1]{#1}
\csname url@samestyle\endcsname
\providecommand{\newblock}{\relax}
\providecommand{\bibinfo}[2]{#2}
\providecommand{\BIBentrySTDinterwordspacing}{\spaceskip=0pt\relax}
\providecommand{\BIBentryALTinterwordstretchfactor}{4}
\providecommand{\BIBentryALTinterwordspacing}{\spaceskip=\fontdimen2\font plus
\BIBentryALTinterwordstretchfactor\fontdimen3\font minus \fontdimen4\font\relax}
\providecommand{\BIBforeignlanguage}[2]{{%
\expandafter\ifx\csname l@#1\endcsname\relax
\typeout{** WARNING: IEEEtran.bst: No hyphenation pattern has been}%
\typeout{** loaded for the language `#1'. Using the pattern for}%
\typeout{** the default language instead.}%
\else
\language=\csname l@#1\endcsname
\fi
#2}}
\providecommand{\BIBdecl}{\relax}
\BIBdecl

\bibitem{yue2021inexact}
S.~Yue, J.~Ren, J.~Xin, S.~Lin, and J.~Zhang, ``Inexact-{ADMM} based federated meta-learning for fast and continual edge learning,'' in \emph{Proceedings of the 22nd ACM International Symposium on Theory, Algorithmic Foundations, and Protocol Design for Mobile Networks and Mobile Computing (ACM MobiHoc)}, 2021, p. 91–100.

\bibitem{mcmahan2017communication}
B.~McMahan, E.~Moore, D.~Ramage, S.~Hampson, and B.~A.~y. Arcas, ``Communication-efficient learning of deep networks from decentralized data,'' in \emph{Proceedings of the 20th International Conference on Artificial Intelligence and Statistics (AISTATS)}, vol.~54, 2017, pp. 1273--1282.

\bibitem{li2020federated}
T.~Li, A.~K. Sahu, A.~Talwalkar, and V.~Smith, ``Federated learning: Challenges, methods, and future directions,'' \emph{IEEE Signal Processing Magazine}, vol.~37, no.~3, pp. 50--60, 2020.

\bibitem{kairouz2021advances}
P.~Kairouz, H.~B. McMahan, B.~Avent, A.~Bellet, M.~Bennis, A.~N. Bhagoji, K.~Bonawitz, Z.~Charles, G.~Cormode, R.~Cummings \emph{et~al.}, ``Advances and open problems in federated learning,'' \emph{Foundations and Trends{\textregistered} in Machine Learning}, vol.~14, no. 1--2, pp. 1--210, 2021.

\bibitem{bonawitz2019towards}
K.~Bonawitz, H.~Eichner, W.~Grieskamp, D.~Huba, A.~Ingerman, V.~Ivanov, C.~Kiddon, J.~Kone\v{c}n\'{y}, S.~Mazzocchi, B.~McMahan, T.~Van~Overveldt, D.~Petrou, D.~Ramage, and J.~Roselander, ``Towards federated learning at scale: System design,'' in \emph{Proceedings of Machine Learning and Systems}, vol.~1, 2019, pp. 374--388.

\bibitem{dinh2020federated}
C.~T. Dinh, N.~H. Tran, M.~N. Nguyen, C.~S. Hong, W.~Bao, A.~Y. Zomaya, and V.~Gramoli, ``Federated learning over wireless networks: Convergence analysis and resource allocation,'' \emph{IEEE/ACM Transactions on Networking}, vol.~29, no.~1, pp. 398--409, 2020.

\bibitem{nguyen2021federated}
D.~C. Nguyen, M.~Ding, P.~N. Pathirana, A.~Seneviratne, J.~Li, and H.~V. Poor, ``Federated learning for {I}nternet of {T}hings: A comprehensive survey,'' \emph{IEEE Communications Surveys \& Tutorials}, vol.~23, no.~3, pp. 1622--1658, 2021.

\bibitem{xue2021asynchronous}
M.~Xue and W.~Chenglin, ``An asynchronous quasi-cloud/edge/client collaborative federated learning mechanism for fault diagnosis,'' \emph{Chinese Journal of Electronics}, vol.~30, no.~5, pp. 969--977, 2021.

\bibitem{li2021privacy}
Y.~Li, X.~Tao, X.~Zhang, J.~Liu, and J.~Xu, ``Privacy-preserved federated learning for autonomous driving,'' \emph{IEEE Transactions on Intelligent Transportation Systems}, vol.~23, no.~7, pp. 8423--8434, 2021.

\bibitem{gecer2024federated}
M.~Gecer and B.~Garbinato, ``Federated learning for mobility applications,'' \emph{ACM Computing Surveys}, vol.~56, no.~5, pp. 1--28, 2024.

\bibitem{deng2024communication}
Y.~Deng, F.~Lyu, T.~Xia, Y.~Zhou, Y.~Zhang, J.~Ren, and Y.~Yang, ``A communication-efficient hierarchical federated learning framework via shaping data distribution at edge,'' \emph{IEEE/ACM Transactions on Networking}, 2024.

\bibitem{yang2019federated}
Q.~Yang, Y.~Liu, T.~Chen, and Y.~Tong, ``Federated machine learning: Concept and applications,'' \emph{ACM Transactions on Intelligent Systems and Technology}, vol.~10, no.~2, pp. 1--19, 2019.

\bibitem{zhang2022federated}
T.~Zhang, L.~Gao, C.~He, M.~Zhang, B.~Krishnamachari, and A.~S. Avestimehr, ``Federated learning for the {I}nternet of {T}hings: Applications, challenges, and opportunities,'' \emph{IEEE Internet of Things Magazine}, vol.~5, no.~1, pp. 24--29, 2022.

\bibitem{wang2023attrleaks}
Z.~Wang, K.~Liu, J.~Hu, J.~Ren, H.~Guo, and W.~Yuan, ``Attrleaks on the edge: Exploiting information leakage from privacy-preserving co-inference,'' \emph{Chinese Journal of Electronics}, vol.~32, no.~1, pp. 1--12, 2023.

\bibitem{brown2020language}
T.~Brown, B.~Mann, N.~Ryder, M.~Subbiah, J.~D. Kaplan, P.~Dhariwal, A.~Neelakantan, P.~Shyam, G.~Sastry, A.~Askell \emph{et~al.}, ``Language models are few-shot learners,'' in \emph{Advances in Neural Information Processing Systems (NeurIPS)}, vol.~33, 2020, pp. 1877--1901.

\bibitem{collins2019quantifying}
J.~Collins, D.~Howard, and J.~Leitner, ``Quantifying the reality gap in robotic manipulation tasks,'' in \emph{Proceedings of International Conference on Robotics and Automation (ICRA)}, 2019, pp. 6706--6712.

\bibitem{bommasani2021opportunities}
R.~Bommasani, D.~A. Hudson, E.~Adeli, R.~Altman, S.~Arora, S.~von Arx, M.~S. Bernstein, J.~Bohg, A.~Bosselut, E.~Brunskill \emph{et~al.}, ``On the opportunities and risks of foundation models,'' \emph{arXiv preprint arXiv:2108.07258}, 2021.

\bibitem{achiam2023gpt}
J.~Achiam, S.~Adler, S.~Agarwal, L.~Ahmad, I.~Akkaya, F.~L. Aleman, D.~Almeida, J.~Altenschmidt, S.~Altman, S.~Anadkat \emph{et~al.}, ``{GPT}-4 technical report,'' \emph{arXiv preprint arXiv:2303.08774}, 2023.

\bibitem{chowdhery2023palm}
A.~Chowdhery, S.~Narang, J.~Devlin, M.~Bosma, G.~Mishra, A.~Roberts, P.~Barham, H.~W. Chung, C.~Sutton, S.~Gehrmann \emph{et~al.}, ``{PaLM}: Scaling language modeling with pathways,'' \emph{Journal of Machine Learning Research}, vol.~24, no. 240, pp. 1--113, 2023.

\bibitem{touvron2023llama}
H.~Touvron, T.~Lavril, G.~Izacard, X.~Martinet, M.-A. Lachaux, T.~Lacroix, B.~Rozi{\`e}re, N.~Goyal, E.~Hambro, F.~Azhar \emph{et~al.}, ``{LLaMA}: Open and efficient foundation language models,'' \emph{arXiv preprint arXiv:2302.13971}, 2023.

\bibitem{wei2022emergent}
J.~Wei, Y.~Tay, R.~Bommasani, C.~Raffel, B.~Zoph, S.~Borgeaud, D.~Yogatama, M.~Bosma, D.~Zhou, D.~Metzler \emph{et~al.}, ``Emergent abilities of large language models,'' \emph{Transactions on Machine Learning Research}, 2022.

\bibitem{shi2022just}
W.~Shi, R.~Shea, S.~Chen, C.~Zhang, R.~Jia, and Z.~Yu, ``Just fine-tune twice: Selective differential privacy for large language models,'' in \emph{Proceedings of the 2022 Conference on Empirical Methods in Natural Language Processing (EMNLP)}, 2022, pp. 6327--6340.

\bibitem{hou2023privately}
C.~Hou, H.~Zhan, A.~Shrivastava, S.~Wang, A.~Livshits, G.~Fanti, and D.~Lazar, ``Privately customizing prefinetuning to better match user data in federated learning,'' in \emph{International Workshop on Pitfalls of Limited Data and Computation for Trustworthy ML at ICLR}, 2023.

\bibitem{yu2023selective}
D.~Yu, S.~Gopi, J.~Kulkarni, Z.~Lin, S.~Naik, T.~L. Religa, J.~Yin, and H.~Zhang, ``Selective pre-training for private fine-tuning,'' \emph{arXiv preprint arXiv:2305.13865}, 2023.

\bibitem{yu2021LLMfinetune}
D.~Yu, S.~Naik, A.~Backurs, S.~Gopi, H.~A. Inan, G.~Kamath, J.~Kulkarni, Y.~T. Lee, A.~Manoel, L.~Wutschitz \emph{et~al.}, ``Differentially private fine-tuning of language models,'' in \emph{International Conference on Learning Representations (ICLR)}, 2021.

\bibitem{mireshghallah2022differentially}
F.~Mireshghallah, A.~Backurs, H.~A. Inan, L.~Wutschitz, and J.~Kulkarni, ``Differentially private model compression,'' in \emph{Advances in Neural Information Processing Systems (NeurIPS)}, vol.~35, 2022, pp. 29\,468--29\,483.

\bibitem{li2023backdoor}
X.~Li, S.~Wang, C.~Wu, H.~Zhou, and J.~Wang, ``Backdoor threats from compromised foundation models to federated learning,'' in \emph{International Workshop on Federated Learning in the Age of Foundation Models in Conjunction at NeurIPS}, 2023.

\bibitem{wang2023can}
B.~Wang, Y.~J. Zhang, Y.~Cao, B.~Li, H.~B. McMahan, S.~Oh, Z.~Xu, and M.~Zaheer, ``Can public large language models help private cross-device federated learning?'' in \emph{International Workshop on Efficient Systems for Foundation Models at ICML}, 2023.

\bibitem{yuan2023rethinking}
J.~Yuan, C.~Yang, D.~Cai, S.~Wang, X.~Yuan, Z.~Zhang, X.~Li, D.~Zhang, H.~Mei, X.~Jia \emph{et~al.}, ``Rethinking mobile {AI} ecosystem in the {LLM} era,'' \emph{arXiv preprint arXiv:2308.14363}, 2023.

\bibitem{deng2020model}
L.~Deng, G.~Li, S.~Han, L.~Shi, and Y.~Xie, ``Model compression and hardware acceleration for neural networks: A comprehensive survey,'' \emph{Proceedings of the IEEE}, vol. 108, no.~4, pp. 485--532, 2020.

\bibitem{kang2023grounding}
Y.~Kang, T.~Fan, H.~Gu, L.~Fan, and Q.~Yang, ``Grounding foundation models through federated transfer learning: A general framework,'' \emph{arXiv preprint arXiv:2311.17431}, 2023.

\bibitem{roy2019mitigating}
P.~C. Roy and V.~N. Boddeti, ``Mitigating information leakage in image representations: A maximum entropy approach,'' in \emph{Proceedings of the IEEE/CVF Conference on Computer Vision and Pattern Recognition}, 2019, pp. 2586--2594.

\bibitem{jeong2022privacy}
J.~Jeong, M.~Cho, P.~Benz, J.~Hwang, J.~Kim, S.~Lee, and T.-h. Kim, ``Privacy safe representation learning via frequency filtering encoder,'' \emph{arXiv preprint arXiv:2208.02482}, 2022.

\bibitem{claude2024}
Anthropic, ``The claude 3 model family: Opus, sonnet, haiku,'' Website, 2024, \url{https://www-cdn.anthropic.com/de8ba9b01c9ab7cbabf5c33b80b7bbc618857627/Model\_Card\_Claude\_3.pdf}.

\bibitem{glm2024chatglm}
T.~GLM, A.~Zeng, B.~Xu, B.~Wang, C.~Zhang, D.~Yin, D.~Rojas, G.~Feng, H.~Zhao, H.~Lai, H.~Yu, H.~Wang, J.~Sun, J.~Zhang, J.~Cheng, J.~Gui, J.~Tang, J.~Zhang, J.~Li, L.~Zhao, L.~Wu, L.~Zhong, M.~Liu, M.~Huang, P.~Zhang, Q.~Zheng, R.~Lu, S.~Duan, S.~Zhang, S.~Cao, S.~Yang, W.~L. Tam, W.~Zhao, X.~Liu, X.~Xia, X.~Zhang, X.~Gu, X.~Lv, X.~Liu, X.~Liu, X.~Yang, X.~Song, X.~Zhang, Y.~An, Y.~Xu, Y.~Niu, Y.~Yang, Y.~Li, Y.~Bai, Y.~Dong, Z.~Qi, Z.~Wang, Z.~Yang, Z.~Du, Z.~Hou, and Z.~Wang, ``Chatglm: A family of large language models from glm-130b to glm-4 all tools,'' 2024.

\bibitem{qin2024mooncake}
R.~Qin, Z.~Li, W.~He, M.~Zhang, Y.~Wu, W.~Zheng, and X.~Xu, ``Mooncake: A kvcache-centric disaggregated architecture for llm serving,'' \emph{URL https://arxiv. org/abs/2407.00079}, 2024.

\bibitem{li2022fedipr}
B.~Li, L.~Fan, H.~Gu, J.~Li, and Q.~Yang, ``Fedipr: Ownership verification for federated deep neural network models,'' \emph{IEEE Transactions on Pattern Analysis and Machine Intelligence}, vol.~45, no.~4, pp. 4521--4536, 2022.

\bibitem{fan2023fate}
T.~Fan, Y.~Kang, G.~Ma, W.~Chen, W.~Wei, L.~Fan, and Q.~Yang, ``Fate-llm: A industrial grade federated learning framework for large language models,'' \emph{arXiv preprint arXiv:2310.10049}, 2023.

\bibitem{xiao2023offsite}
G.~Xiao, J.~Lin, and S.~Han, ``Offsite-tuning: Transfer learning without full model,'' \emph{arXiv preprint arXiv:2302.04870}, 2023.

\bibitem{hong2016convergence}
M.~Hong, Z.-Q. Luo, and M.~Razaviyayn, ``Convergence analysis of alternating direction method of multipliers for a family of nonconvex problems,'' \emph{SIAM Journal on Optimization}, vol.~26, no.~1, pp. 337--364, 2016.

\bibitem{wang2019global}
Y.~Wang, W.~Yin, and J.~Zeng, ``Global convergence of {ADMM} in nonconvex nonsmooth optimization,'' \emph{Journal of Scientific Computing}, vol.~78, pp. 29--63, 2019.

\bibitem{barber2024convergence}
R.~F. Barber and E.~Y. Sidky, ``Convergence for nonconvex {ADMM}, with applications to {CT} imaging,'' \emph{Journal of Machine Learning Research}, vol.~25, no.~38, pp. 1--46, 2024.

\bibitem{lin2020collaborative}
S.~Lin, G.~Yang, and J.~Zhang, ``A collaborative learning framework via federated meta-learning,'' in \emph{Proceedings of the 40th IEEE International Conference on Distributed Computing Systems (ICDCS)}, 2020, pp. 289--299.

\bibitem{fallah2020personalized}
A.~Fallah, A.~Mokhtari, and A.~Ozdaglar, ``Personalized federated learning with theoretical guarantees: A model-agnostic meta-learning approach,'' in \emph{Advances in Neural Information Processing Systems (NeurIPS)}, vol.~33, 2020, pp. 3557--3568.

\bibitem{koneny2016fl}
J.~Konecný, H.~B. McMahan, F.~X. Yu, P.~Richtárik, A.~Suresh, and D.~Bacon, ``Federated learning: Strategies for improving communication efficiency,'' \emph{arXiv preprint arXiv:1610.05492, 2016, 8}, 2016.

\bibitem{nguyen2022begin}
J.~Nguyen, J.~Wang, K.~Malik, M.~Sanjabi, and M.~Rabbat, ``Where to begin? on the impact of pre-training and initialization in federated learning,'' \emph{arXiv preprint arXiv:2206.15387}, 2022.

\bibitem{tan2022federated}
Y.~Tan, G.~Long, J.~Ma, L.~Liu, T.~Zhou, and J.~Jiang, ``Federated learning from pre-trained models: A contrastive learning approach,'' \emph{Advances in neural information processing systems}, vol.~35, pp. 19\,332--19\,344, 2022.

\bibitem{zhang2024upload}
J.~Zhang, Y.~Liu, Y.~Hua, and J.~Cao, ``An upload-efficient scheme for transferring knowledge from a server-side pre-trained generator to clients in heterogeneous federated learning,'' in \emph{Proceedings of the IEEE/CVF Conference on Computer Vision and Pattern Recognition}, 2024, pp. 12\,109--12\,119.

\bibitem{he2020group}
C.~He, M.~Annavaram, and S.~Avestimehr, ``Group knowledge transfer: Federated learning of large cnns at the edge,'' \emph{Advances in Neural Information Processing Systems}, vol.~33, pp. 14\,068--14\,080, 2020.

\bibitem{tan2022pfl}
A.~Z. Tan, H.~Yu, L.~Cui, and Y.~Q, ``Towards personalized federated learning,'' \emph{IEEE Transactions on Neural Networks and Learning Systems}, vol.~34, pp. 9587--9603, 2022.

\bibitem{jeong2018communication}
E.~Jeong, S.~Oh, H.~Kim, J.~Park, M.~Bennis, and S.-L. Kim, ``Communication-efficient on-device machine learning: Federated distillation and augmentation under non-iid private data,'' \emph{arXiv preprint arXiv:1811.11479}, 2018.

\bibitem{wang2020optimizing}
H.~Wang, Z.~Kaplan, D.~Niu, and B.~Li, ``Optimizing federated learning on non-{IID} data with reinforcement learning,'' in \emph{Proceedings of the 39th International Conference on Computer Communications (INFOCOM)}, 2020, pp. 1698--1707.

\bibitem{arivazhagan2019federated}
M.~G. Arivazhagan, V.~Aggarwal, A.~K. Singh, and S.~Choudhary, ``Federated learning with personalization layers,'' \emph{arXiv preprint arXiv:1912.00818}, 2019.

\bibitem{smith2017federated}
V.~Smith, C.-K. Chiang, M.~Sanjabi, and A.~S. Talwalkar, ``Federated multi-task learning,'' \emph{Advances in Neural Information Processing Systems (NeurIPS)}, vol.~30, 2017.

\bibitem{finn2017model}
C.~Finn, P.~Abbeel, and S.~Levine, ``Model-agnostic meta-learning for fast adaptation of deep networks,'' in \emph{Proceedings of the 34th International Conference on Machine Learning (ICML)}, vol.~70, 2017, pp. 1126--1135.

\bibitem{chen2018federated}
F.~Chen, M.~Luo, Z.~Dong, Z.~Li, and X.~He, ``Federated meta-learning with fast convergence and efficient communication,'' \emph{arXiv preprint arXiv:1802.07876}, 2018.

\bibitem{jiang2019improving}
Y.~Jiang, J.~Kone{\v{c}}n{\`y}, K.~Rush, and S.~Kannan, ``Improving federated learning personalization via model agnostic meta learning,'' \emph{arXiv preprint arXiv:1909.12488}, 2019.

\bibitem{magnusson2015convergence}
S.~Magn{\'u}sson, P.~C. Weeraddana, M.~G. Rabbat, and C.~Fischione, ``On the convergence of alternating direction {L}agrangian methods for nonconvex structured optimization problems,'' \emph{IEEE Transactions on Control of Network Systems}, vol.~3, no.~3, pp. 296--309, 2015.

\bibitem{wang2014convergence}
F.~Wang, Z.~Xu, and H.-K. Xu, ``Convergence of {B}regman alternating direction method with multipliers for nonconvex composite problems,'' \emph{arXiv preprint arXiv:1410.8625}, 2014.

\bibitem{wang2018convergence}
F.~Wang, W.~Cao, and Z.~Xu, ``Convergence of multi-block {B}regman {ADMM} for nonconvex composite problems,'' \emph{Science China Information Sciences}, vol.~61, pp. 1--12, 2018.

\bibitem{jiang2019structured}
B.~Jiang, T.~Lin, S.~Ma, and S.~Zhang, ``Structured nonconvex and nonsmooth optimization: Algorithms and iteration complexity analysis,'' \emph{Computational Optimization and Applications}, vol.~72, no.~1, pp. 115--157, 2019.

\bibitem{lanza2017nonconvex}
A.~Lanza, S.~Morigi, I.~Selesnick, and F.~Sgallari, ``Nonconvex nonsmooth optimization via convex--nonconvex majorization--minimization,'' \emph{Numerische Mathematik}, vol. 136, pp. 343--381, 2017.

\bibitem{mukkamala2020convex}
M.~C. Mukkamala, P.~Ochs, T.~Pock, and S.~Sabach, ``Convex-concave backtracking for inertial {B}regman proximal gradient algorithms in nonconvex optimization,'' \emph{SIAM Journal on Mathematics of Data Science}, vol.~2, no.~3, pp. 658--682, 2020.

\bibitem{han2021pre}
X.~Han, Z.~Zhang, N.~Ding, Y.~Gu, X.~Liu, Y.~Huo, J.~Qiu, Y.~Yao, A.~Zhang, L.~Zhang \emph{et~al.}, ``Pre-trained models: Past, present and future,'' \emph{AI Open}, vol.~2, pp. 225--250, 2021.

\bibitem{marcelino2018transfer}
P.~Marcelino, ``Transfer learning from pre-trained models,'' \emph{Towards data science}, vol.~10, no. 330, p.~23, 2018.

\bibitem{zhuang2020comprehensive}
F.~Zhuang, Z.~Qi, K.~Duan, D.~Xi, Y.~Zhu, H.~Zhu, H.~Xiong, and Q.~He, ``A comprehensive survey on transfer learning,'' \emph{Proceedings of the IEEE}, vol. 109, no.~1, pp. 43--76, 2020.

\bibitem{park2019relational}
W.~Park, D.~Kim, Y.~Lu, and M.~Cho, ``Relational knowledge distillation,'' in \emph{Proceedings of the IEEE/CVF Conference on Computer Vision and Pattern Recognition (CVPR)}, 2019, pp. 3967--3976.

\bibitem{ji2021show}
M.~Ji, B.~Heo, and S.~Park, ``Show, attend and distill: Knowledge distillation via attention-based feature matching,'' in \emph{Proceedings of the 35th AAAI Conference on Artificial Intelligence (AAAI)}, vol.~35, no.~9, 2021, pp. 7945--7952.

\bibitem{zhao2022decoupled}
B.~Zhao, Q.~Cui, R.~Song, Y.~Qiu, and J.~Liang, ``Decoupled knowledge distillation,'' in \emph{Proceedings of the IEEE/CVF Conference on Computer Vision and Pattern Recognition (CVPR)}, 2022, pp. 11\,953--11\,962.

\bibitem{boyd2011distributed}
S.~Boyd, N.~Parikh, E.~Chu, B.~Peleato, J.~Eckstein \emph{et~al.}, ``Distributed optimization and statistical learning via the alternating direction method of multipliers,'' \emph{Foundations and Trends{\textregistered} in Machine learning}, vol.~3, no.~1, pp. 1--122, 2011.

\bibitem{zhang2021fedpd}
X.~Zhang, M.~Hong, S.~Dhople, W.~Yin, and Y.~Liu, ``{FedPD}: A federated learning framework with adaptivity to non-iid data,'' \emph{IEEE Transactions on Signal Processing}, vol.~69, pp. 6055--6070, 2021.

\bibitem{fallah2020convergence}
A.~Fallah, A.~Mokhtari, and A.~Ozdaglar, ``On the convergence theory of gradient-based model-agnostic meta-learning algorithms,'' in \emph{Proceedings of the 23rd International Conference on Artificial Intelligence and Statistics (AISTATS)}, vol. 108, 2020, pp. 1082--1092.

\bibitem{krishnan2020meta}
R.~Krishnan and P.~Balaprakash, ``Meta continual learning via dynamic programming,'' \emph{arXiv preprint arXiv:2008.02219}, 2020.

\bibitem{tian2019crd}
Y.~Tian, D.~Krishnan, and P.~Isola, ``Contrastive representation distillation,'' in \emph{International Conference on Learning Representations (ICLR)}, 2019.

\bibitem{krizhevsky2009learning}
A.~Krizhevsky, ``Learning multiple layers of features from tiny images,'' \emph{Master's thesis, University of Tront}, 2009.

\bibitem{xiao2017fashion}
H.~Xiao, K.~Rasul, and R.~Vollgraf, ``Fashion-mnist: a novel image dataset for benchmarking machine learning algorithms,'' \emph{arXiv preprint arXiv:1708.07747}, 2017.

\bibitem{cohen2017emnist}
G.~Cohen, S.~Afshar, J.~Tapson, and A.~Van~Schaik, ``Emnist: Extending mnist to handwritten letters,'' in \emph{2017 international joint conference on neural networks (IJCNN)}.\hskip 1em plus 0.5em minus 0.4em\relax IEEE, 2017, pp. 2921--2926.

\bibitem{chrabaszcz2017downsampled}
P.~Chrabaszcz, I.~Loshchilov, and F.~Hutter, ``A downsampled variant of imagenet as an alternative to the cifar datasets,'' \emph{arXiv preprint arXiv:1707.08819}, 2017.

\bibitem{tan2022fedproto}
Y.~Tan, G.~Long, L.~Liu, T.~Zhou, Q.~Lu, J.~Jiang, and C.~Zhang, ``Fedproto: Federated prototype learning across heterogeneous clients,'' in \emph{Proceedings of the AAAI Conference on Artificial Intelligence}, vol.~36, no.~8, 2022, pp. 8432--8440.

\bibitem{zhang2023fedcp}
J.~Zhang, Y.~Hua, H.~Wang, T.~Song, Z.~Xue, R.~Ma, and H.~Guan, ``Fedcp: Separating feature information for personalized federated learning via conditional policy,'' in \emph{Proceedings of the 29th ACM SIGKDD Conference on Knowledge Discovery and Data Mining}, 2023, pp. 3249--3261.

\bibitem{resnet_pytorch}
\BIBentryALTinterwordspacing
PyTorch, ``{ResNet} — {Torchvision} main documentation,'' 2025-02-21. [Online]. Available: \url{https://pytorch.org/vision/main/models/resnet.html}
\BIBentrySTDinterwordspacing

\bibitem{vigneswaran2021feature}
\BIBentryALTinterwordspacing
R.~Vigneswaran, M.~T. Law, V.~N. Balasubramanian, and M.~Tapaswi, ``Feature generation for long-tail classification,'' 2021. [Online]. Available: \url{https://arxiv.org/abs/2111.05956}
\BIBentrySTDinterwordspacing

\end{thebibliography}

\appendix

\section*{Proof of \cref{fsmoothlemma}}

Recall that $F_i(\theta) = L_i(\phi_i(\theta),\mathcal{D}^q_i)$. First, we can obtain $\nabla F_i(x)=(I_n-\alpha\nabla^2 L_i(x,\mathcal{D}^{s}_i))\nabla L_i(x-\alpha\nabla L_i(x,\mathcal{D}^{s}_i),\mathcal{D}^q_i)$. Based on \cref{Lsmooth}, the following facts hold:
\begin{align}
\label{fsmooth_eq1}
\Vert\nabla F_i(x)-\nabla F_i(y)\Vert\le\Vert \nabla L_i(x-\alpha\nabla L_i(x,\mathcal{D}^{s}_i),\mathcal{D}^q_i)\nonumber\\
-\nabla L_i(y-\alpha\nabla L_i(y,\mathcal{D}^{s}_i),\mathcal{D}^q_i) \Vert\\
\nonumber
+\alpha \Vert\nabla^2 L_i(x,\mathcal{D}^{s}_i)\nabla L_i(x-\alpha\nabla L_i(x,\mathcal{D}^{s}_i),\mathcal{D}^q_i)\\
\label{fsmooth_eq2}
-\nabla^2 L_i(y,\mathcal{D}^{s}_i)\nabla L_i(y-\alpha\nabla L_i(y,\mathcal{D}^{s}_i),\mathcal{D}^q_i)\Vert
\end{align}
and 
\begin{align}
    \label{Hessianbound}
    -\mu_i I_n\preceq \nabla^2 L_i(x,\mathcal{D}^{s}_i) \preceq \mu_i I_n,~\forall x\in\mathbb{R}^n.
\end{align}
To prove \cref{fsmooth}, we need to bound \cref{fsmooth_eq1} and \cref{fsmooth_eq2}. For \cref{fsmooth_eq1}, based on \cref{Lsmooth}, we can write 
\begin{align}
    \label{fsmooth_eq3}
    \nonumber
    \Vert \nabla L_i(x-\alpha\nabla L_i(x,\mathcal{D}^{s}_i),\mathcal{D}^q_i)-\nabla L_i(y-\alpha\nabla L_i(y,\mathcal{D}^{s}_i),\mathcal{D}^q_i) \Vert\\
    \nonumber
    \le\mu_i\Vert x-y-\alpha(\nabla L_i(x,\mathcal{D}^{s}_i)-\nabla L_i(y,\mathcal{D}^{s}_i))\Vert\\
    \nonumber
    \le\mu_i( \Vert x-y\Vert + \alpha\Vert \nabla L_i(x,\mathcal{D}^{s}_i)-\nabla L_i(y,\mathcal{D}^{s}_i)\Vert)\\
    \le (1+\alpha\mu_i)\mu_i\Vert x-y\Vert.
\end{align}
To bound \cref{fsmooth_eq2}, it can be shown that
\begin{align}
    \label{fsmooth_eq4}
    \Vert\nabla^2 L_i(x,\mathcal{D}^{s}_i)\nabla L_i(x-\alpha\nabla L_i(x,\mathcal{D}^{s}_i),\mathcal{D}^q_i)-\nabla^2 L_i(y,\mathcal{D}^{s}_i)\nonumber\\
    \cdot\nabla L_i(y-\alpha\nabla L_i(y,\mathcal{D}^{s}_i),\mathcal{D}^q_i)\Vert\nonumber\\
    =\Vert\nabla^2 L_i(x,\mathcal{D}^{s}_i)\nabla L_i(x-\alpha\nabla L_i(x,\mathcal{D}^{s}_i),\mathcal{D}^q_i)\nonumber\\
    -\nabla^2 L_i(x,\mathcal{D}^{s}_i)\nabla L_i(y-\alpha\nabla L_i(y,\mathcal{D}^{s}_i),\mathcal{D}^q_i)\nonumber\\
    +\nabla^2 L_i(x,\mathcal{D}^{s}_i)\nabla L_i(y-\alpha\nabla L_i(y,\mathcal{D}^{s}_i),\mathcal{D}^q_i)\nonumber\\
    -\nabla^2 L_i(y,\mathcal{D}^{s}_i)\nabla L_i(y-\alpha\nabla L_i(y,\mathcal{D}^{s}_i),\mathcal{D}^q_i)\Vert\nonumber\\
    \le \Vert \nabla^2 L_i(x,\mathcal{D}^{s}_i)\Vert\Vert \nabla L_i(x-\alpha\nabla L_i(x,\mathcal{D}^{s}_i),\mathcal{D}^q_i)\nonumber\\
    -\nabla L_i(y-\alpha\nabla L_i(y,\mathcal{D}^{s}_i),\mathcal{D}^q_i)\Vert\nonumber
    \\
    +\Vert \nabla^2 L_i(x,\mathcal{D}^{s}_i)-\nabla^2 L_i(y,\mathcal{D}^{s}_i)\Vert\nonumber\\
    \cdot\Vert\nabla L_i(y-\alpha\nabla L_i(y,\mathcal{D}^{s}_i),\mathcal{D}^q_i)\Vert\nonumber\\
    \nonumber
    \le ((1+\alpha\mu_i)\mu^2_i+\zeta_i\Vert\nabla L_i(y-\alpha\nabla L_i(y,\mathcal{D}^{s}_i),\mathcal{D}^q_i)\Vert)\Vert x-y\Vert\\
    \le ((1+\alpha\mu_i)\mu^2_i+\zeta_i\beta_i)\Vert x-y\Vert
\end{align}
where \cref{fsmooth_eq4} follows \cref{gradientbound,Hessianbound,fsmooth_eq3} and \cref{HessianLipschitz}. 
Combining \cref{fsmooth_eq3,fsmooth_eq4} yields the result.

\section*{Proof of \cref{dualbound}}

First, denote
\begin{align}
    \label{tilde_F}
    \tilde{\nabla}F_i(\theta^{t+1})\doteq \nabla f^q_i(\phi^{t+1}_i)-\alpha g^{t+1}_i
\end{align}
with $\phi^{t+1}_i=\theta^{t+1}-\alpha\nabla f_i(\theta^{t+1})$. From \cref{eq:update_thetai_inex}, we have the following observation:
\begin{align}
    \label{dualbound_eq1}
    w_i\tilde{\nabla}F_i(\theta^{t+1})+ y^t_i+\rho_i(\theta^{t+1}_i-\theta^{t+1})=0.
\end{align}
Using \cref{eq:update_y}, we conclude that \cref{dualbound_eq1} is equivalent to
\begin{align}
    \label{dualbound_eq2}
    - y^{t+1}_i=w_i\tilde{\nabla}F_i(\theta^{t+1}).
\end{align}
Thus, for all $t\in\mathbb{N}$, the following holds:
\begin{align}
    \label{dualbound_eq3}
    \Vert y^{t+1}_i- y^t_i\Vert= w_i\underbrace{\Vert\tilde{\nabla}F_i(\theta^{t+1})- \tilde{\nabla}F_i(\theta^t)\Vert}_{\text{(a)}}.
\end{align}
Using \cref{fsmoothlemma}, we derive the upper bound of (a) as follows:
\begin{align}
    \label{dualbound_eq4}
    \nonumber
    \text{(a)}
    =&\;\Vert(\nabla F_i(\theta^{t+1})-\tilde{\nabla}F_i(\theta^{t+1}))-(\nabla F_i(\theta^{t+1})-\tilde{\nabla}F_i(\theta^t))\Vert\\
    \nonumber
    \le&\;\Vert\nabla F_i(\theta^{t+1})-\tilde{\nabla}F_i(\theta^{t+1})\Vert+\Vert\nabla F_i(\theta^{t+1})-\tilde{\nabla}F_i(\theta^t)\Vert\\
    \nonumber
    \le&\;\Vert\nabla F_i(\theta^{t+1})-\tilde{\nabla}F_i(\theta^{t+1})\Vert+\Vert\nabla F_i(\theta^{t+1})-\nabla F_i(\theta^t)\Vert\\
    &+\Vert\nabla F_i(\theta^t)-\tilde{\nabla}F_i(\theta^t)\Vert\nonumber\\
    \le&\;\nu_i\Vert\theta^{t+1}-\theta^t\Vert + (\delta_{i,t}+\delta_{i,t+1})\alpha\zeta_i\beta^2_i
\end{align}
where the last equality uses the following result in  \cite{fallah2020convergence}
\begin{align*}
    \Vert\nabla^2 f_i(\theta^{t+1})\nabla f^q_i(\phi^{t+1}_i)-g^{t+1}_i\Vert\le\delta_{i,t+1}\zeta_i\beta^2_i.
\end{align*}
Plugging \cref{dualbound_eq4} into \cref{dualbound_eq3}, we have
\begin{align*}
    \Vert y^{t+1}_i- y^t_i\Vert\le w_i\nu_i\Vert\theta^{t+1}-\theta^t\Vert + (\delta_{i,t}+\delta_{i,t+1})\alpha w_i\zeta_i\beta^2_i,
\end{align*}
which completes the proof.

\section*{Proof of \cref{l_descent}}

First, define $\hat{\mathcal{L}}_i(\theta_i,\theta^{t+1}, y^t_i)$ and $\tilde{\mathcal{L}}_i(\theta_i,\theta^{t+1}, y^t_i)$ as
\begin{align*}
    &\hat{\mathcal{L}}_i(\theta_i,\theta^{t+1}, y^t_i)
    \doteq \langle(I-\alpha\nabla^2 f_i(\theta^{t+1}))\nabla f^q_i(\phi^{t+1}_i),\theta_i-\theta^{t+1}\rangle \nonumber\\
    &~~\cdot w_i + w_i F_i(\theta^{t+1})+\langle y^t_i,\theta_i-\theta^{t+1}\rangle + \rho_i/2\cdot \Vert \theta_i-\theta^{t+1}\Vert^2,\\
    &\tilde{\mathcal{L}}_i(\theta_i,\theta^{t+1}, y^t_i)
    \doteq w_i\langle\nabla f^q_i(\phi^{t+1}_i)-\alpha g^{t+1}_i,\theta_i-\theta^{t+1}\rangle \nonumber\\
    &~~+ w_i F_i(\theta^{t+1})+\langle y^t_i,\theta_i-\theta^{t+1}\rangle + \rho_i/2\cdot \Vert \theta_i-\theta^{t+1}\Vert^2,
\end{align*}
where $g^{t+1}_i$ and $\phi^{t+1}_i$ are defined in \cref{estimatehessian,eq:update_phi}. For any $i\in\mathcal{I}$, using Lemma \ref{fsmoothlemma} yields
\begin{align}
    \label{l_descent_eq1}
    \mathcal{L}_i(\theta_i,\theta^{t+1}, y^t_i)\le\hat{\mathcal{L}}_i(\theta_i,\theta^{t+1}, y^t_i)+w_i\nu_i/2\cdot \Vert \theta_i-\theta^{t+1}\Vert^2.
\end{align}
Recall that 
\begin{align*}
    \Vert\nabla^2 f_i(\theta^{t+1})\nabla f^q_i(\phi^{t+1}_i)-g^{t+1}_i\Vert\le\zeta_i\beta^2_i\delta_{i,t+1}.
\end{align*}
Using the Cauchy-Schwarz inequality, we can write
\begin{align}
    \label{l_descent_eq2}
    \hat{\mathcal{L}}_i(\theta_i,\theta^{t+1}, y^t_i)\le \tilde{\mathcal{L}}_i(\theta_i,\theta^{t+1}, y^t_i) + \alpha w_i\zeta_i\beta^2_i\delta_{i,t+1}\Vert \theta_i-\theta^{t+1}\Vert.
\end{align}
Combining (\ref{l_descent_eq1}) and (\ref{l_descent_eq2}) yields that
\begin{align}
    \label{l_descent_eq3}
    \mathcal{L}_i(\theta_i,\theta^{t+1}, y^t_i)&\le \tilde{\mathcal{L}}_i(\theta_i,\theta^{t+1}, y^t_i)+w_i\nu_i/2\cdot \Vert \theta_i-\theta^{t+1}\Vert^2\nonumber\\
    &+\alpha w_i\zeta_i\beta^2_i\delta_{i,t+1}\Vert \theta_i-\theta^{t+1}\Vert.
\end{align}
Based on (\ref{dualbound_eq1}) and the strong convexity of $\tilde{\mathcal{L}}_i(\theta_i,\theta^{t+1}, y^t_i)$ with modulus $\rho_i$, for each $i\in\mathcal{I}$, the following fact holds: 
\begin{align}
    \label{l_descent_eq4}
    \tilde{\mathcal{L}}_i(\theta^{t+1}_i,\theta^{t+1}, y^t_i)-\tilde{\mathcal{L}}_i(\theta^t_i,\theta^{t+1}, y^t_i)\le-\rho_i/2\cdot\Vert\theta^{t+1}_i-\theta^t_i\Vert^2.
\end{align}
It follows that
\begin{align}
    \nonumber
    &\tilde{\mathcal{L}}_i(\theta^t_i,\theta^{t+1}, y^t_i)-\mathcal{L}_i(\theta^t_i,\theta^{t+1}, y^t_i)\\
    \nonumber
    =\,&w_i(\langle\nabla f^q_i(\phi^{t+1}_i)-\alpha g^{t+1}_i,\theta^t_i-\theta^{t+1}\rangle+ F_i(\theta^{t+1})- F_i(\theta^t_i))\\
    \nonumber
    =\,&w_i F_i(\theta^{t+1})-w_i F_i(\theta^t_i)-w_i\langle \nabla F_i(\theta^t_i),\theta^{t+1}-\theta^t_i\rangle\\
    \nonumber
    &-w_i\nu_i/2\cdot\Vert \theta^t_i-\theta^{t+1}\Vert^2+w_i\nu_i/2\cdot \Vert \theta^t_i-\theta^{t+1}\Vert^2\\
    \nonumber
    &+w_i\langle\nabla f^q_i(\phi^{t+1}_i)-\alpha g^{t+1}_i-\nabla F_i(\theta^t_i),\theta^t_i-\theta^{t+1}\rangle\\
    \le\,& w_i\langle\nabla f^q_i(\phi^{t+1}_i)-\alpha g^{t+1}_i-\nabla F_i(\theta^t_i),\theta^t_i-\theta^{t+1}\rangle
    \label{l_descent_eq5_1}
    \\
    \nonumber
    &+w_i\nu_i/2\cdot\Vert \theta^t_i-\theta^{t+1}\Vert^2\\
    \nonumber
    \le\,& w_i\Vert\nabla f^q_i(\phi^{t+1}_i)-\alpha g^{t+1}_i - \nabla F_i(\theta^{t+1})\Vert\cdot\Vert \theta^t_i-\theta^{t+1}\Vert\\
    \nonumber
    &+w_i\Vert\nabla F_i(\theta^{t+1})- \nabla F_i(\theta^t_i)\Vert\cdot\Vert \theta^t_i-\theta^{t+1}\Vert\\
    \nonumber
    &+w_i\nu_i/2\cdot\Vert \theta^t_i-\theta^{t+1}\Vert^2\\
    \nonumber
    \le\,&3w_i\nu_i/2\cdot\Vert \theta^t_i-\theta^{t+1}\Vert^2 + \alpha w_i\zeta_i\beta^2_i\delta_{i,t+1}\Vert \theta^t_i-\theta^{t+1}\Vert\\
    \le\,&3w_i\nu_i(\Vert \theta^t_i-\theta^{t+1}_i\Vert^2+\Vert \theta^{t+1}_i-\theta^{t+1}\Vert^2)+\alpha w_i\zeta_i\beta^2_i\delta_{i,t+1}
    \nonumber
    \\
    &\cdot(\Vert \theta^t_i-\theta^{t+1}_i\Vert+\Vert \theta^{t+1}_i-\theta^{t+1}\Vert),
    \label{l_descent_eq5}
\end{align}
where \cref{l_descent_eq5_1} is derived from Lemma \ref{fsmoothlemma} and the last inequality is obtained based on the following fact: 
    \begin{align}
        \Vert x+y\Vert^2\le 2\Vert x\Vert^2+2\Vert y\Vert^2,~x,y\in\mathbb{R}^n.
    \end{align}
Combining \cref{l_descent_eq3,l_descent_eq4,l_descent_eq5}, we conclude that
\begin{align}
    \label{l_descent_eq6}
    \nonumber
     &\mathcal{L}_i(\theta^{t+1}_i,\theta^{t+1}, y^t_i)- \mathcal{L}_i(\theta^t_i,\theta^{t+1}, y^t_i)\\
     \nonumber
     \le\,& \tilde{\mathcal{L}}_i(\theta^{t+1}_i,\theta^{t+1}, y^t_i)-\tilde{\mathcal{L}}_i(\theta^t_i,\theta^{t+1}, y^t_i)+\tilde{\mathcal{L}}_i(\theta^t_i,\theta^{t+1}, y^t_i)\\
     \nonumber
     &-\mathcal{L}_i(\theta^t_i,\theta^{t+1}, y^t_i)+w_i\nu_i/2\cdot\Vert \theta^t_i-\theta^{t+1}\Vert^2\\
     \nonumber
     &+\alpha w_i\zeta_i\beta^2_i\delta_{i,t+1}\Vert \theta^{t+1}_i-\theta^{t+1}\Vert\\
     \nonumber
     \le\,& -(\rho_i-8w_i\nu_i)/2\cdot\Vert\theta^{t+1}_i-\theta^t_i\Vert^2+4w_i\nu_i\Vert \theta^{t+1}_i-\theta^{t+1}\Vert^2\\
     \nonumber
     &+2\alpha w_i\zeta_i\beta^2_i\delta_{i,t+1}\Vert \theta^{t+1}_i-\theta^{t+1}\Vert\\
     \nonumber
     &+\alpha w_i\zeta_i\beta^2_i\delta_{i,t+1}\Vert \theta^t_i-\theta^{t+1}_i\Vert\\
     \nonumber
     \le\,& -(\rho_i-8w_i\nu_i)/2\cdot\Vert\theta^{t+1}_i-\theta^t_i\Vert^2\\
     \nonumber
     &+4w_i\nu_i/\rho^2_i\cdot\Vert  y^{t+1}_i- y^t_i\Vert^2\\
     \nonumber
     &+(2\alpha w_i\zeta_i\beta^2_i\delta_{i,t+1})/\rho_i\cdot\Vert  y^{t+1}_i- y^t_i\Vert\\
     &+\alpha w_i\zeta_i\beta^2_i\delta_{i,t+1}\Vert \theta^t_i-\theta^{t+1}_i\Vert,
\end{align}
where the last inequality is derived from \cref{eq:update_y}. The proof is completed.

\section*{Proof of \cref{L_descent}}

Based on \cref{eq:update_y}, we have
\begin{align}
    \label{L_descent_eq1}
    \nonumber
    &\mathcal{L}(\{\theta^{t+1}_i,y^{t+1}_i\},\theta^{t+1})-\mathcal{L}(\{\theta^{t+1}_i,y^t_i\},\theta^{t+1})\\
    \nonumber
    =\,&\sum\nolimits_{i\in\mathcal{I}}\langle  y^{t+1}_i- y^t_i,\theta^{t+1}_i-\theta^{t+1}\rangle\\
    =\,&\sum\nolimits_{i\in\mathcal{I}}1/\rho_i\cdot\Vert y^{t+1}_i- y^t_i\Vert^2.
\end{align}
Using \cref{l_descent} and \cref{Lsmooth}, we have
\begin{align}
    \label{L_descent_eq2}
    \nonumber
    &\mathcal{L}(\{\theta^{t+1}_i,y^t_i\},\theta^{t+1})-\mathcal{L}(\{\theta^t_i,y^t_i\},\theta^t)\\
    \nonumber
    =\,&\mathcal{L}(\{\theta^{t+1}_i,y^t_i\},\theta^{t+1})-\mathcal{L}(\{\theta^t_i,y^t_i\},\theta^{t+1})\\
    \nonumber
    &+\mathcal{L}(\{\theta^t_i,y^t_i\},\theta^{t+1})-\mathcal{L}(\{\theta^t_i,y^t_i\},\theta^t)\\
    \nonumber
    =\,&\sum\nolimits_{i\in\mathcal{I}}\left(\mathcal{L}_i(\theta^{t+1}_i,\theta^{t+1}, y^t_i)-\mathcal{L}_i(\theta^t_i,\theta^{t+1}, y^t_i)\right)\\
    \nonumber
    &+\underbrace{\mathcal{L}(\{\theta^t_i,y^t_i\},\theta^{t+1})-\mathcal{L}(\{\theta^t_i,y^t_i\},\theta^t)}_\text{(a)}\\
    \nonumber
    \le\,& -\sum\nolimits_{i\in\mathcal{I}}((\rho_i-8w_i\nu_i)/2\cdot\Vert\theta^{t+1}_i-\theta^t_i\Vert^2\\
    \nonumber
    &+(\rho_i-\lambda\mu_r/|\mathcal{I}|)/2\cdot\Vert\theta^{t+1}-\theta^t\Vert^2\\
    \nonumber
    &-4w_i\nu_i/\rho^2_i\cdot\Vert  y^{t+1}_i- y^t_i\Vert^2\\
    &-2\alpha w_i\zeta_i\beta^2_i\delta_{i,t+1}/\rho_i\cdot\Vert y^{t+1}_i- y^t_i\Vert\\
    \nonumber
    &-\alpha w_i\zeta_i\beta^2_i\delta_{i,t+1}\Vert \theta^{t+1}_i-\theta^t_i\Vert),
\end{align}
where the bound of (a) is derived in the similar way as in \cref{l_descent}. Combining \cref{L_descent_eq1,L_descent_eq2}, we conclude that
\begin{align}
    \label{L_descent_eq3}
    \nonumber
    &\mathcal{L}(\{\theta^{t+1}_i,y^{t+1}_i\},\theta^{t+1})-\mathcal{L}(\{\theta^t_i,y^t_i\},\theta^t)\\
    \nonumber
    =\,& \mathcal{L}(\{\theta^{t+1}_i,y^{t+1}_i\},\theta^{t+1})-\mathcal{L}(\{\theta^{t+1}_i,y^t_i\}_i,\theta^{t+1})\\
    \nonumber
    &+\mathcal{L}(\{\theta^{t+1}_i,y^t_i\},\theta^{t+1})-\mathcal{L}(\{\theta^t_i,y^t_i\},\theta^t)\\
    \nonumber
    =\,&-\sum\nolimits_{i\in\mathcal{I}}( (\rho_i-8w_i\nu_i)/2\cdot\Vert\theta^{t+1}_i-\theta^t_i\Vert^2\\
    \nonumber
    &+(\rho_i-\lambda\mu_r/|\mathcal{I}|)/2\cdot\Vert\theta^{t+1}-\theta^t\Vert^2\\
    \nonumber
    &-(4w_i\nu_i/\rho^2_i+1/\rho_i)\Vert y^{t+1}_i- y^t_i\Vert^2\\
    \nonumber
    &-2\alpha w_i\zeta_i\beta^2_i\delta_{i,t+1}/\rho_i\cdot\Vert y^{t+1}_i- y^t_i\Vert\\
    \nonumber
    &-\alpha w_i\zeta_i\beta^2_i\delta_{i,t+1}\Vert \theta^t_i-\theta^{t+1}_i\Vert)\\
    \nonumber
    \le\,&-\sum\nolimits_{i\in\mathcal{I}}((\rho_i-8w_i\nu_i)/2\cdot\Vert\theta^{t+1}_i-\theta^t_i\Vert^2\\
    \nonumber
    &+(\rho_i-\lambda\mu_r/|\mathcal{I}|)/2\cdot\Vert\theta^{t+1}-\theta^t\Vert^2\\
    \nonumber
    &-(4w_i\nu_i/\rho^2_i+1/\rho_i)(2w^2_i\nu^2_i\Vert\theta^{t+1}-\theta^t\Vert^2 \\
    \nonumber
    &+ 2(\delta_{i,t}+\delta_{i,t+1})^2(\alpha w_i\zeta_i\beta^2_i)^2)\\
    \nonumber
    &-2\alpha w_i\zeta_i\beta^2_i\delta_{i,t+1}/\rho_i\cdot(w_i\nu_i\Vert\theta^{t+1}-\theta^t\Vert \\
    \nonumber
    &+ (\delta_{i,t}+\delta_{i,t+1})\alpha w_i\zeta_i\beta^2_i)\\
    \nonumber
    &-\alpha w_i\zeta_i\beta^2_i\delta_{i,t+1}\Vert \theta^t_i-\theta^{t+1}_i\Vert)\\
    \nonumber
    =\,&-\sum\nolimits_{i\in\mathcal{I}}(a_{i,e}\Vert\theta^{t+1}_i-\theta^t_i\Vert^2+a_{i,p}\Vert\theta^{t+1}-\theta^t\Vert^2\\
    &-b^{t+1}_{i,e}\Vert\theta^{t+1}_i-\theta^t_i\Vert-b^{t+1}_{i,p}\Vert\theta^{t+1}-\theta^t\Vert-c^{t+1}_i),
\end{align}
thereby completing the proof.

\section*{Proof of \cref{lagrangian_lowerbound}}

From \cref{dualbound_eq2}, we have
\begin{align}
    - y^{t+1}_i=w_i\tilde{\nabla}F_i(\theta^{t+1})
\end{align}
where $\tilde{\nabla}F_i(\theta^{t+1})= \nabla f^q_i(\phi^{t+1}_i)-\alpha g^{t+1}_i$. Due to Lemma \ref{fsmoothlemma}, we can write
\begin{align}
    \label{lagrangian_lowerbound_eq1}
    \nonumber
    F_i(\theta^{t+1})\le\;& F_i(\theta^{t+1}_i)+\langle \nabla F_i(\theta^{t+1}_i),\theta^{t+1}-\theta^{t+1}_i\rangle\\
    \nonumber
    &+\nu_i/2\cdot\Vert\theta^{t+1}_i-\theta^{t+1}\Vert^2\\
    \nonumber
    =\;& F_i(\theta^{t+1}_i)+\langle \nabla F_i(\theta^{t+1}),\theta^{t+1}-\theta^{t+1}_i\rangle\\
    \nonumber
    &+\langle \nabla F_i(\theta^{t+1}_i)-\nabla F_i(\theta^{t+1}),\theta^{t+1}-\theta^{t+1}_i\rangle\\
    \nonumber
    &+\nu_i/2\cdot\Vert\theta^{t+1}_i-\theta^{t+1}\Vert^2\\
    \nonumber
    \le\;& F_i(\theta^{t+1}_i)+\langle F_i(\theta^{t+1}),\theta^{t+1}-\theta^{t+1}_i\rangle\\
    &+3\nu_i/2\cdot\Vert\theta^{t+1}_i-\theta^{t+1}\Vert^2.
\end{align}
Based on the definition of the augmented Lagrangian function in \cref{eq:lagrangian} and \cref{lagrangian_lowerbound_eq1}, we have the following observation:
\begin{align}
    \nonumber
    &\mathcal{L}(\{\theta^{t+1}_i,y^{t+1}_i\},\theta^{t+1})\\
    \nonumber
    =\,&\lambda R_h(\theta^{t+1},\theta_p)+\sum\nolimits_{i\in\mathcal{I}}(w_i F_i(\theta^{t+1}_i)
    \\
    \nonumber
    &+\langle y^{t+1} _i,\theta^{t+1}_i-\theta^{t+1}\rangle+\rho_i/2\cdot\Vert \theta^{t+1}_i-\theta^{t+1}\Vert^2)\\
    \nonumber
    =\,&\lambda R_h(\theta^{t+1},\theta_p)+\sum\nolimits_{i\in\mathcal{I}}(w_i F_i(\theta^{t+1}_i)\\
    \nonumber
    &+\langle w_i\tilde{\nabla}F_i(\theta^{t+1}),\theta^{t+1}-\theta^{t+1}_i\rangle\\
    \nonumber
    &+\rho_i/2\cdot\Vert \theta^{t+1}_i-\theta^{t+1}\Vert^2)\\
    \nonumber
    =\,&\lambda R_h(\theta^{t+1},\theta_p)+\sum\nolimits_{i\in\mathcal{I}}(w_i F_i(\theta^{t+1}_i)\\
    \nonumber
    &+\langle w_i\nabla F_i(\theta^{t+1}),\theta^{t+1}-\theta^{t+1}_i\rangle\\
    \nonumber
    &+ w_i\langle \tilde{\nabla}F_i(\theta^{t+1})-\nabla F_i(\theta^{t+1}),\theta^{t+1}-\theta^{t+1}_i\rangle\\
    \nonumber
    &+\rho_i/2\cdot\Vert \theta^{t+1}_i-\theta^{t+1}\Vert^2)\\
    \nonumber
    \ge\,& \lambda R_h(\theta^{t+1},\theta_p)+\sum\nolimits_{i\in\mathcal{I}}(w_i F_i(\theta^{t+1}_i)\\
    \nonumber
    &+\langle w_i\nabla F_i(\theta^{t+1}),\theta^{t+1}-\theta^{t+1}_i\rangle\\
    \nonumber
    &- w_i\Vert \tilde{\nabla}F_i(\theta^{t+1})-\nabla F_i(\theta^{t+1})\Vert\Vert\theta^{t+1}-\theta^{t+1}_i\Vert\\
    \nonumber
    &+\frac{\rho_i}{2}\Vert \theta^{t+1}_i-\theta^{t+1}\Vert^2)\\
    \nonumber
    \ge\,& \lambda R_h(\theta^{t+1},\theta_p)+ \sum\nolimits_{i\in\mathcal{I}}(w_i F_i(\theta^{t+1})\\
    \nonumber
    &+(\rho_i-3\nu_i)/2\cdot\Vert \theta^{t+1}_i-\theta^{t+1}\Vert^2\\
    \nonumber
    &- \alpha w_i\zeta_i\beta^2_i\delta_{i,t+1}\Vert\theta^{t+1}-\theta^{t+1}_i\Vert)
\end{align}
where the last inequality is derived from $\nu_i$-smoothness of $F_i$ and first-order Taylor expansion. According to \cref{lowerbounded}, $\lambda R_h(\theta^{t+1},\theta_p)+\sum_{i\in\mathcal{I}} w_i F_i(\theta^{t+1})$ is lower bounded. Due to \cref{parameterassumption}, it is easy to show that
\begin{align}
    \sum\nolimits_{i\in\mathcal{I}}((\rho_i-3\nu_i)/2\cdot\Vert \theta^{t+1}_i-\theta^{t+1}\Vert^2\nonumber\\
    - \alpha w_i\zeta_i\beta^2_i\delta_{i,t+1}\Vert\theta^{t+1}-\theta^{t+1}_i\Vert)>-\infty,
\end{align}
thereby completing the proof.

\section*{Proof of \cref{lem:forgetting}}

From \cref{coro:data_impact}, the following fact holds:
\begin{align}
\nonumber
&\mathbb{E}\{\Vert \sum\nolimits_{i\in\mathcal{I}}w_i \nabla L_i(\theta_{\epsilon}-\alpha\nabla L_i(\theta_{\epsilon}))+\lambda \nabla R_h(\theta_{\epsilon},\theta_p)\Vert\}\\
&\le\epsilon+\sum\nolimits_{i\in\mathcal{I}}w_i\sigma^g_i(\frac{\alpha\mu_i}{\sqrt{D^{s}_i}}+\frac{1}{\sqrt{D^{q}_i}}).
\end{align}
Based on \cref{Lsmooth}, we obtain
\begin{align}
\label{eq:forgetting_l2_1}
 \lambda\Vert \nabla R_h(\theta_{\epsilon},\theta_p)\Vert=\Vert\sum\nolimits_{i\in\mathcal{I}}w_i \nabla L_i(\theta_{\epsilon}-\alpha\nabla L_i(\theta_{\epsilon}))\nonumber\\
 -\sum\nolimits_{i\in\mathcal{I}}w_i \nabla L_i(\theta_{\epsilon}-\alpha\nabla L_i(\theta_{\epsilon}))+\lambda \nabla R_h(\theta_{\epsilon},\theta_p)\Vert\nonumber\\
\le\Vert\sum\nolimits_{i\in\mathcal{I}}w_i \nabla L_i(\theta_{\epsilon}-\alpha\nabla L_i(\theta_{\epsilon}))+\lambda \nabla R_h(\theta_{\epsilon},\theta_p)\Vert\nonumber\\
+\Vert\sum\nolimits_{i\in\mathcal{I}}w_i \nabla L_i(\theta_{\epsilon}-\alpha\nabla L_i(\theta_{\epsilon}))\Vert\nonumber\\
\le\epsilon+\sum\nolimits_{i\in\mathcal{I}}w_i(\beta_i+\frac{\alpha\mu_i\sigma^g_i}{\sqrt{D^{s}_i}}+\frac{\sigma^g_i}{\sqrt{D^{q}_i}}).
\end{align}
Due to the convexity of $R_h(\cdot,\theta_p)$ and the fact $R_h(\theta_p,\theta_p)=0$, we have
\begin{align}
\mathbb{E}\{R_h(\theta,\theta_p)\}\le\;&\frac{1}{\lambda}(\epsilon+\sum\nolimits_{i\in\mathcal{I}}w_i(\beta_i+\frac{\alpha\mu_i\sigma^g_i}{\sqrt{D^s_i}}+\frac{\sigma^g_i}{\sqrt{D^q_i}}))\nonumber\\
\cdot\;&\Vert\theta_{\epsilon}-\theta_p\Vert.
\end{align}
Equation \eqref{eq:forgetting_general_strong_ex} can be directly derived via \eqref{eq:strongly_convex_primal_ex}.

\vfill

\end{document}